\documentclass{article}

\usepackage[top=1in, bottom=1in, left=1in, right=1in]{geometry}

\usepackage{comment,url,algorithm,algorithmic,graphicx,subcaption,relsize}
\usepackage{amssymb,amsfonts,amsmath,amsthm,amscd,dsfont,mathrsfs,mathtools,nicefrac}
\usepackage{float,psfrag,epsfig,color,xcolor,url,hyperref}
\usepackage{epstopdf,bbm,enumitem}
\usepackage[toc,page]{appendix}
\usepackage[mathscr]{euscript}

\def\ddefloop#1{\ifx\ddefloop#1\else\ddef{#1}\expandafter\ddefloop\fi}

\def\ddef#1{\expandafter\def\csname bb#1\endcsname{\ensuremath{\mathbb{#1}}}}
\ddefloop ABCDEFGHIJKLMNOPQRSTUVWXYZ\ddefloop

\def\ddef#1{\expandafter\def\csname c#1\endcsname{\ensuremath{\mathcal{#1}}}}
\ddefloop ABCDEFGHIJKLMNOPQRSTUVWXYZ\ddefloop

\def\ddef#1{\expandafter\def\csname v#1\endcsname{\ensuremath{\boldsymbol{#1}}}}
\ddefloop ABCDEFGHIJKLMNOPQRSTUVWXYZabcdefghijklmnopqrstuvwxyz\ddefloop

\def\ddef#1{\expandafter\def\csname v#1\endcsname{\ensuremath{\boldsymbol{\csname #1\endcsname}}}}
\ddefloop {alpha}{beta}{gamma}{delta}{epsilon}{varepsilon}{zeta}{eta}{theta}{vartheta}{iota}{kappa}{lambda}{mu}{nu}{xi}{pi}{varpi}{rho}{varrho}{sigma}{varsigma}{tau}{upsilon}{phi}{varphi}{chi}{psi}{omega}{Gamma}{Delta}{Theta}{Lambda}{Xi}{Pi}{Sigma}{varSigma}{Upsilon}{Phi}{Psi}{Omega}{ell}\ddefloop

\def\balign#1\ealign{\begin{align}#1\end{align}}
\def\baligns#1\ealigns{\begin{align*}#1\end{align*}}
\def\balignat#1\ealign{\begin{alignat}#1\end{alignat}}
\def\balignats#1\ealigns{\begin{alignat*}#1\end{alignat*}}
\def\bitemize#1\eitemize{\begin{itemize}#1\end{itemize}}
\def\benumerate#1\eenumerate{\begin{enumerate}#1\end{enumerate}}

\newenvironment{talign*}
{\csname align*\endcsname}
{\endalign}
\newenvironment{talign}
{\csname align\endcsname}
{\endalign}

\def\balignst#1\ealignst{\begin{talign*}#1\end{talign*}}
\def\balignt#1\ealignt{\begin{talign}#1\end{talign}}

\let\originalleft\left
\let\originalright\right
\renewcommand{\left}{\mathopen{}\mathclose\bgroup\originalleft}
\renewcommand{\right}{\aftergroup\egroup\originalright}

\def\tinycitep*#1{{\tiny\citep*{#1}}}
\def\tinycitealt*#1{{\tiny\citealt*{#1}}}
\def\tinycite*#1{{\tiny\cite*{#1}}}
\def\smallcitep*#1{{\scriptsize\citep*{#1}}}
\def\smallcitealt*#1{{\scriptsize\citealt*{#1}}}
\def\smallcite*#1{{\scriptsize\cite*{#1}}}

\def\mbi#1{\boldsymbol{#1}} 

\def\mbb#1{\mathbb{#1}}

\newcommand{\norm}[1]{\left\lVert#1\right\rVert}
\newcommand{\snorm}[1]{\lVert#1\rVert}

\theoremstyle{plain}

\def\R{\mathbb{R}}

\def\N{\mathbb{N}}

\def\<{\left\langle} 
\def\>{\right\rangle}

\def\norm#1{\left\|{#1}\right\|} 

\def\Re{\operatorname{Re}} 

\def\E{\mbb{E}} 

\def\P{\mbb{P}} 

\newcommand{\Tr}{\operatorname{Tr}}

\DeclareSymbolFont{rsfs}{U}{rsfs}{m}{n}
\DeclareSymbolFontAlphabet{\mathscrsfs}{rsfs}

\providecommand{\diag}{\mathop\mathrm{diag}}
\providecommand{\offdiag}{\mathop\mathrm{offdiag}}

\def\rank#1{\mathrm{rank}({#1})}
\def\supp#1{\mathrm{supp}({#1})}

\ifdefined\nonewproofenvironments\else
\ifdefined\ispres\else
\newtheorem{theorem}{Theorem}
\newtheorem{lemma}[theorem]{Lemma}
\newtheorem{coro}[theorem]{Corollary}

\newtheorem{prop}[theorem]{Proposition}

\theoremstyle{definition}
\newtheorem{defi}[theorem]{Definition}

\renewenvironment{proof}{\noindent\textbf{Proof.}\hspace*{.3em}}{\qed\\}
\newenvironment{proof-sketch}{\noindent\textbf{Proof Sketch}
	\hspace*{0.3em}}{\qed\\}
\newenvironment{proof-idea}{\noindent\textbf{Proof Idea}
	\hspace*{0.3em}}{\qed\\}
\newenvironment{proof-of-lemma}[1][{}]{\noindent\textbf{Proof of Lemma {#1}.}
	\hspace*{0.3em}}{\qed\\}
\newenvironment{proof-of-theorem}[1][{}]{\noindent\textbf{Proof of Theorem {#1}.}
	\hspace*{0.3em}}{\qed\\}
\newenvironment{proof-of-coro}[1][{}]{\noindent\textbf{Proof of Corollary {#1}.}
	\hspace*{0.3em}}{\qed\\}
\fi

\newtheorem{proposition}[theorem]{Proposition}

\newtheorem{assumption}{Assumption}

\fi
\makeatletter
\@addtoreset{equation}{section}
\makeatother

\hypersetup{
colorlinks,
linkcolor={red!50!black},
citecolor={blue!50!black},
urlcolor={blue!80!black}
}

\mathtoolsset{showonlyrefs}

\usepackage{caption}
\usepackage{authblk}
\usepackage{etoc}
\usepackage{wrapfig}

\allowdisplaybreaks

\newcommand{\iid}{\stackrel{\mathrm{\tiny{i.i.d.}}}{\sim}}

\newcommand{\C}{\mathbb{C}}

\newcommand{\eps}{\varepsilon}

\renewcommand{\Im}{\operatorname{Im}}
\newcommand{\bI}{\mathbf{1}}
\newcommand{\dist}[1]{\textnormal{dist}(#1)}
\newcommand{\vcR}{\mbi{\cR}}

\newcommand{\spec}{\operatorname{spec}}

\newcommand{\MP}{\textnormal{MP}}

\newcommand{\SD}[1]{O_{\prec}\left(#1 \right)}
\newcommand{\SDcE}[1]{O_{\prec}^{\,\cE}\left(#1 \right)}
\newcommand{\SDcEL}[1]{O_{\prec}^{\,\cE_L}\left(#1 \right)}
\newcommand{\SDcES}[1]{O_{\prec}^{\,\cE(S)}\left(#1 \right)}

\newcommand{\citep}[1]{\cite{#1}} 
\newcommand{\citet}[1]{\cite{#1}}

\makeatletter
\renewcommand{\paragraph}{%
\@startsection{paragraph}{4}%
{\z@}{1.5ex \@plus 1ex \@minus .2ex}{-1em}%
{\normalfont\normalsize\bfseries}%
}
\makeatother

\title{Nonlinear spiked covariance matrices and signal propagation\\
 in deep neural networks}
 
\author{
Zhichao Wang\thanks{Department of Mathematics, University of California San Diego.  \texttt{zhw036@ucsd.edu}.} ,\,
Denny Wu\thanks{Center for Data Science, New York University, and Flatiron Institute. \texttt{dennywu@nyu.edu}.} ,\,
Zhou Fan\thanks{Department of Statistics and Data Science, Yale University. \texttt{zhou.fan@yale.edu}. 
} 

}
\date{}

\begin{document}
\maketitle

\begin{abstract}

Many recent works have studied the eigenvalue spectrum of the Conjugate Kernel
(CK) defined by the nonlinear feature map of a feedforward neural network.
However, existing results only establish weak convergence of the empirical
eigenvalue distribution, and fall short of providing precise quantitative
characterizations of the ``spike'' eigenvalues and eigenvectors that
often capture the low-dimensional signal structure of the 
learning problem. In this work, we characterize these signal eigenvalues and
eigenvectors for a nonlinear version of the spiked covariance model,
including the CK as a special case. Using this general result, we give
a quantitative description of how spiked eigenstructure in the input data
propagates through the hidden layers of a neural network with random weights.
As a second application, we study a simple regime of representation learning
where the weight matrix develops a rank-one signal component over training
and characterize the alignment of the target function with the spike
eigenvector of the CK on test data.

\end{abstract}

\section{Introduction}

Kernel matrices associated with the nonlinear feature map
of deep neural networks (NNs) provide insight into the optimization dynamics \citep{jacot2018neural,montanari2020interpolation,fort2020deep} and predictive performance \citep{lee2017deep,arora2019fine,ortiz2021can}; consequently, properties of these kernel matrices can guide the design of network architecture 
\citep{xiao2018dynamical,martens2021rapid,li2022neural} and learning algorithms
\citep{karakida2020understanding,zhou2022dataset}. Particular
emphasis has been placed on the \textit{spectral properties} of kernel matrices, due to their connection with the training and test performance of the underlying NN \citep{bordelon2020spectrum,loureiro2021learning,wei2022more}. 

In this paper, we focus on the \textit{conjugate kernel} (CK)
\citep{neal1995bayesian,cho2009kernel} defined as the Gram matrix of the
features at the penultimate (or more generally, any intermediate) NN layer.
In the high-dimensional asymptotic setting where the width of the NN and the
number of training samples diverge at the same rate, prior works 
employed random matrix theory to analyze the limit eigenvalue distribution of
the CK matrix at random initialization
\citep{pennington2017nonlinear,louart2018random,peche2019note,fan2020spectra}.
These and related characterizations of the CK resolvent enable precise
computations of various errors for NNs with random first-layer
weights, known as \textit{random features models} 
\citep{mei2019generalization,tripuraneni2021covariate,hassani2022curse}. 

It is worth noting that while existing results establish the weak convergence of
the empirical spectral measure, the precise behavior of ``spike'' eigenvalues
that are separated from the spectral bulk remains largely unexplored.
In learning applications, these spike eigenvalues and corresponding eigenvectors
are often the primary spectral features (signal) of interest, because they pertain to
low-rank structure of the underlying learning problem (e.g., class labels or
the direction of the target function). For the linearly defined \textit{spiked
covariance model} $\vX=\vZ\vSigma^{1/2} \in \R^{n \times d}$,
whose dependence across features is induced by a
linear map $\vSigma^{1/2}(\cdot)$ applied to $\vZ$ having i.i.d.\ coordinates,
classical work in random matrix theory provides a quantitative description of
the spike eigenvalue/eigenvector behavior
\citep{johnstone2001distribution,baik2006eigenvalues,benaych2012singular,bloemendal2016principal}.
In this paper, we establish an analogous characterization of spiked spectral
structure for the CK, motivated in part by the following applications: 

\begin{itemize}[leftmargin=*]
    \item \textbf{Structured input data.} Real data often contain low-dimensional structure despite the high ambient dimensionality \citep{lee2007nonlinear,hastie2009elements,pope2021intrinsic}, and the leading eigenvectors of the input covariance matrix may be good predictors of the training labels. 
    Common examples where the input features exhibit a low-dimensional spiked structure include
Gaussian mixture models \citep{loureiro2021learning,refinetti2021classifying,arous2023high1} and the
block-covariance setting of \citep{ghorbani2020neural,ba2023learning,mousavi2023gradient}. Assuming that the input data $\vX$ has informative spikes
eigenvectors, we ask the natural question: 
    \vspace{-1.mm}
    \begin{center}
        \textit{
        How does the low-dimensional signal propagate through nonlinear layers
of the NN? \\ When do we observe a similar spiked structure in the CK matrix? }
    \end{center}
    \item \textbf{Spiked weight matrices in early training.} It is known that NNs can
learn useful representations that adapt to the learning problem, and outperform
the random features model defined by randomly initialized weights \citep{ghorbani2019limitations,wei2019regularization,abbe2022merged}. 
    Recent works have shown that when the target function is low-dimensional,
the gradient update for two-layer NNs around initialization is \textit{low-rank}
\citep{ba2022high,damian2022neural,wang2022spectral}, and hence the updated weight matrix $\vW$
is well-approximated by a spiked model. We consider the following question on this pre-trained kernel model in NNs: 
    \vspace{-1.mm}
    \begin{center}
        \textit{
        When gradient descent produces a spiked structure in the weight matrix,
how does the feature representation of the NN change, in terms of spectral properties of the CK?}
    \end{center}
\end{itemize}

\subsection{Our Contributions}

We analyze the spike eigenstructure in a general nonlinear spiked 
covariance model, which includes the CK as a special case. Specifically, we
characterize the BBP phase transition \citep{baik2005phase} and first-order
limits of the eigenvalues and eigenvector alignments in the proportional
asymptotics regime, for spike eigenvalues of bounded size.
Our work makes the following contributions:
\begin{itemize}[leftmargin=*]
    \item \textbf{Signal propagation in deep random NNs.} Following the setup of
\cite{fan2020spectra}, we consider the CK matrix defined by a multi-layer
fully-connected NN at random initialization, where the width of each layer grows
linearly with the sample size. Given spiked input data, we compute the magnitude
of the leading CK eigenvalues and the alignments between the corresponding CK
eigenvectors with those of the input data, across network depth.
    \item \textbf{Feature learning in two-layer NNs.} We consider the
early-phase feature learning setting in \cite{ba2022high}, where the first-layer
weights in a two-layer NN are optimized by gradient descent, and the learned
weight matrix exhibits a rank-one spiked structure. We characterize the
spiked eigenstructure of the corresponding CK matrix for independent test data,
and the alignment of spike eigenvectors with the test labels.
This provides a quantitative description of how gradient descent improves the NN representation. 
    \item \textbf{Spectral analysis for nonlinear spiked covariance models.} We
give a general analysis of the signal eigenvalues/eigenvectors of spiked
covariance matrices with arbitrary and possibly nonlinear dependence across
features, showing a ``Gaussian equivalence'' with the quantitative spectral
properties of linear spiked covariance models established by
\cite{bai2012sample}. We prove a deterministic equivalent for the Stieltjes
transform and resolvent for any spectral argument separated from the support of
the limit spectral measure,
extending recent results for spectral arguments bounded away from the
positive real line \citep{chouard2022quantitative,chouard2023deterministic,schroder2023deterministic}.
\end{itemize}

\subsection{Related Works}

\paragraph{Eigenvalues of nonlinear random matrices.} 
Global convergence of the empirical eigenvalue distribution of
nonlinear kernel matrices has been studied in both proportional and polynomial
scaling regimes
\citep{el2010spectrum,cheng2013spectrum,fan2019spectral,lu2022equivalence,dubova2023universality}.
Building upon related techniques, recent works characterized the spectrum of the
CK matrix \citep{pennington2017nonlinear,louart2018random,peche2019note} and the
neural tangent kernel (NTK) matrix
\citep{montanari2020interpolation,adlam2020neural}, with generalizations to
deeper networks studied in \cite{fan2020spectra} and \cite{chouard2023deterministic}. 

\cite{benigni2022largest} gave a precise characterization
of the largest eigenvalue in a one-hidden-layer CK matrix when the input data
$\vX$ and weight matrix $\vW$ both have i.i.d.~entries, identifying possible
uninformative spike eigenvalues when the nonlinear activation is not an odd
function. \cite{guionnet2023spectral} and \cite{feldman2023spectral} recently
characterized spiked eigenstructure in models where an activation is applied to
a spiked Wigner matrix or rectangular information-plus-noise matrix entrywise,
for possibly growing spike sizes and activations having degenerate
information/Hermite coefficients.

\paragraph{Precise error analysis of NNs.} 
An important application of spectral analyses of the CK matrix is the precise
computation of generalization error of random features regression, 
first performed for two-layer models in proportional scaling regimes
\citep{louart2018random,mei2019generalization} and later extended to deep
random features models \citep{schroder2023deterministic,bosch2023precise} and
polynomial scaling regimes \citep{ghorbani2019linearized,xiao2022precise}. 
These risk analyses reveal a \textit{Gaussian equivalence principle}, where
generalization error coincides with that of a Gaussian covariates model,
and this equivalence has been extended to other settings of nonlinear
(regularized) empirical risk minimization 
\citep{hu2020universality,goldt2021gaussian,montanari2022universality}.

Going beyond random features, \cite{ba2022high} derived the
precise asymptotics of representation learning in a two-layer NN when the
first-layer weights are trained by one (or finitely many) gradient descent
steps; see also \cite{damian2022neural,ba2023learning,dandi2023learning}.
The computation follows from an information-plus-noise characterization
of the weight matrix due to a low-rank gradient update. \cite{moniri2023theory}
derived a corresponding information-plus-noise decomposition
of the CK matrix defined by the resulting trained weights, in an asymptotic
regime different from ours where the learning rate and spike eigenvalues
diverge. \cite{arous2023high2} examined the emerging spike
eigenstructure in the NN Hessian that arises during SGD training.

\paragraph{Eigenvalues of sample covariance matrices.} Asymptotic spectral
analyses of sample covariance matrices have a long history in random matrix
theory
\citep{marchenko1967,silverstein1995strong,silverstein1995empirical,bai1998no},
with the strongest known results in the linearly defined model
$\vX=\vZ\vSigma^{1/2}$, see e.g.\ \cite{bloemendal2014isotropic,knowles2017anisotropic}. Outside of this linear setting, \cite{srivastava2013covariance} and
\cite{chafai2018convergence} develop sharp bounds for the extremal eigenvalues
with isotropic population covariance, and \cite{bao2022extreme} develop
eigenvalue rigidity and Tracy-Widom fluctuation results for isotropic and log-concave distributions.

The spiked covariance model was introduced in \cite{johnstone2001distribution}.
\cite{baik2005phase,baik2006eigenvalues,paul2007asymptotics} initiated the study
of spiked eigenstructure and phase transition phenomena for spiked
covariance matrices with isotropic bulk covariance.
\cite{peche2006largest,benaych2011eigenvalues,benaych2012singular,capitaine2013additive,capitaine2018limiting}
studied spiked eigenstructure in related Wigner and information-plus-noise
models. Closely related to our work are the results of
\cite{bai2012sample} that characterize spike eigenvalues in linearly
defined models $\vX=\vZ\vSigma^{1/2}$ with general population covariance
$\vSigma$, and we extend this characterization to nonlinear settings.

\section{Results for neural network models}\label{sec:main_results}

\subsection{Propagation of signal through multi-layer neural networks}\label{sec:spike_CK}

Consider input features $\vX=[\vx_1,\ldots,\vx_n]\in\R^{d\times n}$, where
$\vx_i\in\R^d$ are independent samples. Define a $L$-hidden-layer feedforward neural network by
\begin{equation}\label{eq:NN}
    \vX_\ell=\frac{1}{\sqrt{d_\ell}}\sigma (\vW_\ell\vX_{\ell-1}) \in \R^{d_\ell
\times n} \qquad \text{ for } \ell=1\ldots,L
\end{equation}
with weight matrices $\vW_\ell \in \R^{d_\ell\times d_{\ell-1}}$, 
$\vX_0 \equiv \vX$ and $d_0 \equiv d$,
and a nonlinear activation function $\sigma:\R \to \R$ applied entrywise. The
Conjugate Kernel (CK) at each layer $\ell=1,\ldots,L$ is given by the Gram matrix
\begin{equation}\label{eq:K_L}
	\vK_\ell=\vX_\ell^\top\vX_\ell \in \R^{n \times n}.
\end{equation}

In the limit $n,d_0,\ldots,d_L \to \infty$ with
$n/d_\ell \to \gamma_\ell \in (0,\infty)$
for each $\ell=0,\ldots,L$, under deterministic
conditions for the input data $\vX$ and for random weight matrices
$\vW_1,\ldots,\vW_L$ as specified below, it is shown in \cite{fan2020spectra}
that the empirical eigenvalue distribution
$\widehat{\mu}_\ell$ of $\vK_\ell$ for each $\ell=1,\ldots,L$ satisfies
the weak convergence
\begin{equation}\label{eq:CKweakconvergence}
\widehat{\mu}_\ell:=\frac{1}{n}\sum_{i=1}^n \delta_{\lambda_i(\vK_\ell)}
\to \mu_\ell \text{ a.s.}
\end{equation}
for limit measures $\mu_1,\ldots,\mu_L$ defined 
as follows: Let $\mu_0$ be the limit eigenvalue distribution of the input gram
matrix $\vK_0=\vX^\top \vX$ (c.f.\ Assumption \ref{assump:data}).
Then, for $\ell=1,\ldots,L$, let
\begin{equation}\label{eq:nu_L_def}
	\nu_{\ell-1}=b_\sigma^2 \otimes \mu_{\ell-1} \oplus (1-b_\sigma^2)
\end{equation}
denote the law of $b_\sigma^2 \mathsf{x}+(1-b_\sigma^2)$ when $\mathsf{x}
\sim \mu_{\ell-1}$ and $b_\sigma:=\E_{\xi \sim \cN(0,1)}[\sigma'(\xi)]$,
and define
\begin{align}\label{eq:mu_L_def}
\mu_\ell=\rho^\MP_{\gamma_\ell}\boxtimes \nu_{\ell-1}.
\end{align}
Here, $\rho_\gamma^\MP \boxtimes \nu$ denotes the \emph{deformed Marcenko-Pastur
law} describing the limit eigenvalue distribution of a sample covariance
matrix with limit dimension ratio $\gamma \in (0,\infty)$ and population
spectral measure $\nu$, and we review its definition in Appendix
\ref{appendix:background}.

In this section, we provide a precise quantitative characterization of the spike
eigenvalues and eigenvectors of $\vK_\ell$ for each $\ell=1,\ldots,L$ when the
input data $\vX$ has a fixed number of spike singular values of bounded
magnitude. We assume the following conditions for the random weights, input
data, and activation.

\begin{assumption}\label{assump:NNasymptotics}
The number of layers $L \geq 1$ is fixed, and
$n,d_0,\ldots,d_L \to \infty$ such that
\[n/d_\ell \to \gamma_\ell \in (0,\infty) \text{ for each } \ell=0,\ldots,L.\]
The weights $\vW_1,\ldots,\vW_L$ have entries
$[\vW_\ell]_{ij} \overset{iid}{\sim} \cN(0,1)$, independent of each other and
of $\vX$.
\end{assumption}

\begin{defi}\label{def:ortho}
A feature matrix $\vX\in\R^{d\times n}$ is {\bf $\tau_n$-orthonormal} if
\[\big|\|\vx_{\alpha}\|_2-1\big| \leq \tau_n,
\qquad \big|\|\vx_{\beta}\|_2-1\big| \leq \tau_n,
\qquad \big|\vx_{\alpha}^\top \vx_{\beta}\big| \leq \tau_n\]
for all pairs $\alpha\neq \beta\in [n]$, where
$\{\vx_{\alpha}\}_{\alpha=1}^n$ are the columns of $\vX$.
\end{defi}

\begin{assumption}\label{assump:data}
For some $\tau_n>0$ such that $\lim_{n \to \infty} \tau_n \cdot n^{1/3}=0$,
$\vX \equiv \vX_0$ is $\tau_n$-orthonormal almost surely for all large $n$.
Furthermore, $\vK_0=\vX^\top \vX$ has eigenvalues
$\lambda_1(\vK_0),\ldots,\lambda_n(\vK_0)$
(not necessarily ordered by magnitude) such that
for some fixed $r \geq 0$, as $n,d \to \infty$,
\begin{enumerate}[label=(\alph*)]
\item There exists a compactly supported probability measure $\mu_0 $ on $[0,\infty)$ 
such that
$$\frac{1}{n-r}\sum_{i=r+1}^n \delta_{\lambda_i(\vK_0)} \to \mu_0
\text{ weakly a.s.}$$ 
and for any fixed $\eps>0$, almost surely for all large $n$,
\[\lambda_i(\vK_0) \in \supp{\mu_0}+(-\eps,\eps) \text{ for all }
i \geq r+1.\]   
\item  There exist distinct values $\lambda_1,\ldots,\lambda_r>0$ with
$\lambda_1,\ldots,\lambda_r \not\in\supp{\mu_0}$ such that
\[\lambda_i(\vK_0) \to \lambda_i\quad \text{ a.s.\ for each }\quad  i=1,\ldots,r.\]
\end{enumerate}
\end{assumption}

\begin{assumption}\label{assump:sigma}
The activation $\sigma:\R \to \R$ is twice differentiable with
$\sup_{x \in \R} |\sigma'(x)|,|\sigma''(x)| \leq \lambda_\sigma$
for some $\lambda_\sigma\in (0,\infty)$.
Under $\xi \sim \cN(0,1)$, we have
$\E[\sigma(\xi)]=0$ and $\E[\sigma^2(\xi)]=1$.
Furthermore,
\begin{equation}\label{eq:conditionsigma}
b_\sigma:=\E[\sigma'(\xi)] \neq 0, \qquad
\E[\sigma''(\xi)]=0.
\end{equation}
\end{assumption}

Assumption \ref{assump:NNasymptotics} defines the linear-width asymptotic
regime. Assumption \ref{assump:data} requires an orthogonality condition
for the input features that is similar to \cite[Definition 3.1]{fan2020spectra},
and also codifies our spiked eigenstructure assumption for the input data.
We briefly comment on \eqref{eq:conditionsigma}
in Assumption \ref{assump:sigma}: The condition $b_\sigma \neq 0$ ensures that
the linear component of $\sigma(\cdot)$ is non-degenerate; if $b_\sigma=0$,
then spiked eigenstructure does not propagate across the NN layers in our
studied regime of bounded spike magnitudes. The condition $\E[\sigma''(\xi)]=0$
ensures that $\vK_\ell$ does not have uninformative spike eigenvalues;
otherwise, as shown in \cite{benigni2022largest}, $\vK_\ell$ may have
spike eigenvalues even when the input $\vK_0$ has no spiked structure.
We assume $\E[\sigma''(\xi)]=0$ for clarity, to avoid characterizing also such
uninformative spikes across layers. This condition holds, in particular, for odd
activation functions $\sigma(\cdot)$ such as $\mathrm{tanh}$.

The following theorem first extends \cite[Theorem 3.4]{fan2020spectra} by
affirming that the weak convergence statement \eqref{eq:CKweakconvergence}
holds under the above assumptions, and furthermore, each $\vK_\ell$ has
no outlier eigenvalues outside its limit spectral support when the input
$\vK_0$ has no spike eigenvalues.

\begin{theorem}\label{thm:no_outlier_ck}
Suppose Assumptions \ref{assump:NNasymptotics}, \ref{assump:data},
and \ref{assump:sigma} hold. Then for each $\ell=1,\ldots,L$,
\eqref{eq:CKweakconvergence} holds weakly a.s.\ as $n \to \infty$. 
Furthermore, if the number of spikes is $r=0$ in Assumption \ref{assump:data},
then for any fixed $\eps>0$, almost surely for all large $n$,
\begin{align*}
\vK_\ell \text{ has no eigenvalues outside }
\supp{\mu_\ell}+(-\eps,\eps).
\end{align*}
\end{theorem}

The main result of this section characterizes the eigenvalues of $\vK_\ell$
outside $\supp{\mu_\ell}$ when $r\geq 1$. To describe this characterization,
define for each $\ell=1,\ldots,L$ the domain
\begin{align*}
\cT_\ell&=\{-1/\lambda:\lambda \in \supp{\nu_{\ell-1}}\}
\end{align*}
where $\nu_{\ell-1}$ is defined by \eqref{eq:nu_L_def},
and define $z_\ell,\varphi_\ell:(0,\infty) \setminus \cT_\ell \to \R$ by
\begin{align}
z_\ell(s)={-}\frac{1}{s}+\gamma_\ell \int\frac{\lambda}{1+\lambda s}
\nu_{\ell-1}(d\lambda), \qquad
\varphi_\ell(s)={-}\frac{sz_\ell'(s)}{z_\ell(s)}.
\label{eq:zell}
\end{align}
It is known from the results of \cite{bai2012sample} and
\cite[Chapter 11]{yao2015sample} that these are precisely the functions
that characterize the spike eigenvalues and eigenvectors in linear spiked
covariance models. Set
\[\cI_0=\{1,\ldots,r\}, \qquad
		s_{i,0}=-\frac{1}{b_\sigma^2\lambda_i+(1-b_\sigma^2)}
\text{ for } i \in \cI_0,\]
where $\lambda_i$ and $b_\sigma$ are defined in Assumptions~\ref{assump:data}
and~\ref{assump:sigma} respectively. Here, $\cI_0$ records the indices of the
spike eigenvalues of the input Gram matrix $\vK_0$.
Then define recursively for $\ell=1,\ldots,L$
\begin{equation}\label{eq:index_set}
\cI_\ell=\Big\{i \in \cI_{\ell-1}:z_\ell'(s_{i,{\ell-1}})>0\Big\},
\qquad
s_{i,\ell}=-\frac{1}{b_\sigma^2z_\ell(s_{i,{\ell-1}})+(1-b_\sigma^2)}
\text{ for } i \in \cI_\ell.
\end{equation}
The condition $z_\ell'(s_{i,\ell-1})>0$ describes the ``phase transition''
phenomenon for spike eigenvalues in this model, where spikes
$i \in \cI_{\ell-1}$ with
$z_\ell'(s_{i,\ell-1})>0$ induce spike eigenvalues in the CK matrix $\vK_\ell$
of the next layer, while spikes with $z_\ell'(s_{i,\ell-1}) \leq 0$ are
absorbed into the bulk spectrum of $\vK_\ell$.

\begin{theorem}\label{thm:ck_spike}
Suppose Assumptions~\ref{assump:NNasymptotics}, \ref{assump:data}, and
\ref{assump:sigma} hold.
Then for each $\ell=1,\ldots,L$:
\begin{enumerate}[label=(\alph*)]
\item $s_{i,\ell-1} \in (0,\infty) \setminus \cT_\ell$ for each $i \in \cI_{\ell-1}$,
so $z_\ell(s_{i,\ell-1})$ and $\cI_\ell$ are well-defined. Furthermore,
if $i \in \cI_\ell$ (i.e.\ if $z_\ell'(s_{i,\ell-1})>0$) then
$z_\ell(s_{i,\ell-1})>0$ and $\varphi_\ell(s_{i,\ell-1})>0$.
\item For any fixed and sufficiently small $\eps>0$, almost surely for all large $n$,
there is a 1-to-1 correspondence between the eigenvalues of $\vK_\ell$
outside $\supp{\mu_\ell}+(-\eps,\eps)$ and $\{i:i\in\cI_\ell\}$.
Denoting these eigenvalues of $\vK_\ell$ by
$\{\widehat \lambda_{i,\ell}:i\in\cI_\ell\}$,
for each $i\in \cI_\ell$ as $n \to \infty$,
\[\widehat \lambda_{i,\ell} \to z_\ell(s_{i,\ell-1}) \text{ a.s.}\]
\item Let $\widehat \vv_{i,\ell}$ be a
unit-norm eigenvector of $\vK_\ell$ corresponding to its eigenvalue
$\widehat \lambda_{i,\ell}$, and let $\vv_j$ be a
unit-norm eigenvector of $\vK_0$ corresponding to its spike
eigenvalue $\lambda_j(\vK_0)$.
Then for each $i \in \cI_\ell$ and $j \in \cI_0$, as $n \to \infty$,
	\[|\widehat \vv_{i,\ell}^\top \vv_j |^2 \to  \prod_{k=1}^\ell
\varphi_{k}(s_{i,k-1}) \cdot \bI\{i=j\} \text{ a.s.}\]
Moreover, for each $i \in \cI_\ell$ and any
unit vector $\vv\in\R^n$ independent of $\vW_1,\ldots,\vW_\ell$,
	\[|\widehat \vv_{i,\ell}^\top \vv |^2-\prod_{k=1}^\ell
\varphi_{k}(s_{i,k-1}) \cdot  |\vv_i^\top \vv |^2 \to 0 \text{ a.s.}\]
\end{enumerate}
\end{theorem}
\vspace{-4mm}
We present the following corollary as a concrete example in which the
assumptions of the theorem are satisfied. The corollary encompasses, for
instance, Gaussian mixture models with a fixed number $r$ of balanced classes,
each class having $\Theta(n)$ samples.
\begin{coro}\label{cor:gmm}
Suppose the input data $\vX$ is itself a low-rank signal-plus-noise matrix
\begin{equation}
	\label{eq:gmm}
\vX=\sum_{i=1}^r \theta_i\va_i\vb_i^\top+\vZ \in \R^{d\times n}
\end{equation}
where $\theta_1,\ldots,\theta_r>0$ are fixed distinct signal strengths,
$\va_1,\ldots,\va_r \in \R^d$ and $\vb_1,\ldots,\vb_r \in \R^n$ are
orthonormal sets of unit vectors, and
$\vZ$ has i.i.d.\ $\cN(0,1/d)$ entries. Assume that $\vb_1,\ldots,\vb_r$
satisfy the $\ell_\infty$-delocalization condition: for any sufficiently small
$\eps>0$ and all large $n$,
\[\max_{i=1}^r \|\vb_i\|_\infty<n^{-1/2+\eps}.\]
Define $\varphi_\ell(\cdot)$ and $s_{i,\ell-1}$ by \eqref{eq:zell}
and \eqref{eq:index_set}, with the initial measures $\mu_0=\rho_{\gamma_0}^\MP$
and $\nu_0=b_\sigma^2 \otimes \mu_0 \oplus (1-b_\sigma^2)$
and initial spike values
$\lambda_i=(1+\theta_i^2)(\gamma_0+\theta_i^2)/\theta_i^2$ for $i\in\cI_0$.

Then for each $\ell=1,\ldots,L$, $\vK_\ell$ has a spike eigenvalue corresponding
to the input signal component $\theta_i$ if and only if
$\theta_i>\gamma_0^{1/4}$ and $i \in \cI_\ell$. In this case,
its corresponding unit eigenvector $\widehat \vv_{i,\ell}$ satisfies, as $n \to
\infty$,
\begin{equation}\label{eq:gmmalignment}
|\widehat \vv_{i,\ell}^\top \vb_i|^2 \to
\prod_{k=1}^\ell \varphi_k(s_{i,k-1}) \cdot
\left(1-\frac{\gamma_0(1+\theta_i^2)}{\theta_i^2(\theta_i^2+\gamma_0)}\right)
\text{ a.s.}
\end{equation}
\end{coro}

\begin{figure}[t]
\centering
\begin{minipage}[t]{0.33\linewidth}
\centering
{\includegraphics[width=\textwidth]{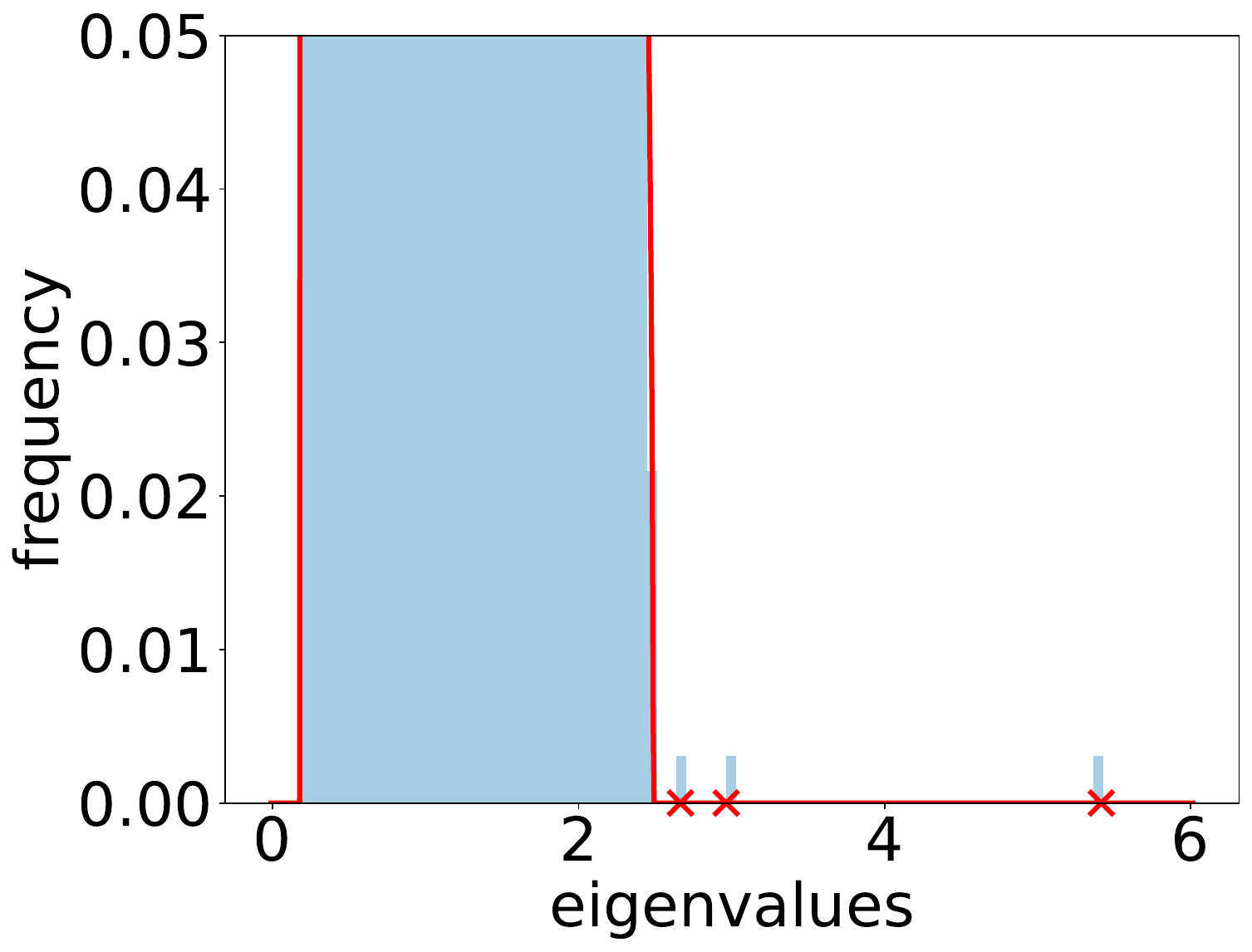}}  \\
\small (a) Spectrum of $\vK_0$ 
\end{minipage}%
\begin{minipage}[t]{0.33\linewidth}
\centering 
{\includegraphics[width=\textwidth]{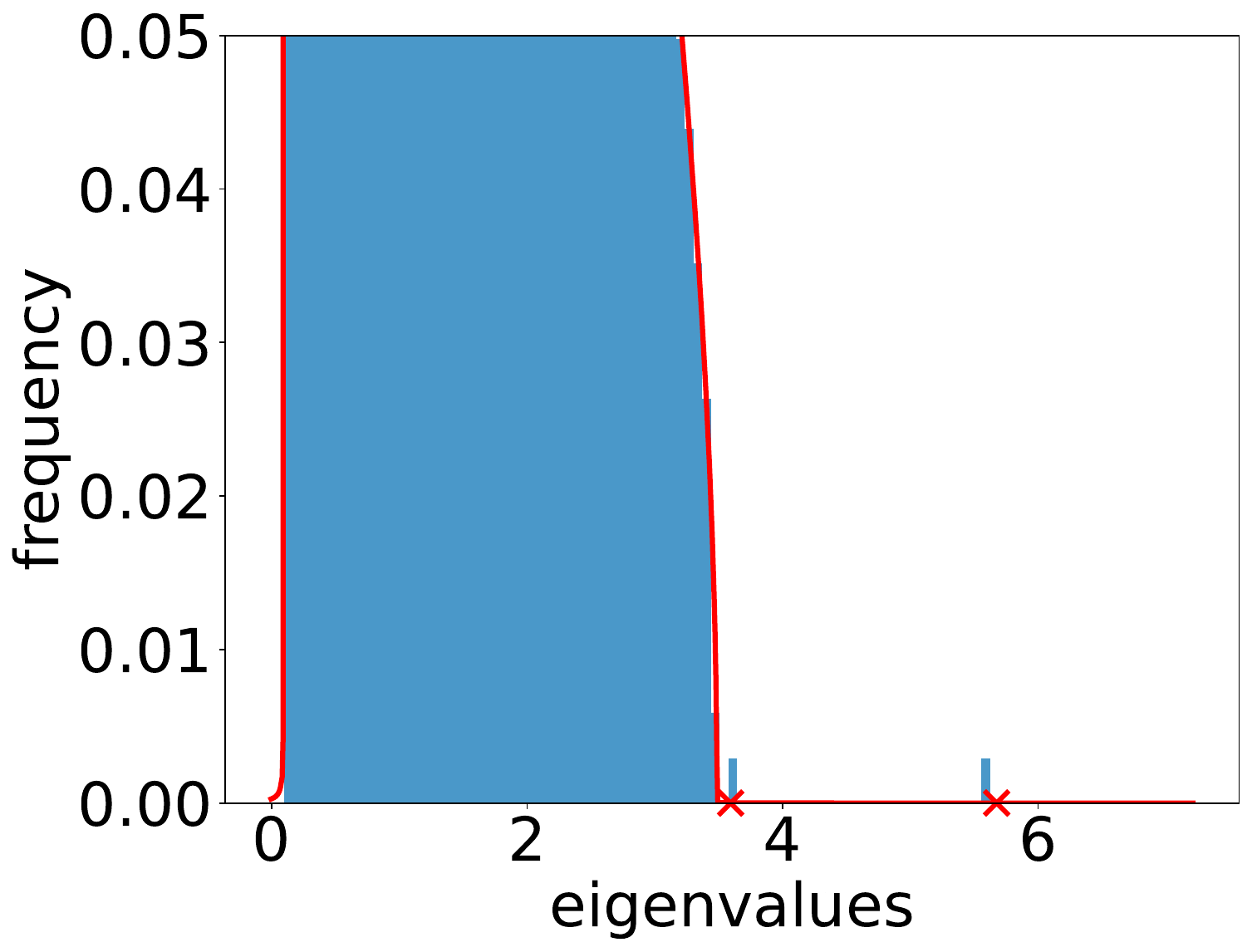}} \\ 
\small (b) Spectrum of $\vK_{1}$
\end{minipage}%
\begin{minipage}[t]{0.33\linewidth}
\centering 
{\includegraphics[width=\textwidth]{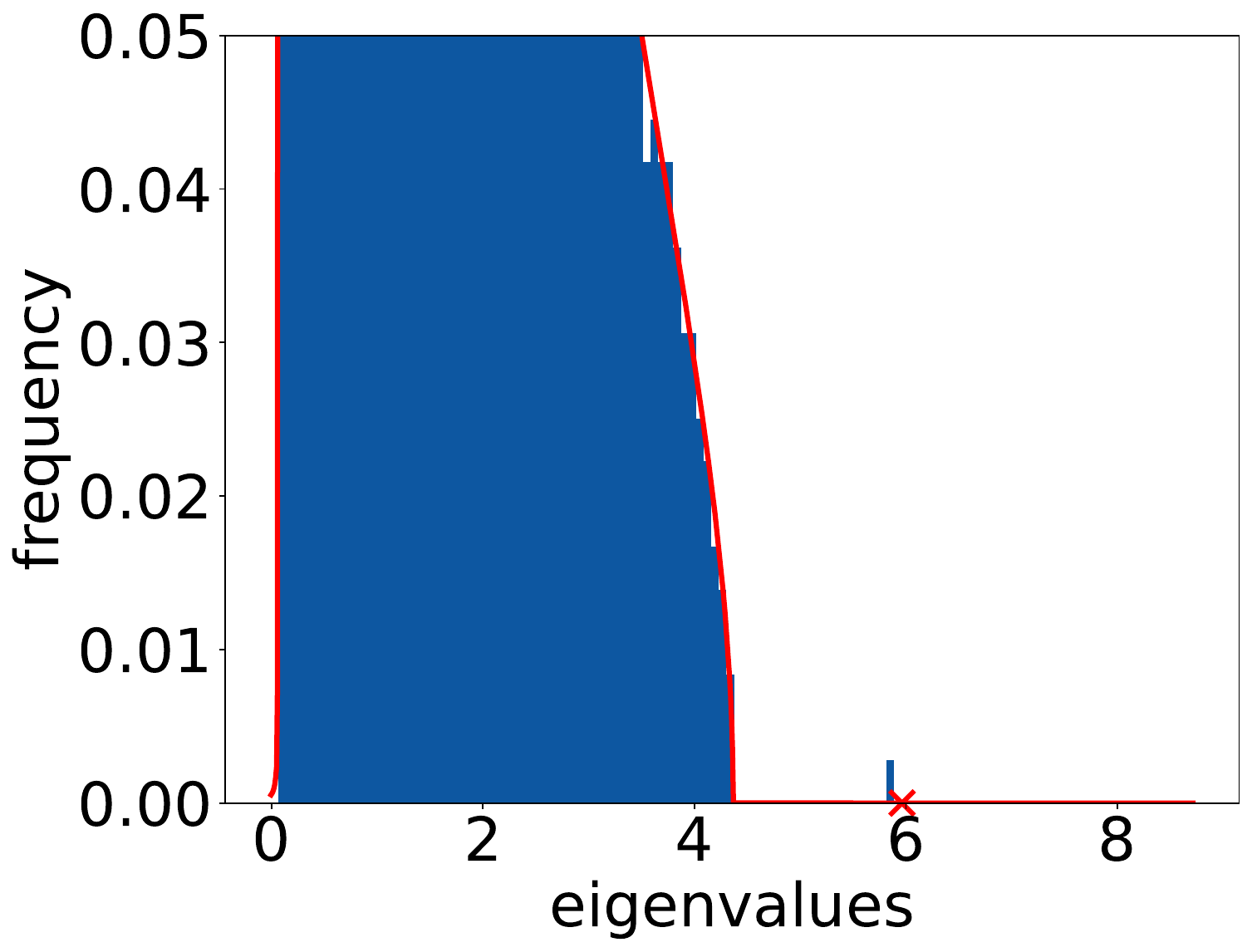}} \\ 
\small (c) Spectrum of $\vK_{2}$
\end{minipage}%
\caption{\small  Spectra of three-layer CK matrices defined by \eqref{eq:K_L} with $n = 5000$,
$d_0=d_1=d_2 = 15000$, and $\sigma\propto \arctan$. Input data is a GMM satisfying \eqref{eq:gmm} with $r=3$, $\theta_1=2.0$, $\theta_2=1.18$,
and $\theta_3=1.0$. (a)-(c) are theoretically predicted (red) and empirical (blue)
bulk distributions and spikes of $\vK_\ell$ for $\ell=0,1,2$.}
\label{fig:GMM}  
\end{figure} 
\paragraph{Numerical illustration.}
A simple illustration of this result for a 3-component Gaussian mixture model
is provided in Figure \ref{fig:GMM}. We note that
$\cI_L \subseteq \cdots \subseteq \cI_0$ and $\varphi_\ell(s_{i,\ell-1})
\in (0,1)$, so the number of spike eigenvalues of $\vK_\ell$
induced by $\vK_0$ and the alignment of the spike eigenvectors of $\vK_\ell$
with the true class label vectors $\{\vb_i\}_{i=1}^r$ are both non-increasing in the network depth, see also Figure~\ref{fig:training}. In other words, at random initialization, the input signal diminishes as the depth of the NN increases. 

In Figure~\ref{fig:training} we highlight two remedies to this ``curse of depth'' at random initialization. 
\begin{itemize}[leftmargin=*]
    \item In Figure~\ref{fig:training}(a) we observe that when the width of NN becomes larger, alignment between the leading eigenvector of $\vK_\ell$ at random initialization and the signal can be preserved across a larger depth. This illustrates the benefit of overparameterization by increasing the network width. 
    \item In Figure~\ref{fig:training}(b) we observe that gradient descent training on the weight matrices also restores and even amplifies the informative signal in the CK matrix of each layer; specifically, after 50 steps of GD training (yellow curve), the alignment between the class labels and the leading eigenvector of $\vK_\ell$ may increase through depth. This demonstrates the benefit of gradient-based feature learning. In Section~\ref{sec:trained_CK} we precisely quantify this improved alignment due to gradient descent in a simplified two-layer setting. 
\end{itemize}

\begin{figure}[t]
\centering
\begin{minipage}[t]{0.42\linewidth}
\centering 
{\includegraphics[width=0.8\textwidth]{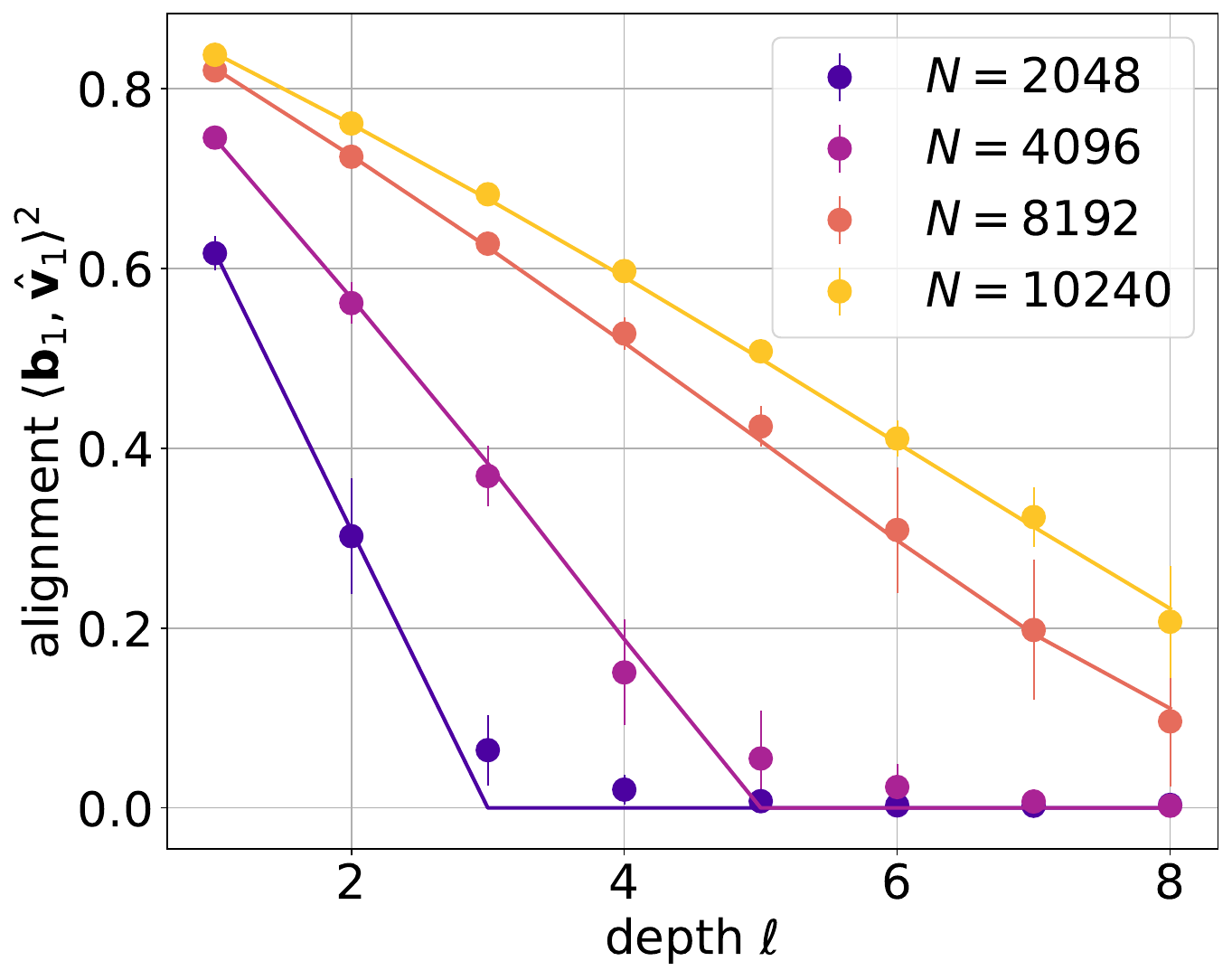}} \\ 
\small (a) Effect of width on alignment. 
\end{minipage}    
\begin{minipage}[t]{0.42\linewidth}
\centering
{\includegraphics[width=0.83\textwidth]{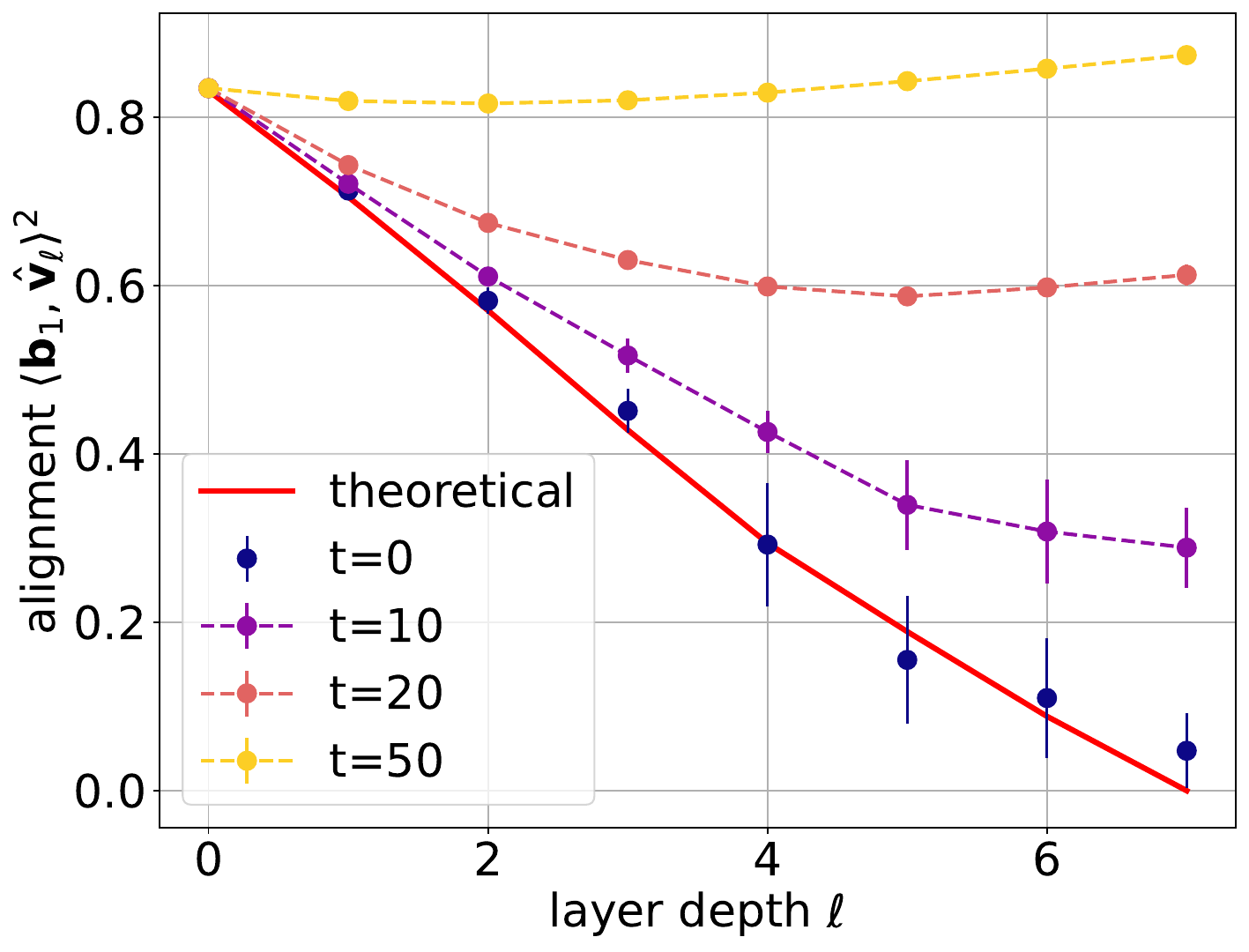}}  \\
\small (b) Effect of GD training. 
\end{minipage} 
\caption{\small  We consider multiple-layer NNs in \eqref{eq:NN} with $ \sigma\propto \tanh$ on Gaussian mixture data \eqref{eq:gmm} for $r=1$, and compute the alignment between the largest eigenvector of the CK matrix $\vK_\ell$ with genuine signal $\vb_1$ (class labels) for different layer $\ell$. 
(a)  NNs at random initialization with varying hidden widths $N=2048,4096,8192,10240$. (b) NNs trained by gradient descent with learning rate $\eta=0.1$ for varying steps $T=0,10,20,50$; we use the $\mu$-parameterization \citep{yang2020feature} to encourage feature learning. $\theta_1$ is $2.5$ and  $1.8$ for (a) and (b), respectively. Dots are empirical values (over 10 runs) and solid curves represent theoretical predictions at random initialization from Theorem~\ref{thm:ck_spike}.  
}   
\label{fig:training}  
\end{figure}

\subsection{CK matrix after $O(1)$ steps of gradient descent}\label{sec:trained_CK}

The preceding section studied the spike eigenstructure of the CK induced by
low-rank structure in the input data.
Here, focusing on a two-layer model, we study an alternative setting 
where spiked structure arises instead
in the weight matrix $\vW$ from gradient descent training.

We consider an early training regime studied in \cite{ba2022high}, with
a width-$N$ two-layer feedforward NN,
\begin{align} 
f_{\text{NN}}(\vx)  
= \frac{1}{\sqrt{N}} \sum_{i=1}^N a_i\sigma(\langle \vx,\vw_i\rangle)
= \frac{1}{\sqrt{N}} \sigma(\vx^\top\vW)\va.
\label{eq:two-layer-nn}
\end{align} 
Here $\vx\in\R^d$ is the input, and
$\vW=[\vw_1,\ldots,\vw_N]\in\R^{d\times N}$ and $\va\in\R^N$ are the network
weights. For clarity of the subsequent discussion, we will
transpose the notation for $\vX$ and $\vW$ from the preceding section,
and incorporate a $1/\sqrt{d}$ scaling into $\vW$ rather than into the input
data $\vX$.

Given are an input feature matrix
$\vX=[\vx_1,\ldots,\vx_n]^\top \in \R^{n\times d}$ and labels $\vy \in \R^n$
for $n$ samples, where $y_i=f_*(\vx_i)+\text{noise}$.
We consider the training of first-layer weights $\vW$
to minimize the mean squared error
\[\cL(\vW) = {\frac{1}{2n}}\sum_{i=1}^n (f_{\text{NN}}(\vx_i)-y_i)^2,\]
fixing the second-layer weight vector $\va$.
From a random initialization $\vW_0 \in \R^{d \times N}$, and
over $T$ steps with learning rates $\eta_1,\ldots,\eta_T$ scaled by
$\sqrt{N}$, the gradient descent (GD) updates take the form
\begin{equation}\label{eq:trained_W}
	\vW_{\!t+1} = \vW_{\!t} + \eta_{t+1} \sqrt{N}\cdot\vG_t,
\qquad \vG_t=-\nabla\cL(\vW_t).
\end{equation}
Of interest is the information about the label function $f_*$ that is learned by
$\vW_{\text{trained}} \equiv \vW_T$, which may be characterized by the spectral
alignment of the CK matrix with the class label vector on independent
test data $(\tilde \vX,\tilde \vy)$.
This use of independent test data may be understood as a pre-training
setup, also considered previously
in \cite{ba2022high,moniri2023theory} and studied for real-world data in \cite{wei2022more}.

It was shown in \cite{ba2022high} that in a training regime
with initialization $\|\vW_0\| \asymp 1$ such that
$|f_{\text{NN}}(\vx_i)| \ll 1$ for each $i=1,\ldots,N$,
and with learning rates $\eta_1,\ldots,\eta_T \asymp
1$ for a fixed number $T$ of GD steps, the weight matrix $\vW$ undergoes a
change during training that is $O(1)$ in operator norm and approximately rank-1,
\[\vW_{\text{trained}}\approx \vW_0+\frac{\eta b_\sigma}{n}\vX^\top\vy\va^\top \qquad
\text{ where } \qquad \eta=\sum_{t=1}^T \eta_t.\]
Moreover, \cite[Conjecture 4]{ba2022high} conjectured that for the CK matrix
\begin{equation}\label{eq:trained_CK}
\vK=\frac{1}{N} \sigma(\tilde \vX\vW_{\text{trained}})\sigma(\tilde
\vX\vW_{\text{trained}})^\top \in \R^{n \times n}
\end{equation}
defined by the pre-trained weights and test data $\tilde \vX$,
the resulting spike eigenvalue and the alignment of its spike eigenvector with the test labels $\tilde
\vy \in \R^n$ are accurately predicted by a Gaussian equivalent model.
Our main result of this section is an affirmative verification of this
conjecture and precise characterization of the spike eigenstructure of $\vK$,
in the following representative setting.

\begin{assumption}\label{assum:GD}
For a two-layer NN in \eqref{eq:two-layer-nn} with GD training defined by \eqref{eq:trained_W}, we assume that~
\begin{enumerate}[label=(\alph*)]
\item $n,d,N \to \infty$ such that $N/d \to\gamma_0 \in
(0,\infty)$ and $N/n \to \gamma_1 \in (0,\infty)$. 
\item Training features $\vX=[\vx_1,\ldots,\vx_n]^\top \in \R^{n \times d}$
have entries $[\vX]_{ij} \overset{iid}{\sim} \cN(0,1)$, training labels $\vy \in
\R^n$ have entries
$y_i = \sigma_*(\vbeta_*^\top\vx_i)+\eps_i$ where $\vbeta_*\in\R^d$ is a
deterministic unit vector and
$\eps_i \overset{iid}{\sim} \cN(0,\sigma_\eps^2)$, and test data
$(\tilde \vX,\tilde \vy)$ is an independent copy of $(\vX,\vy)$.
\item The NN activation $\sigma:\R \to \R$ and label function
$\sigma_*:\R \to \R$ both satisfy Assumption~\ref{assump:sigma}, with
$b_\sigma:=\E[\sigma'(\xi)] \neq 0$ and
$b_{\sigma_*}:=\E[\sigma_*'(\xi)] \neq 0$.
\item The weight initializations satisfy
$[\vW_0]_{ij}\iid\cN(0,1/d)$ and $a_j \iid \cN(0,1/N)$.
\item The number of iterations $T$ and learning rates $\eta_1,\ldots,\eta_T$ are
fixed independently of $n,d,N$.
\end{enumerate}
\end{assumption}

Under these assumptions, the following theorem characterizes the spike eigenvalue of the CK matrix and the alignment between the corresponding eigenvector and the test labels, 
as a function of the learning rate $\eta_t$ and the number of gradient descent steps $T$. 

\begin{figure}[t]
\centering
\begin{minipage}[t]{0.46\linewidth}
\centering
{\includegraphics[height=0.6\textwidth]{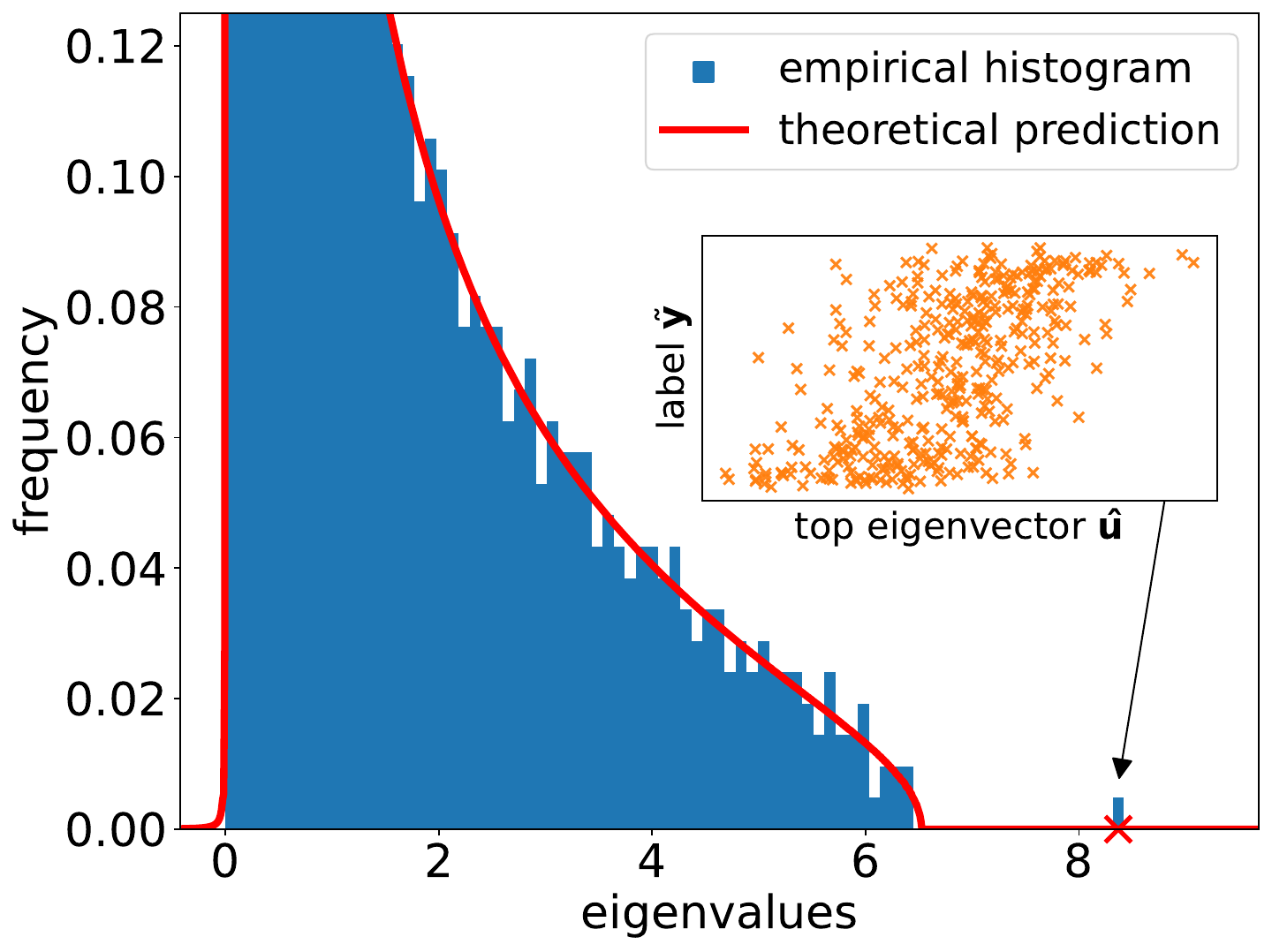}}  \\
\small (a) Spectrum of the updated CK. 
\end{minipage}
\begin{minipage}[t]{0.46\linewidth} 
\centering 
{\includegraphics[height=0.6\textwidth]{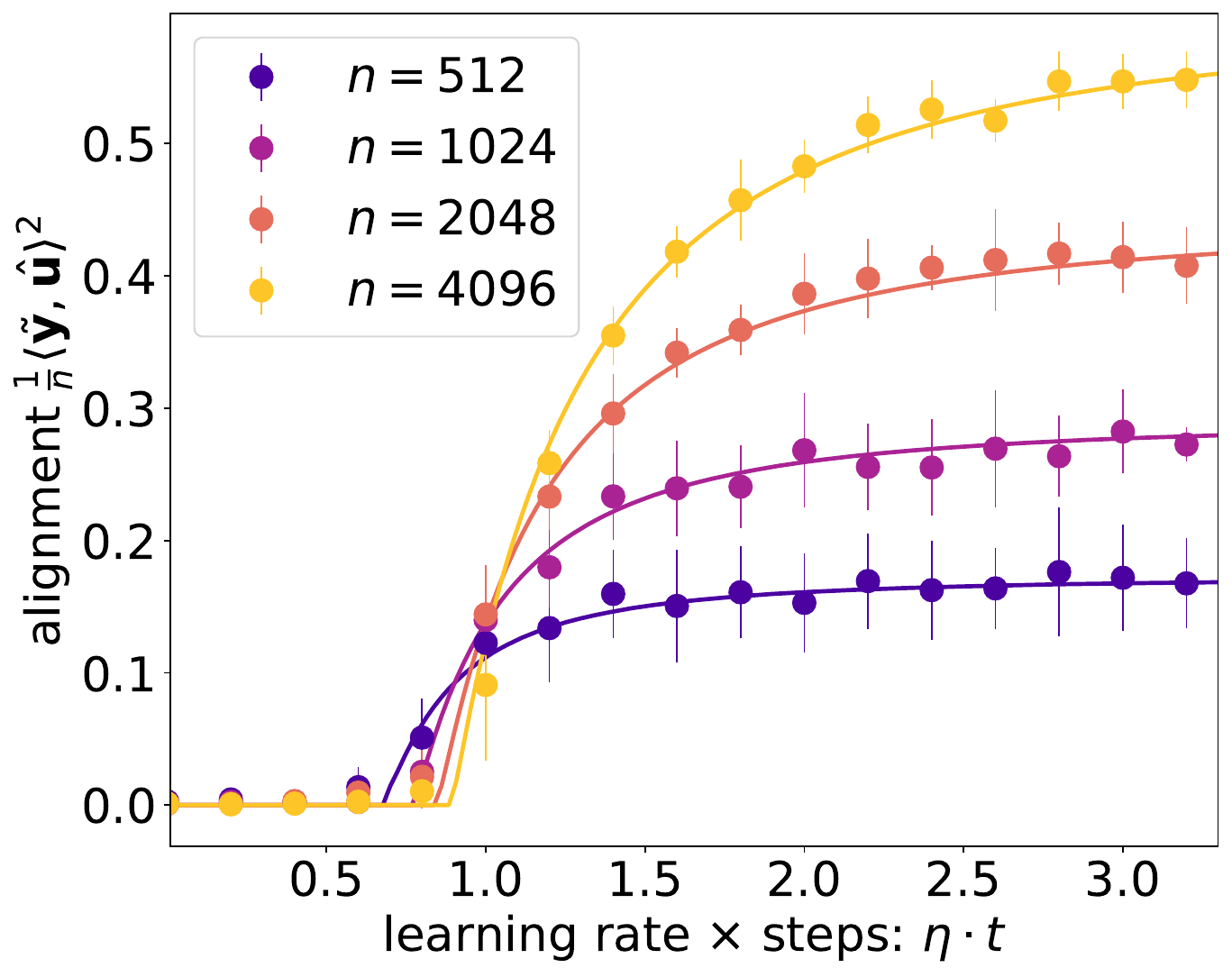}} \\ 
\small (b) Eigenvector alignment of the updated CK. 
\end{minipage}   
\caption{\small $(a)$ We set $n=2000, d=1600,N=2400,\eta\cdot t = 2$, and $\sigma=\sigma_*=\text{erf}$. 
$(b)$ We set $d=2048,N=1024,\eta=0.2$, $\sigma=\text{tanh}, \sigma_*=\text{SoftPlus}$, and vary the sample size $n$ and number of GD steps $t$; 
dots represent empirical simulations (over 10 runs) and solid curves are theoretical predictions from Theorem~\ref{thm:gd_spike}. 
 }   
\label{fig:GD}  
\end{figure} 

\begin{theorem}\label{thm:gd_spike}
Suppose that Assumption \ref{assum:GD} holds, and set
$\eta=\sum_{t=1}^T \eta_t$. Define  
\begin{equation}\label{def:theta_12}
\theta_1= b_\sigma\eta\cdot\sqrt{(\gamma_1/\gamma_0)(1+\sigma_\eps^2) +
b_{\sigma_*}^{2}}, \qquad
\theta_2=b_\sigma b_{\sigma_*}\eta.
\end{equation} 
Let $z_1(\cdot)$ and $\varphi_1(\cdot)$ be defined by \eqref{eq:zell} with
$\gamma_1$ and $\nu_0=b_\sigma^2\otimes\rho^{\MP}_{\gamma_0} \oplus (1-b_\sigma^2)$, and set
\[\lambda_1=b_\sigma^2\frac{(1+\theta_1^2)(\gamma_0+\theta_1^2)}{\theta_1^2}+1-b_\sigma^2.\]
Then $\vK$ defined by \eqref{eq:trained_CK} has a spike eigenvalue if and only
if $\theta_1>\gamma_0^{1/4}$ and $z'(-1/\lambda_1)>0$. In this case,
$\lambda_{\max}(\vK) \to \gamma_1^{-1}z(-1/\lambda_1)$ a.s.,
and the leading unit eigenvector $\widehat \vu\in\R^n$ of $\vK$ satisfies
\begin{equation}\label{eq:gd_alignment}
\frac{1}{\sqrt{n}}|\tilde{\vy}^\top \widehat\vu|\to b_\sigma b_{\sigma_*}
\frac{\sqrt{z(-1/\lambda_1)\varphi(-1/\lambda_1)}}{\lambda_1}\cdot
\frac{\theta_2\sqrt{(\theta_1^4-\gamma_0)(\gamma_0+\theta_1^2)}}{\theta_1^3}>0
\text{ a.s.}
\end{equation}
\end{theorem}

\paragraph{Numerical illustration.}
Figure~\ref{fig:GD} empirically validates the predictions of
Theorem~\ref{thm:gd_spike}, for a two-layer NN trained with a small number of
GD steps.
Figure \ref{fig:GD}(a) shows that one spike eigenvalue emerges over training in the test-data CK, the location of which is accurately predicted by Theorem~\ref{thm:gd_spike}; 
moreover, the leading eigenvector $\widehat{\vu}$ aligns with the labels $\tilde{\vy}$. 
This is quantified in Figure~\ref{fig:GD}(b), where above a phase transition
threshold, the alignment $\langle\widehat{\vu},\tilde{\vy}\rangle^2$ (predicted
by \eqref{eq:gd_alignment}) increases with the learning rate or number of GD steps; 
in addition, alignment also increases with the training set size $n$.
Compared with random initialization ($\eta=0$), this illustrates that training
improves the NN representation, and the test-data CK contains information on the label function $f_*$. 

\section{Analysis of a nonlinear spiked covariance model}\label{sec:proof_overview}

The results of Sections~\ref{sec:spike_CK} and~\ref{sec:trained_CK} rest on an
analysis of spiked eigenstructure in a general nonlinear spiked covariance
model. We describe the assumptions and statement of this general result
informally here, deferring formal and more quantitative statements
to Appendix \ref{sec:spiked_sample_covariance}.

Let
$\vG=\frac{1}{\sqrt{N}}[\vg_1,\ldots,\vg_N]^\top\in\R^{N\times n}$
have independent rows $\vg_1,\ldots,\vg_N \in \R^n$ with mean 0 and common
covariance $\vSigma \in \R^{n \times n}$. We assume that the law of $\vg_i$
satisfies concentration of quadratic forms $\vg_i^\top \vA\vg_j$, but has
otherwise arbitrary dependence across coordinates.
As $n,N \to \infty$ with $n/N \to \gamma \in (0,\infty)$,
the eigenvalues of $\vSigma$ satisfy
\[\lambda_i(\vSigma) \to \lambda_i \text{ for } i=1,\ldots,r, \qquad
\frac{1}{n-r}\sum_{i=r+1}^n \delta_{\lambda_i(\vSigma)} \to \nu \text{ weakly},\]
for fixed spike values $\lambda_1,\ldots,\lambda_r>0$ and a deterministic
limit spectral law $\nu$. Then the empirical spectral law of the sample
covariance matrix $\vK=\vG^\top \vG$ satisfies
\[\frac{1}{n}\sum_{i=1}^n \delta_{\lambda_i(\vK)} \to \mu=\rho_\gamma^{\MP}
\boxtimes \nu \text{ weakly a.s.}\]
Under these assumptions, let us define
\[z(s)={-}\frac{1}{m}+\gamma\int \frac{\lambda}{1+\lambda s}d\nu(\lambda),
\qquad \varphi(s)=\frac{z'(s)}{(-1/s)z(s)}.\]
\vspace{-\baselineskip}

\begin{theorem}[informal]\label{thm:informal}~
\begin{enumerate}[label=(\alph*)]
\item If $r=0$, then all eigenvalues of $\vK$ converge to
$\supp{\mu} \cup \{0\}$. More generally for $r \geq 0$, the eigenvalues of
$\vK$ asymptotically
separated from $\supp{\mu} \cup \{0\}$ are in 1-to-1 correspondence with
$\cI=\{i:z'(-1/\lambda_i)>0\}$, and
$\widehat{\lambda}_i(\vK) \to z(-1/\lambda_i)$.
\item For each $i \in \cI$ and any deterministic unit vector $\vv \in \R^n$,
$(\vv^\top \widehat \vv_i)^2-\varphi(-1/\lambda_i)(\vv^\top \vv_i)^2 \to 0$,
where $\vv_i,\widehat\vv_i$ are the unit eigenvectors of $\vSigma,\vK$ for
eigenvalues $\lambda_i(\vSigma),\widehat \lambda_i(\vK)$.
\item Let $\vu=\frac{1}{\sqrt{N}}(u_1,\ldots,u_N)^\top \in \R^N$ be such that
$[\vu,\vG] \in \R^{N \times (n+1)}$ has i.i.d.\ rows $\{[u_j,\vg_j]\}_{j=1}^N$, and denote
$\E[u\vg]=\E[u_j\vg_j]$ for all $j \in [N]$. Then for each $i \in \cI$,
\[(\vu^\top \widehat{\vu}_i)^2-\frac{z(-1/\lambda_i)\varphi(-1/\lambda_i)}{\lambda_i^2}
(\E[u\vg]^\top \vv_i)^2 \to 0\]
where $\widehat{\vu}_i$ is the unit eigenvector of
$\vG\vG^\top$ for its eigenvalue
$\widehat\lambda_i(\vG\vG^\top)=\widehat\lambda_i(\vK)$.
\end{enumerate}
\end{theorem}

Statements (a--b) are known in a linear
setting $\vg_i=\vSigma^{1/2}\vz_i$ when $\vz_i$ has i.i.d.\ entries,
see e.g.\ \citep{bai2012sample} and \cite[Theorems 11.3 and 11.5]{yao2015sample}. The above theorem thus
verifies an exact asymptotic equivalence between spiked spectral phenomena
in a nonlinear spiked covariance model with those of a linearly defined 
(possibly Gaussian) model.

In Section~\ref{sec:spike_CK}, each CK matrix $\vK_\ell$ has (approximately)
the structure of the above matrix $\vK$ over the randomness of $\vW_\ell$,
conditional on the features $\vX_{\ell-1}$ of the preceding layer, and
Theorem \ref{thm:ck_spike} follows from Theorem \ref{thm:informal}(a,b).
In Section~\ref{sec:trained_CK}, the CK matrix $\vK$ defined by trained weights
has (approximately) this structure over the randomness of $\tilde \vX$,
conditional on $\vW_{\mathrm{trained}}$, and Theorem
\ref{thm:gd_spike} follows from Theorem \ref{thm:informal}(a,c).

\paragraph{Proof ideas.} Analyses in the linearly defined model
$\vg_i=\vSigma^{1/2}\vz_i$ commonly stem from block matrix inversion identities 
with respect to the block decompositions
\[\vSigma=\begin{pmatrix}
	\vSigma_r & \mathbf{0}\\
	\mathbf{0} & \vSigma_0
\end{pmatrix}, \qquad
\vG=\begin{pmatrix} \vG_r & \vG_0 \end{pmatrix}\]
where $\vSigma_r$ contains the spike eigenvalues of $\vSigma$, and $\vG_r$ is
independent of $\vG_0$. This independence does not hold in our setting,
and we develop a different ``master equation'' approach.

Let $\widehat\lambda^{1/2}$ be a spike singular value of $\vG$ with
corresponding unit singular vectors $(\widehat\vu,\widehat\vv)$. We consider the
linearized equation
\begin{equation}\label{eq:linearized}
0=\begin{pmatrix} -\widehat{\lambda}\vI & \vG^\top \\ \vG & -\vI \end{pmatrix}
\begin{pmatrix} \widehat\vv \\ \widehat{\lambda}^{1/2}\widehat\vu
\end{pmatrix}.\end{equation}
Writing $\vV_r \in \R^{n \times r}$ for the $r$ spike eigenvectors of $\vSigma$,
we define a generalized resolvent
\[\vcR(z,\alpha)=\begin{pmatrix} -z\vI-\alpha\vV_r\vV_r^\top & \vG^\top \\ \vG & -\vI
\end{pmatrix}^{-1},\]
add to \eqref{eq:linearized} the quantity
$-\alpha \begin{pmatrix} \vV_r \\ \mathbf{0}
\end{pmatrix} \cdot \vV_r^\top \widehat \vv$ on both sides for some large
$\alpha>0$, and rewrite this as
\begin{equation} \label{eq:eigenvector_master}
\begin{pmatrix} \widehat\vv \\ \widehat\lambda^{1/2}\widehat\vu \end{pmatrix}
=-\alpha\,\vcR(\widehat\lambda,\alpha)\begin{pmatrix}\vV_r \\ \mathbf{0} \end{pmatrix}
\cdot \vV_r^\top \widehat\vv.
\end{equation}
We will show that $\vcR(z,\alpha)$ exists and is bounded in operator norm
for any $z$ separated from the limit
bulk spectral support of $\vK$ and any large enough $\alpha>0$. Then,
multiplying \eqref{eq:eigenvector_master} by $(\vV_r^\top\;\mathbf{0})$ and
applying a block matrix inversion identity,
\begin{align*}
\vV_r^\top \widehat\vv&={-}\alpha \begin{pmatrix} \vV_r \\ \mathbf{0}
\end{pmatrix}^\top \vcR(\widehat\lambda,\alpha)\begin{pmatrix}\vV_r \\ \mathbf{0} \end{pmatrix}
\cdot \vV_r^\top \widehat\vv
=-\alpha \vV_r^\top\Big(\vG^\top
\vG-\widehat\lambda \vI-\alpha\vV_r\vV_r^\top \Big)^{-1}\vV_r \cdot \vV_r^\top \widehat\vv.
\end{align*}
As a result, spike eigenvalues $\widehat\lambda$ are roots
$z=\widehat\lambda$ of the master equation
	\[\det\left(\vI_{r}+\alpha
\vV_r^\top\Big(\vG^\top\vG-z\vI-\alpha\vV_r\vV_r^\top\Big)^{-1}\vV_r \right)=0,\]
for any fixed and large $\alpha>0$. Singular vector alignments may be
characterized likewise from \eqref{eq:eigenvector_master}.

The core of the proof is an asymptotic analysis of this master equation
via a deterministic equivalent approximation
\begin{equation}\label{eq:detequivinformal}
\vv_1^\top \vR(\vGamma)\vv_2:=
\vv_1^\top (\vG^\top \vG-\vGamma)^{-1}\vv_2 \approx {-}\vv_1^\top
(\vGamma+z\tilde m(z)\vSigma)^{-1}\vv_2
\end{equation}
for any deterministic unit vectors $\vv_1,\vv_2 \in \R^n$ and low-rank perturbations $\vGamma$
of $z\vI$, where $\tilde m(z)$ is the
Stieltjes transform of the ``companion'' limit measure $\tilde \mu$ for the
eigenvalue distribution of $\vG\vG^\top \in \R^{N \times N}$. We extend
results of \cite{chouard2022quantitative,schroder2023deterministic}
by establishing this approximation not only for $\vGamma=z\vI$ but also
perturbations thereof, and for spectral arguments $z \in \C \setminus
\supp{\mu}$ that may belong to the positive real line. The latter
extension requires showing, a priori, that all eigenvalues of
$\vK=\vG^\top \vG$ fall close to $\supp{\mu}$ in the absence of spiked structure.
We show this by adapting an argument of \cite{bai1998no} and using a fluctuation
averaging lemma described below.

Let us conclude with a brief discussion of our proof
of \eqref{eq:detequivinformal}: From manipulations of the identity
	\[\Tr \vB=
\Tr (\vG^\top \vG-\vGamma)\vR(\vGamma)\vB={-}\Tr
\vR(\vGamma)\vB\vGamma+\frac{1}{N}\sum_{i=1}^N \vg_i^\top \vR(\vGamma)\vB\vg_i\]
for appropriately chosen matrices $\vB \in \C^{n \times n}$,
the Sherman-Morrison (leave-one-out) formula for matrix
inversion applied to $\vR(\vGamma)$, and the concentration of bilinear
forms in $\vg_i$, one may show
\begin{equation}\label{eq:prefluctavg}
\vv_1^\top (\vGamma+z\tilde m(z)\vSigma)^{-1}\vv_2
\approx {-}\vv_1^\top \vR(\vGamma)\vv_2+
			 \frac{1}{1+N^{-1}\Tr \vSigma\vR(\vGamma)}
			 \cdot \frac{1}{N}\sum_{i=1}^N (1-\E_{\vg_i})T_i
\end{equation}
where $T_i=\vg_i^\top
	\vR^{(i)}(\vGamma)\vv_2 \cdot \vv_1^\top(\vGamma+z m_{\tilde
\vK}^{(i)}(\vGamma)\vSigma)^{-1} \vg_i$. Here, $\vR^{(i)}(\vGamma)$ and
$m_{\tilde \vK}^{(i)}(\vGamma)$ are generalized leave-one-out resolvents and
empirical Stieltjes transforms defined by $\{\vg_j\}_{j \neq i}$, and $\E_{\vg_i}$
is the partial expectation over only $\vg_i$. Under our assumptions for $\vg_i$,
each error term $(1-\E_{\vg_i})T_i$ has mean 0 and $O(1)$ fluctuations.
We develop a fluctuation averaging lemma using recursive applications of the
Sherman-Morrison identity to further resolve the dependence of
$\vR^{(i)}(\vGamma)$ and $m_{\tilde \vK}^{(i)}(\vGamma)$ on fixed subsets of
rows $\{\vg_j\}_{j \neq i}$, to show that the errors
$(1-\E_{\vg_i})T_i$ are weakly correlated across $i \in [N]$. Hence their
average has a mean 0 and fluctuates on the asymptotically negligible
scale of $O(N^{-1/2})$, and applying this to \eqref{eq:prefluctavg} shows
\eqref{eq:detequivinformal}.

\bigskip

\subsection*{Acknowledgements}
This research was supported in part by NSF DMS-2142476, NSF DMS-2055340 and NSF DMS-2154099.

{
\fontsize{10}{11}\selectfont  
 
\bibliography{ref}
\bibliographystyle{alpha}

}
\appendix

\newpage
\subsection*{Organization of the Appendices}
\begin{itemize}[leftmargin=*]
\item Appendix \ref{appendix:background} introduces relevant notation and background.
\item Appendix \ref{sec:spiked_sample_covariance} states our main results for the general nonlinear spiked covariance model
\[\vK=\vG^\top \vG,\]
formalizing the discussion in Section \ref{sec:proof_overview}. These results
are divided into two subsections: Appendix \ref{subsec:nonasymptotic} gives a
``no outliers'' statement for $\vK$ and a deterministic equivalent approximation
for its resolvent, under minimal asymptotic assumptions.
Appendix \ref{subsec:spikes} then states the main characterizations of spike
eigenvalues/eigenvectors in an asymptotic setting with a spiked eigenstructure.
\item Appendix \ref{appendix:resolvent} develops a general fluctuation averaging
lemma for the sample covariance model, and proves the results of
Appendix \ref{subsec:nonasymptotic}.
\item Appendix \ref{appendix:spike} proves the
results of Appendix \ref{subsec:spikes}.
\item Finally, Appendix \ref{sec:CK_spike_proof} proves the results of Section \ref{sec:spike_CK}
on propagation of spiked eigenstructure through the layers of a neural network, and
Appendix \ref{sec:trained_CK_proof} proves the results of Section
\ref{sec:trained_CK} on the eigenstructure of the CK after gradient descent
training.
\end{itemize}

\bigskip

\section{Notations and background}\label{appendix:background}

\subsection{Stochastic domination}

We use the following standard notation for stochastic domination of random
variables, see e.g.\ \cite[Definition 2.4]{erdHos2013averaging}:
For random variables $X \equiv X(u)$ and $Y \equiv Y(u) \geq 0$ depending implicitly on $N$
and a parameter $u \in U_N$, as $N \to \infty$, we write
\[X \prec Y \text{ or } X=\SD{Y} \text{ uniformly over } u \in U_N\]
if, for any fixed $\eps,D>0$ and all large $N$,
\[\sup_{u\in U_N}\P \Big[|X(u)|>N^\eps Y(u)\Big]<N^{-D}.\]
Throughout, ``for all large $N$'' means for all $N \geq N_0$ where
$N_0$ may depend on $\eps,D$, any quantities that are constant in the context
of the statement, and convergence rates of the spike eigenvalues and
empirical spectral measures in the given assumptions.

If $X=\bI\{\cE\}$ is the indicator of an event $\cE \equiv \cE_N$, then
$\bI\{\cE\} \prec 0$ means $\P[\cE]<N^{-D}$ for any fixed $D>0$ and all large
$N$. If $X$ and $Y$ are both deterministic, then $X \prec Y$ means
$|X| \leq N^\eps Y$ (deterministically) for any $\eps>0$ and all large $N$.
For an event $\cE \equiv \cE_N$, we will write
\[X=\SDcE{Y}\]
as shorthand for $X \cdot \bI\{\cE\} \prec Y$.

We will use the following basic properties often implicitly.

\begin{proposition}\label{prop:domination}
Suppose $X \prec Y$ uniformly over $u \in U_N$.
\begin{enumerate}[label=(\alph*)]
\item If $|U_N|\leq N^C$ for a constant $C>0$, then
for any fixed $\eps,D>0$ and all large $N$,
\[\P \Big[\text{there exists } u\in U_N \text{ with } |X(u)|
\geq N^\eps Y(u)\Big] \leq N^{-D}.\]
\item If $|U_N| \leq N^C$ for a constant $C>0$, then
$\sum_{u \in U_N} X(u) \prec \sum_{u \in U_N} Y(u)$.
\item If $|U_N| \leq C$ for a constant $C>0$, then
$\prod_{u \in U_N} X(u) \prec \prod_{u \in U_N} Y(u)$.
\item If $Y$ is deterministic, and $\E[X^2] \leq N^C$ and $Y \geq N^{-C}$
for a constant $C>0$, then also $\E[|X|] \prec Y$ uniformly over $u \in U_N$.
\end{enumerate}
\end{proposition}
\begin{proof}
The first three statements follow from a union bound over $U_N$. For the last
statement, for any fixed $\eps>0$, observe that
\[\E|X| \leq N^{\eps/2}Y+\E\Big[|X|\bI\{|X|>N^{\eps/2}Y\}\Big]
\leq N^{\eps/2}Y+\E[X^2]^{1/2}\P[|X|>N^{\eps/2}Y]^{1/2}.\]
Applying $\E[X^2] \leq N^C$, $Y \geq N^{-C}$, and
$\P[|X|>N^{\eps/2}Y]<N^{-D}$ for sufficiently large $D>0$ shows that the
second term is less than $N^{\eps/2}Y$ for all large $N$, hence $\E|X|<N^\eps
Y$.
\end{proof}

\subsection{Deformed Marcenko-Pastur law}

For a probability measure $\nu$ supported on $[0,\infty)$ and an aspect ratio
parameter $\gamma>0$, consider the deformed Marcenko–Pastur measure
\[\mu=\rho^{\MP}_{\gamma}\boxtimes \nu\]
and its ``companion'' probability measure
\[\tilde \mu=\gamma\mu+(1-\gamma)\delta_0.\]
Here, $\mu$ and $\tilde\mu$ represent the limit eigenvalue distributions of
$\vG^\top \vG \in \R^{n \times n}$ and
$\vG\vG^\top \in \R^{N \times N}$ respectively, when
$\vG=\frac{1}{\sqrt{N}}[\vg_1,\ldots,\vg_N] \in \R^{N
\times n}$ has i.i.d.\ rows with mean 0 and covariance $\vSigma$, and
$n,N \to \infty$ with $n/N \to \gamma$ and 
$\frac{1}{n}\sum_{i=1}^n \delta_{\lambda_i(\vSigma)} \to \nu$ weakly.

These measures $\mu,\tilde\mu$ may be defined by their Stieltjes transform
\begin{equation}\label{eq:stieltjes}
m(z)=\int \frac{1}{x-z}d\mu(x), \qquad
\tilde m(z)=\int \frac{1}{x-z}d\tilde \mu(x)
\end{equation}
where $\tilde m(z)=\gamma m(z)+(1-\gamma)(-1/z)$. 
By the results of \citep{marchenko1967,silverstein1995empirical},
for any $z \in \C^+$, $m(z)$ and $\tilde m(z)$ are the unique roots in
$\{m \in \C:\gamma m+(1-\gamma)(-1/z) \in\C^+\}$ and $\C^+$, respectively,
to the \textit{Marcenko-Pastur equations}
\begin{equation}\label{eq:MPeq}
m(z)=\int \frac{1}{\lambda(1-\gamma-\gamma z m(z))-z}\,d\nu(\lambda), \quad
z=-\frac{1}{\tilde m(z)}+\gamma \int \frac{\lambda}{1+\lambda
\tilde m(z)}d\nu(\lambda).
\end{equation}
We define $m(z),\tilde m(z)$ via (\ref{eq:stieltjes}) also on the full domains
$\C \setminus \supp{\mu}$ and $\C \setminus \supp{\tilde \mu}$ respectively,
where the support sets $\supp{\mu}$ and $\supp{\tilde \mu}$ may differ only at
the single point $\{0\}$.

In the setting $\vSigma=\vI$ (and $\nu=\delta_1$), the law $\mu=\rho_\gamma^\MP$
is the standard Marcenko-Pastur law, with explicit density function with respect
to Lebesgue measure
\[d\rho_\gamma^\MP(\lambda)=\frac{1}{2\pi}\frac{\sqrt{(\lambda_{+}-\lambda)(\lambda-\lambda_{-})}}{\gamma\lambda}\cdot\bI_{\lambda
\in [\lambda_{-},\lambda_+]}\,d\lambda, \qquad \lambda_{\pm}:=(1\pm\sqrt{\gamma})^2\]
for $\gamma \leq 1$, and an additional point mass $(1-1/\gamma)$ at 0
when $\gamma>1$.

In general, $\mu$ and $\tilde \mu$ do not have analytically
explicit densities. However, $\supp{\tilde \mu}$ is explicitly
characterized in \cite{silverstein1995analysis}, and we review this
characterization here: Define
\begin{equation}\label{eq:zdomain}
\cT=\{0\} \cup \{-1/\lambda:\lambda \in \supp{\nu}\}.
\end{equation}
For $\tilde m \in \C \setminus \cT$, define
\begin{equation}\label{eq:inv_z}
z(\tilde m)={-}\frac{1}{\tilde m}+\gamma \int\frac{\lambda}{1+\lambda \tilde m}d\nu(\lambda).
\end{equation}	
In light of the second equation of (\ref{eq:MPeq}), this may be understood as
a formal inverse of $\tilde m(z)$.
From~\cite[Theorems 4.1 and 4.2]{silverstein1995analysis},
we have the following properties.

\begin{proposition}\label{prop:inv_z}
$\tilde m(\cdot)$ defines a bijection from $\{z \in \R \setminus \supp{\tilde
\mu}\}$ to $\{\tilde m \in \R \setminus \cT:z'(\tilde m)>0\}$, whose inverse
function is $z(\cdot)$. In particular, $x \in \R$ does not belong to
$\supp{\tilde \mu}$ if and only if there exists
$\tilde m \in \R \setminus \cT$
such that $z'(\tilde m)>0$ and $z(\tilde m)=x$.
\end{proposition}

\subsection{Additional notation}

For a probability measure $\mu$, its support is the closed set
\[\supp{\mu}=\{x \in \R:\mu(O)>0 \text{ for any open neighborhood } O \ni x\}.\]
We write $\dist{x,A}=\inf\{|x-y|:y \in A\}$ and define the $\eps$-neighborhood
\[\supp{\mu}+(-\eps,\eps)=\{x \in \R:\dist{x,\supp{\mu}}<\eps\}.\]
We write $\delta_x$ for the probability measure given by a point mass at
$x \in \R$, $a\mu_0+(1-a)\mu_1$ for the convex combination of $\mu_0,\mu_1$,
and $a \otimes \mu \oplus b$ for
the law of $a\mathsf{x}+b$ when $\mathsf{x} \sim \mu$.

For vectors, $\|\vv\| \equiv \|\vv\|_2$ is the Euclidean norm. For matrices,
$\|\vM\|$ is the
operator norm $\sup_{\vv:\|\vv\|=1} \|\vM\vv\|$, $\|\vM\|_F$ is the
Frobenius norm $(\Tr \vM^\top\vM)^{1/2}$, $\Tr$ is the (unnormalized)
matrix trace, and $\vA \odot \vB$ is the entrywise (Hadamard) product.
We write $\diag(\vv)$ for the diagonal matrix with vector $\vv$ along the
main diagonal, and $\vI_n$ for the $n \times n$ identity matrix.


\section{Results for the nonlinear spiked covariance model}\label{sec:spiked_sample_covariance}

\subsection{Deterministic equivalent for the resolvent}\label{subsec:nonasymptotic}

We consider the sample covariance and Gram matrix
\[\vK=\vG^\top\vG \in \R^{n \times n}, \quad
\widetilde \vK=\vG\vG^\top \in \R^{N \times N}, \quad \text{ where } \quad
\vG=\frac{1}{\sqrt{N}}[\vg_1,\ldots,\vg_N] \in \R^{N \times n}.\]
The following are our basic assumptions, where we
recall that $\bI\{\cE\} \prec 0$ means $\P[\cE] \leq N^{-D}$ for any fixed $D>0$
and all large $N$.
\begin{assumption}\label{assump:G}
\phantom{}
The rows of $\vG$ are independent and satisfy
$\E[\vg_i]=\mathbf{0}$ and $\E[\vg_i\vg_i^\top]=\vSigma$ for all $i\in[N]$,
such that:
\begin{enumerate}[label=(\alph*)]
\item There exist constants $C,c>0$ such that $c<n/N<C$ and
$\|\vSigma\|<C$.
\item There exists a constant $B>0$ such that
$\bI\{\|\vK\|>B\} \prec 0$.
\item Uniformly over deterministic matrices $\vA \in \C^{n \times n}$ and over
$i \neq j \in [N]$,
\[\vg_i^\top \vA \vg_i - \E[\vg_i^\top \vA \vg_i]
\prec \|\vA\|_F, \qquad \vg_i^\top \vA \vg_j \prec \|\vA\|_F.\]
\item For any integer $\alpha>0$, there exists a constant $C=C(\alpha)>0$
such that $\E[\|\vg_i\|^\alpha] \leq N^C$.
\end{enumerate}
\end{assumption}

Denote the finite-$N$
dimension ratio and empirical eigenvalue distribution of $\vSigma$ by
\begin{equation}\label{eq:gammanuN}
\gamma_N=\frac{n}{N}, \qquad \nu_N=\frac{1}{n}\sum_{i=1}^n
\delta_{\lambda_i(\vSigma)}.
\end{equation}
Let
\[\mu_N=\rho_{\gamma_N}^\MP \boxtimes \nu_N, 
\qquad \tilde \mu_N=\gamma_N\mu_N+(1-\gamma_N)\delta_0.\]
Denote the Stieltjes transforms of $\mu_N,\tilde \mu_N$ by
$m_N(z),\tilde m_N(z)$. These are characterized exactly as in \eqref{eq:MPeq}
with $(\gamma_N,\nu_N)$ in place of $(\gamma,\nu)$.

We first establish that with high probability,
$\vK$ and $\widetilde \vK$ have no outlier eigenvalues far from the support set
\begin{equation}\label{eq:support}
\cS_N=\supp{\mu_N} \cup \{0\}=\supp{\tilde \mu_N} \cup \{0\}.
\end{equation}
\begin{theorem}\label{thm:no_outlier}
Suppose Assumption \ref{assump:G} holds. Then for any fixed $\eps>0$,
\begin{align*}
\bI\Big\{\vK \text{ has an eigenvalue outside }
\cS_N+(-\eps,\eps)\Big\} \prec 0.
\end{align*}
\end{theorem}
In asymptotic settings where $\nu_N \to \nu$ and $\mu_N \to \mu$ weakly and
$\vSigma$ has no spike eigenvalues, this set $\cS_N$ will converge to
$\cS:=\supp{\mu} \cup \{0\}$. In general, $\cS_N$ may contain intervals around
spike eigenvalues of $\vK$ that are separated from $\supp{\mu} \cup \{0\}$
if $\vSigma$ has a spiked structure, and this will be clarified in the
subsequent section.

Next, we establish a deterministic equivalent approximation for the resolvent of
$\vK$, for spectral arguments separated from this support set $\cS_N$.
Let us denote by
\[\vR(z)=(\vK-z\vI)^{-1}, \qquad
m_{\vK}(z)=\frac{1}{n} \Tr \vR(z)\]
the resolvent and Stieltjes transform of $\vK$ for $z\not\in\supp{\mu_N}$.
For any $\eps>0$, define the domain
\begin{equation}\label{eq:U_e}
U_N(\eps)=\Big\{z \in \C:\,|z| \leq \eps^{-1},\,
\dist{z,\cS_N} \geq \eps\Big\}.
\end{equation}

\begin{theorem}\label{thm:deterministic_equ}
Suppose Assumption \ref{assump:G} holds. Then for any fixed $\eps>0$,
uniformly over $z \in U_N(\eps)$ and over deterministic matrices
$\vA \in \C^{n \times n}$, we have
\[m_{\vK}(z)-m_N(z) \prec \frac{1}{N},
\qquad \Tr \Big[\vR(z)\vA-(-z\tilde m_N(z)\vSigma-z\vI)^{-1}\vA\Big]
\prec \frac{1}{\sqrt{N}}\,\|\vA\|_F.\]
\end{theorem}
For spectral arguments $z \in \C \setminus \R_+$ separated from the positive
real line, such a result has been shown recently in
\cite{chouard2022quantitative,schroder2023deterministic} (using different proof
techniques). We use Theorem \ref{thm:no_outlier} as an input to establish this
approximation also for spectral arguments in $\R_+ \setminus \cS_N$, as such a
result (and its extension to a generalized resolvent) is needed
for our analysis of spiked eigenstructure to follow.

\subsection{Spike eigenvalues and eigenvectors}\label{subsec:spikes}

Now we consider an asymptotic setting with a specific spiked structure for the
population covariance matrix $\vSigma$, having a fixed number of spikes outside the
support of the weak limit of its spectral law.
This assumption is summarized as follows.

\begin{assumption}\label{assum:spike}
$\vSigma$ has eigenvalues $\lambda_1(\vSigma),\ldots,\lambda_n(\vSigma)$ (not
necessarily ordered by magnitude) where,
for a fixed integer $r \geq 0$, as $N \to \infty$:
\begin{enumerate}[label=(\alph*)]
\item $n/N \to \gamma \in (0,\infty)$.
\item There exists a probability measure $\nu$ with compact support in
$(0,\infty)$, such that
\[\frac{1}{n-r}\sum_{i=r+1}^n \delta_{\lambda_i(\vSigma)} \to \nu \text{
weakly.}\]
Furthermore, for any fixed $\eps>0$ and all large $N$,
\[\lambda_i(\vSigma) \in \supp{\nu}+(-\eps,\eps) \text{ for all }
i \geq r+1.\]
\item There exist distinct values $\lambda_1,\ldots,\lambda_r>0$ with
$\lambda_1,\ldots,\lambda_r \not\in\supp{\nu}$ such that
\[\lambda_i(\vSigma) \to \lambda_i \text{ for all } i=1,\ldots,r.\]
\end{enumerate}
\end{assumption}

Under this assumption, we analyze the outlier singular values of $\vG$ and their
corresponding singular vectors. Let
\[\gamma_{N,0}=\frac{n-r}{N}, \qquad \nu_{N,0}=\frac{1}{n-r}\sum_{i=r+1}^n
\delta_{\lambda_i(\vSigma)}\]
be the finite-$N$ aspect ratio and population spectral measure corresponding to
the bulk component of $\vSigma$. Define the laws
\[\mu_{N,0}=\rho_{\gamma_{N,0}}^{\MP} \boxtimes\nu_{N,0}, \qquad
\tilde \mu_{N,0}=\gamma_{N,0}\mu_{N,0}+(1-\gamma_{N,0})\delta_0\]
and let $m_{N,0}(z),\tilde m_{N,0}(z)$ be their Stieltjes transforms.
In the setting of Assumption \ref{assum:spike}, we note that
$\mu_{N,0} \to \mu=\rho_\gamma^{\MP} \boxtimes \nu$ and
$\tilde \mu_{N,0} \to \tilde \mu=\gamma \mu+(1-\gamma)\delta_0$ weakly as
$N \to \infty$, where the Stieltjes transforms $m(z),\tilde m(z)$ of these
limits $\mu,\tilde \mu$ are characterized by \eqref{eq:MPeq}.

Denote the limit support set
\begin{equation}\label{eq:limitsupport}
\cS=\supp{\mu} \cup \{0\}=\supp{\tilde \mu} \cup \{0\}.
\end{equation}
Under Assumption \ref{assum:spike} when $r=0$, i.e.\ $\vSigma$ does
not have spike eigenvalues, the following is a
corollary of Theorem \ref{thm:no_outlier}. A similar ``no outlier'' statement
has been shown for linearly defined sample covariance models in
\citep{bai1998no}.

\begin{coro}\label{cor:no_outlier}
Suppose Assumptions \ref{assump:G} and \ref{assum:spike} hold, where $r=0$.
Then for any fixed $\eps>0$,
\begin{align*}
\bI\Big\{\vK \text{ has an eigenvalue outside }
\cS+(-\eps,\eps)\Big\} \prec 0.
\end{align*}
\end{coro}

We now give a more quantitative version of Theorem \ref{thm:informal} stated
informally in Section \ref{sec:proof_overview}, which describes the
spike eigenvalues of $\vK=\vG^\top \vG$ and corresponding singular vectors of $\vG$ when
there are possibly spike eigenvalues in $\vSigma$. Define the domain
\[\cT_{N,0}=\{0\} \cup \{-1/\lambda:\lambda \in \supp{\nu_{N,0}}\}.\]
For $\tilde m \in \C \setminus \cT_{N,0}$, define the functions
\begin{equation}\label{eq:zN0}
z_{N,0}(\tilde m)=-\frac{1}{\tilde m}+\gamma_{N,0}
\int \frac{\lambda}{1+\lambda \tilde m}d\nu_{N,0}(\lambda), \qquad
\varphi_{N,0}(\tilde m)={-}\frac{\tilde m z_{N,0}'(\tilde m)}{z_{N,0}(\tilde m)}.
\end{equation}
We note that under Assumption \ref{assum:spike}, the domain $\cT_{N,0}$
converges in Hausdorff distance to $\cT$ as defined in \eqref{eq:zdomain}.
We will verify in the proof (c.f.\ Lemma \ref{lemma:supportcompareb})
that $z_{N,0}(\tilde m) \to z(\tilde m)$ 
and $z_{N,0}'(\tilde m) \to z'(\tilde m)$ for each fixed
$\tilde m \in \C \setminus \cT$,
where $z(\cdot)$ is as defined in \eqref{eq:inv_z}. Then also
$\varphi_{N,0}(\tilde m) \to \varphi(\tilde m)$ for the limiting function
\begin{equation}\label{eq:phi_limit}
\varphi(\tilde m)={-}\frac{\tilde m z'(\tilde m)}{z(\tilde m)}.
\end{equation}

\begin{theorem}\label{thm:spiked}
Suppose Assumptions \ref{assump:G} and \ref{assum:spike} hold. Let
\[\cI=\big\{i \in \{1,\ldots,r\}:z'(-1/\lambda_i)>0\big\}.\]
\begin{enumerate}[label=(\alph*)]
\item For any sufficiently small constant $\eps>0$ and all large $N$,
on an event $\cE \equiv \cE_N$ satisfying $\bI\{\cE^c\} \prec 0$,
there is a 1-to-1 correspondence between the eigenvalues of
$\vK$ outside $\cS+(-\eps,\eps)$ and $\{\lambda_i:i \in \cI\}$.
Denoting these eigenvalues of $\vK$ by
$\{\widehat \lambda_i:i \in \cI\}$, we have
\[\widehat \lambda_i-z_{N,0}(-1/\lambda_i(\vSigma))=\SDcE{\frac{1}{\sqrt{N}}}\]
for each $i \in \cI$, where
$z_{N,0}(-1/\lambda_i(\vSigma)) \to z(-1/\lambda_i)>0$ as $N \to \infty$.
\item On this event $\cE$, for each $i \in \cI$, let $\widehat \vv_i \in \R^n$
be a unit-norm eigenvector of $\vK$ (i.e.\ right singular vector of $\vG$)
corresponding to its eigenvalue $\widehat \lambda_i$, and let $\vv_i$ be a unit-norm
eigenvector of $\vSigma$ corresponding to $\lambda_i(\vSigma)$.
Then, uniformly over (deterministic) unit vectors $\vv\in\R^n$,
\begin{align}
|\vv^\top \widehat\vv_i|-\sqrt{\varphi_{N,0}
({-}1/\lambda_i(\vSigma))} \cdot |\vv^\top\vv_i|=\SDcE{\frac{1}{\sqrt{N}}}
\end{align}
where $\varphi_{N,0}({-}1/\lambda_i(\vSigma)) \to \varphi({-}1/\lambda_i)>0$
as $N \to \infty$. In particular, for each $i \in \cI$,
$|\vv_i^\top\widehat{\vv}_i|^2 \to \varphi({-}1/\lambda_i)$
and $\sup_{j \in [n]:j \neq i} |\vv_j^\top\widehat{\vv}_i|^2 \to 0$
almost surely as $N \to \infty$.
\item Let $\vu=\frac{1}{\sqrt{N}}(u_1,\ldots,u_N)^\top \in\R^N$ be a random
vector such that $[\vu,\vG] \in \R^{N \times (n+1)}$ has independent rows
also satisfying Assumption~\ref{assump:G}. Denote by $\E[u\vg] \in \R^n$
the common value of $\E[u_j\vg_j]$ for all $j \in [N]$.

On this event $\cE$, for each $i \in \cI$, let $\widehat \vu_i \in \R^N$ be a unit-norm
eigenvector of $\widetilde\vK$ (i.e.\ left singular vector of $\vG$)
corresponding to its eigenvalue $\widehat \lambda_i$,
and let $\vv_i$ be the eigenvector of $\vSigma$ as in part (b). Then
\begin{align}
|\vu^\top\widehat{\vu}_i|-
\frac{\sqrt{z_{N,0}(-1/\lambda_i(\vSigma))\varphi_{N,0}({-}1/\lambda_i(\vSigma))}}{\lambda_i(\vSigma)}\cdot\left|\E[u\vg]^\top\vv_i\right|=\SDcE{\frac{1}{\sqrt{N}}}.
\end{align}
\end{enumerate} 
\end{theorem}


\section{Analysis of the resolvent}\label{appendix:resolvent}

We prove the results of Appendix \ref{subsec:nonasymptotic}. Appendix
\ref{sec:fluctuation} first develops a fluctuation averaging lemma for
the sample covariance model. Appendix \ref{sec:nooutliers} applies this lemma
within the arguments of \cite{bai1998no}, to prove the ``no outliers'' result
of Theorem \ref{thm:no_outlier}. Appendix \ref{sec:deterministicequiv} uses
Theorem \ref{thm:no_outlier} and a second application of the fluctuation
averaging lemma to prove the deterministic equivalent
approximation of Theorem \ref{thm:deterministic_equ}.

\subsection{Fluctuation averaging lemma}\label{sec:fluctuation}

Recall the definitions
\[\vK=\vG^\top \vG, \qquad \widetilde \vK=\vG\vG^\top.\]
For $S \subset [N]$, let $\vG^{(S)} \in \R^{(N-|S|) \times n}$ be the matrix
obtained by removing the rows of $\vG$ corresponding to $i \in S$, and define
\[\vK^{(S)}={\vG^{(S)}}^\top \vG^{(S)}
=\frac{1}{N}\sum_{i \in [N] \setminus S} \vg_i\vg_i^\top \in \R^{n \times n}.\]
Then, for $\vGamma \in \C^{n \times n}$, define
\begin{equation}\label{eq:leaveoneout}
\begin{gathered}
\qquad \vR^{(S)}(\vGamma)=(\vK^{(S)}-\vGamma)^{-1},
\qquad m_{\vK}^{(S)}(\vGamma)=\frac{1}{n}\Tr \vR^{(S)}(\vGamma),\\
\tilde m_{\vK}^{(S)}(\vGamma)=\gamma_N m_{\vK}^{(S)}(\vGamma)
+(1-\gamma_N)\left({-}\frac{1}{z}\right)
=\frac{1}{N}\,\Tr\vR^{(S)}(\vGamma)+\left(1-\frac{n}{N}\right)\left({-}\frac{1}{z}\right).
\end{gathered}
\end{equation}
Importantly, these quantities are independent of $\{\vg_i:i \in S\}$.
We say that $\vR^{(S)}(\vGamma)$ exists (and hence also $m_{\vK}^{(S)},\tilde m_{\vK}^{(S)}$ exist) when $\vK^{(S)}-\vGamma$ is invertible.
For simplicity, we write $\vR=\vR^\emptyset$, $\vR^{(i)}=\vR^{(\{i\})}$,
$\vR^{(Si)}=\vR^{(S \cup \{i\})}$, and similarly for $m_{\vK}$ and $\tilde
m_{\vK}$.

\begin{lemma}\label{lemma:fluctuationavg}
Suppose Assumption \ref{assump:G} holds. Suppose also that there are constants
$C_0,c_0,\delta,\upsilon>0$, $N$-dependent
domains $U \subset \C \setminus \{0\}$ and
$\cD_\Gamma,\cD_A \subseteq \C^{n \times n}$, and $N$-dependent maps
$\Phi_N:\cD_\Gamma \times \cD_A \to (N^{-\upsilon},N^\upsilon)$ and
$\Psi_N:\cD_\Gamma \to (N^{-\upsilon},N^{1-\delta})$,
such that for any fixed $L \geq 1$, the events
\begin{align}
\cE(z,\vGamma,\vA,S)&=\Big\{\vR^{(S)}(\vGamma) \text{ exists, }
\|\vR^{(S)}(\vGamma)\vA\|_F \leq \Phi_N(\vGamma,\vA),\;
\|\vR^{(S)}(\vGamma)\|_F \leq \Psi_N(\vGamma),\nonumber\\
&\quad
\|(z^{-1}\vGamma+\tilde m_{\vK}^{(S)}(\vGamma) \vSigma)^{-1}\| \leq C_0, \text{ and }
|1+N^{-1} \vg_j^\top \vR^{(S)}(\vGamma)\vg_j| \geq c_0 \text{ for all } j \in S \Big\}\label{eq:ESdef}
\end{align}
satisfy $\bI\{\cE(z,\vGamma,\vA,S)^c\} \prec 0$
uniformly over $z \in U$, $\vGamma \in \cD_\Gamma$, $\vA \in \cD_A$,
and $S \subset [N]$ with $|S| \leq L$.

Then, denoting by $\E_{\vg_i}$ the partial expectation over only $\vg_i$
(i.e.\ conditional on $\{\vg_j\}_{j \neq i}$),
also uniformly over $z \in U$, $\vGamma \in \cD_\Gamma$, and $\vA \in \cD_A$,
\begin{equation}\label{eq:FAexpression}
\frac{1}{N}\sum_{i=1}^N (1-\E_{\vg_i})\big[\vg_i^\top \vR^{(i)}(\vGamma)\vA
(z^{-1}\vGamma+\tilde m_{\vK}^{(i)}(\vGamma)\vSigma)^{-1}\vg_i\big]
\prec \max\left(\frac{\Psi_N(\vGamma)}{N},\frac{1}{\sqrt{N}}\right)
\cdot \Phi_N(\vGamma,\vA).
\end{equation}
\end{lemma}

We remark that applying Assumption~\ref{assump:G}(c) and the conditions
of $\cE(z,\vGamma,\vA,i)$ separately to each summand of the left side of
~\eqref{eq:FAexpression} gives the naive bound
\begin{align*}
&\frac{1}{N}\sum_{i=1}^N (1-\E_{\vg_i})[\vg_i^\top \vR^{(i)}(\vGamma)\vA
(z^{-1}\vGamma+\tilde m_{\vK}^{(i)}(\vGamma)\vSigma)^{-1}\vg_i]\\
&\prec \max_{i=1}^N \|\vR^{(i)}(\vGamma)\vA\|_F
\cdot \|(z^{-1}\vGamma+\tilde m_{\vK}^{(i)}(\vGamma)\vSigma)^{-1}\|
\prec \Phi_N(\vGamma,\vA).
\end{align*}
The content of the lemma is to improve this by the additional
factor of $\max(\frac{\Psi_N(\vGamma)}{N},\frac{1}{\sqrt{N}})  \ll 1$.

In this work, we will apply Lemma \ref{lemma:fluctuationavg} only to
spectral arguments $z$ with $O(1)$-separation from $\supp{\mu_N}$
(and matrices $\vGamma=z\vI$ or a finite-rank perturbation thereof), in which
case we will take $\Psi_N(\vGamma)=C/\sqrt{N}$ for a constant $C>0$.
For full-rank matrices $\vA$ 
having bounded operator norm, we will take also
$\Phi_N(\vGamma,\vA)=C/\sqrt{N}$, whereas for finite-rank matrices $\vA$
we will take $\Phi_N(\vGamma,\vA)=C$. We state the result
here more abstractly, as it may be of independent interest to prove local
laws in this nonlinear sample covariance model for spectral arguments $z$
that approach $\supp{\mu_N}$.

In the remainder of this section, we prove Lemma \ref{lemma:fluctuationavg}.
Fix $z \in U$, $\vGamma \in \cD_\Gamma$, and $\vA \in \cD_A$, and
write as shorthand
\[\vR^{(S)}=\vR^{(S)}(\vGamma),\qquad
\tilde m^{(S)}=\tilde m_{\vK}^{(S)}(\vGamma), \qquad
\vOmega^{(S)}=(z^{-1}\vGamma+\tilde m_{\vK}^{(S)}(\vGamma)\vSigma)^{-1},\]
\[\Phi_N=\Phi_N(\vGamma,\vA), \qquad
\Psi_N=\Psi_N(\vGamma),\qquad
\cE(S)=\cE(z,\vGamma,\vA,S).\]
All subsequent instances of $\prec$ will be
implicitly uniform over $z \in U$,
$\vGamma \in \cD_\Gamma$, and $\vA \in \cD_A$.
Define the quantities, for $i \in S$, $j,k \in S \setminus \{i\}$, and
$d \geq 0$,
\begin{align*}
Y_i^{(S)}[d]&=\Tr (\vg_i\vg_i^\top-\vSigma)\vR^{(S)}\vA\vOmega^{(S)}
[\vSigma\vOmega^{(S)}]^d,\\
Z_{ijk}^{(S)}[d]&=N^{-1}\Tr (\vg_i\vg_i^\top-\vSigma)\vR^{(S)}\vg_j\vg_k^\top
\vR^{(S)}\vA\vOmega^{(S)}[\vSigma\vOmega^{(S)}]^d,\\
B_{jk}^{(S)}&=N^{-1}\vg_j^\top\vR^{(S)}\vg_k,\\
C_{jk}^{(S)}&=N^{-2}\vg_j^\top(\vR^{(S)})^2\vg_k,\\
Q_j^{(S)}&=(1+N^{-1} \vg_j^\top \vR^{(S)} \vg_j)^{-1}.
\end{align*}
For each $L \geq 1$, define also the event
\begin{equation}\label{eq:EL}
\cE_L=\bigcap_{S \subset [N]:|S| \leq L} \cE(S).
\end{equation}

\begin{lemma}\label{lemm:bounds}
For any fixed $L,D \geq 1$, uniformly over $S \subset [N]$
with $|S| \leq L$, and over $i \in S$ and
$j,k \in S \setminus \{i\}$ and $d \leq D$,
\begin{equation}\label{eq:FAbounds}
\begin{gathered}
Y_i^{(S)}[d]=\SDcES{\Phi_N},
\quad Z_{ijk}^{(S)}[d]=\SDcES{N^{-1}\Psi_N\Phi_N},\\
\quad B_{jk}^{(S)}=\SDcES{N^{-1}\Psi_N} \text{ for } j \neq k,
\quad C_{jk}^{(S)}=\SDcES{N^{-2}\Psi_N^2},
\quad Q_j^{(S)}=\SDcES{1}.
\end{gathered}
\end{equation}
Furthermore, for any $\alpha>0$, there exists a constant
$C=C(\alpha,L,D)>0$ such that
\begin{equation}\label{eq:FAEbounds}
\begin{gathered}
\E\big[|Y_i^{(S)}[d]|^\alpha\bI\{\cE(S)\}\big]<N^C, \quad
\E\big[|Z_{ijk}^{(S)}[d]|^\alpha\bI\{\cE(S)\}\big]<N^C,\\
\E\big[|B_{jk}^{(S)}|^\alpha\bI\{\cE(S)\}\big]<N^C, \quad
\E\big[|C_{jk}^{(S)}|^\alpha\bI\{\cE(S)\}\big]<N^C, \quad
\E\big[|Q_j^{(S)}|^\alpha\bI\{\cE(S)\}\big]<N^C.
\end{gathered}
\end{equation}
\end{lemma}
\begin{proof}
On the event $\cE(S)$, we have by definition $Q_j^{(S)} \leq 1/c_0$, so the
two statements for $Q_j^{(S)}$ hold immediately. The remaining statements of
\eqref{eq:FAEbounds} follow easily from Holder's inequality, the moment bounds
for $\|\vg_i\|$ in Assumption \ref{assump:G}(d),
the bound $\|\vSigma\|<C$ in Assumption \ref{assump:G}(a),
and the conditions $\|\vR^{(S)}\vA\| \leq \Phi_N \leq N^\upsilon$,
$\|\vR^{(S)}\|_F \leq \Psi_N \leq N$, and $\|\vOmega^{(S)}\| \leq C_0$
defining $\cE(S)$.

For the bounds for $B_{jk}^{(S)}$ and $C_{jk}^{(S)}$ in \eqref{eq:FAbounds},
note that when $j \neq k$, Assumption \ref{assump:G}(c) implies $B_{jk}^{(S)}
\prec N^{-1}\|\vR^{(S)}\|_F$ and $C_{jk}^{(S)} \prec N^{-2}\|(\vR^{(S)})^2\|_F
\leq N^{-2}\|\vR^{(S)}\|_F^2$.
When $j=k$, Assumption \ref{assump:G}(c) implies also
\begin{align*}
C_{jj}^{(S)} &\prec N^{-2}|\Tr \vSigma(\vR^{(S)})^2|+N^{-2}\|(\vR^{(S)})^2\|_F\\
&\leq N^{-2}\|\vSigma \vR^{(S)}\|_F \|\vR^{(S)}\|_F+N^{-2}\|\vR^{(S)}\|_F^2
\leq N^{-2}(\|\vSigma\|+1)\|\vR^{(S)}\|_F^2.
\end{align*}
Then these bounds in \eqref{eq:FAbounds} follow from the condition
$\|\vR^{(S)}\|_F \leq \Psi_N$ defining $\cE(S)$.

Finally, for the bounds for $Y_i^{(S)}[d]$ and $Z_{ijk}^{(S)}[d]$ in
\eqref{eq:FAbounds}, observe that for any matrix $\vA \in \C^{n \times n}$ 
independent of $\vg_i$, we have
$\Tr(\vg_i\vg_i^\top-\vSigma)\vA \prec \|\vA\|_F$
by Assumption \ref{assump:G}(c). Then
$Y_i^{(S)}[d] \prec \|\vR^{(S)}\vA \vOmega^{(S)}[\vSigma\vOmega^{(S)}]^d\|_F
\leq \|\vR^{(S)}\vA\|_F \cdot
\|\vOmega^{(S)}\|^{d+1}\|\vSigma\|^d$, so the bound for $Y_i^{(S)}[d]$ in
\eqref{eq:FAbounds} follows
from the conditions $\|\vR^{(S)}\vA\|_F \leq \Phi_N$ and $\|\vOmega^{(S)}\| \leq
C_0$ defining $\cE(S)$.
For $Z_{ijk}^{(S)}[d]$, similarly by Assumption \ref{assump:G}(c),
\[Z_{ijk}^{(S)} \prec N^{-1}\|\vR^{(S)}\vg_j\vg_k^\top
\vR^{(S)}\vA \vOmega^{(S)}[\vSigma\vOmega^{(S)}]^d\|_F
\leq N^{-1}\|\vR^{(S)}\vg_j\| \cdot \|\vg_k^\top
\vR^{(S)}\vA\| \cdot \|\vOmega^{(S)}\|^{d+1}\|\vSigma\|^d.\]
Applying again Assumption \ref{assump:G}(c), we have
\begin{align*}
\|\vR^{(S)}\vg_j\|_2^2
&=\vg_j^\top (\vR^{(S)})^*\vR^{(S)}\vg_j
\prec |\Tr \vSigma(\vR^{(S)})^*\vR^{(S)}|
+\|(\vR^{(S)})^*\vR^{(S)}\|_F \prec \|\vR^{(S)}\|_F^2
\end{align*}
and similarly $\|\vg_k^\top \vR^{(S)}\vA\|_2^2 \prec \|\vR^{(S)}\vA\|_F^2$.
Then the bound for $Z_{ijk}^{(S)}[d]$ in \eqref{eq:FAbounds}
follows from the conditions
$\|\vR^{(S)}\vA\|_F \leq \Phi_N$, $\|\vR^{(S)}\|_F \leq \Psi_N$,
and $\|\vOmega^{(S)}\| \leq C_0$ defining $\cE(S)$.
\end{proof}

\begin{lemma}
Fix any $L,D \geq 1$. Then there exist coefficients
$\alpha(d,d',D) \in \R$ such that the following holds:
Uniformly over $S \subset [N]$
with $|S| \leq L-1$, and over $i \in S$, $j,k \in S \setminus \{i\}$,
$l \in [N] \setminus S$, and $d \leq D$,
\begin{align}
Y_i^{(S)}[d]&=\sum_{d'=d}^{d+\lceil D/2 \rceil}
\alpha(d,d',D)\Big[C_{ll}^{(Sl)}Q_l^{(Sl)}
\Big]^{d'-d}\Big(Y_i^{(Sl)}[d']-Z_{ill}^{(Sl)}[d']Q_l^{(Sl)}\Big)
+\SDcEL{N^{-D}\Psi_N^D\Phi_N}\label{eq:FAYexpand}\\
Z_{ijk}^{(S)}[d]&=\sum_{d'=d}^{d+\lceil D/2 \rceil}
\alpha(d,d',D)\Big[C_{ll}^{(Sl)}Q_l^{(Sl)}\Big]^{d'-d}
\Big(Z_{ijk}^{(Sl)}[d']-Z_{ilk}^{(Sl)}[d']B_{lj}^{(Sl)}Q_l^{(Sl)}\nonumber\\
&\quad
-Z_{ijl}^{(Sl)}[d']B_{kl}^{(Sl)}Q_l^{(Sl)}+Z_{ill}^{(Sl)}[d']
B_{lj}^{(Sl)}B_{kl}^{(Sl)}(Q_l^{(Sl)})^2\Big)
+\SDcEL{N^{-D}\Psi_N^D\Phi_N},\label{eq:FAZexpand}\\
B_{jk}^{(S)}&=B_{jk}^{(Sl)}-B_{jl}^{(Sl)}B_{lk}^{(Sl)}Q_l^{(Sl)},\label{eq:FABexpand}\\
C_{jk}^{(S)}&=C_{jk}^{(Sl)}-B_{jl}^{(Sl)}C_{lk}^{(Sl)}Q_l^{(Sl)}
-C_{jl}^{(Sl)}B_{lk}^{(Sl)}Q_l^{(Sl)}
+B_{jl}^{(Sl)}C_{ll}^{(Sl)}B_{lk}^{(Sl)}(Q_l^{(Sl)})^2,\label{eq:FACexpand}\\
Q_j^{(S)}&=\sum_{d=1}^{\lceil D/2 \rceil}
\left(Q_j^{(Sl)}\right)^d\left[(B_{jl}^{(Sl)})^2Q_l^{(Sl)}\right]^{d-1}+\SDcEL{N^{-D}\Psi_N^D}.\label{eq:FAQexpand}
\end{align}
\end{lemma}
\begin{proof}
By the Sherman-Morrison formula, on the event $\cE_L$ where $\vR^{(S)}$ and
$\vR^{(Sl)}$ both exist, we have
\begin{equation}\label{eq:FAShermanMorrison}
\vR^{(S)}=\vR^{(Sl)}-N^{-1} \vR^{(Sl)}\vg_l\vg_l^\top \vR^{(Sl)}
\cdot Q_l^{(Sl)}.
\end{equation}
Applying this to each copy of $\vR^{(S)}$ defining $B_{jk}^{(S)}$ and
$C_{jk}^{(S)}$ yields immediately \eqref{eq:FABexpand} and \eqref{eq:FACexpand},
as well as the identities
\begin{align*}
z^{-1}\vGamma+\tilde m^{(S)}\vSigma&=z^{-1}\vGamma+\Big(N^{-1}\Tr
\vR^{(S)}+(1-\gamma_N)(-1/z)\Big)\vSigma\\
&=(z^{-1}\vGamma+\tilde m^{(Sl)}\vSigma)-C_{ll}^{(Sl)}Q_l^{(Sl)}\vSigma,\\
1+B_{jj}^{(S)}&=1+B_{jj}^{(Sl)}-(B_{jl}^{(Sl)})^2Q_l^{(Sl)}.
\end{align*}
Taking inverses and applying the expansion
\[(\vA-\vDelta)^{-1}=\sum_{d=1}^{\lceil D/2 \rceil}
\vA^{-1}(\vDelta\vA^{-1})^{d-1}
+(\vA-\vDelta)^{-1}(\vDelta\vA^{-1})^{\lceil D/2 \rceil},\]
we obtain
\begin{align}
\vOmega^{(S)}&=\sum_{d=1}^{\lceil D/2 \rceil}
\vOmega^{(Sl)}[C_{ll}^{(Sl)}Q_l^{(Sl)}
\vSigma\vOmega^{(Sl)}]^{d-1}+\vE,\label{eq:Omegaexpand}\\
Q_j^{(S)}&=\sum_{d=1}^{\lceil D/2 \rceil}
Q_j^{(Sl)}[(B_{jl}^{(Sl)})^2Q_l^{(Sl)}Q_j^{(Sl)}]^{d-1}+e,\label{eq:Qexpand}
\end{align}
for remainder terms $\vE \in \C^{n \times n}$ and $e \in \C$ satisfying,
by the bounds of Lemma \ref{lemm:bounds},
\[\|\vE\|=\SDcEL{|C_{ll}^{(Sl)}|^{D/2}}
=\SDcEL{(N^{-1}\Psi)^D}, \quad |e|=\SDcEL{|(B_{jl}^{(Sl)})^2|^{D/2}}
=\SDcEL{(N^{-1}\Psi)^D}.\]
In particular,
\eqref{eq:Qexpand} shows \eqref{eq:FAQexpand}. Applying \eqref{eq:Omegaexpand}
to the definitions of $Y_i^{(S)}[d]$ and $Z_{ijk}^{(S)}[d]$, we get
\begin{align*}
Y_i^{(S)}[d]&=\Tr (\vg_i\vg_i^\top-\vSigma)\vR^{(S)}\vA
\left(\sum_{d'=1}^{\lceil D/2 \rceil}
\vOmega^{(Sl)}[C_{ll}^{(Sl)}Q_l^{(Sl)}
\vSigma\vOmega^{(Sl)}]^{d'-1}+\vE\right)\\
&\qquad\qquad \cdot \left(\vSigma\left[\sum_{d'=1}^{\lceil D/2 \rceil}
\vOmega^{(Sl)}[C_{ll}^{(Sl)}Q_l^{(Sl)}
\vSigma\vOmega^{(Sl)}]^{d'-1}+\vE\right]\right)^d,\\
Z_{ijk}^{(S)}[d]&=\frac{1}{N}\Tr (\vg_i\vg_i^\top-\vSigma)\vR^{(S)}
\vg_j\vg_k^\top \vR^{(S)}\vA
\left(\sum_{d'=1}^{\lceil D/2 \rceil}
\vOmega^{(Sl)}[C_{ll}^{(Sl)}Q_l^{(Sl)}
\vSigma\vOmega^{(Sl)}]^{d'-1}+\vE\right)\\
&\qquad\qquad\cdot
\left(\vSigma\left[\sum_{d'=1}^{\lceil D/2 \rceil}
\vOmega^{(Sl)}[C_{ll}^{(Sl)}Q_l^{(Sl)}
\vSigma\vOmega^{(Sl)}]^{d'-1}+\vE\right]\right)^d.
\end{align*}
For any matrix $\vB \in \C^{n \times n}$ independent of $\vg_i$, observe that
$\Tr (\vg_i\vg_i^\top-\vSigma)\vR^{(S)}\vA\vB=\SDcES{\Phi_N\|\vB\|}$
and $\Tr (\vg_i\vg_i^\top-\vSigma)\vR^{(S)}\vg_j\vg_k^\top \vR^{(S)}\vA\vB
=\SDcES{\Psi_N\Phi_N\|\vB\|}$
by the same arguments as those bounding $Y_i^{(S)}[d]$ and $Z_{ijk}^{(S)}[d]$
in the proof of Lemma \ref{lemm:bounds}. Then,
expanding the above and absorbing all terms containing $\vE$ 
and all terms with combined power of $C_{ll}^{(Sl)}$ larger than $D/2$
into $\SDcES{N^{-D}\Psi_N^D \Phi_N}$ remainders, we obtain for some coefficients
$\alpha(d,d',D) \in \R$ that
\begin{align*}
Y_i^{(S)}[d]&=\Tr (\vg_i\vg_i^\top-\vSigma)\vR^{(S)}\vA
\sum_{d'=0}^{\lceil D/2 \rceil} \alpha(d,d',D)[C_{ll}^{(Sl)}Q_l^{(Sl)}]^{d'}
\vOmega^{(Sl)}[\vSigma\vOmega^{(Sl)}]^{d+d'}\\
&\hspace{1in}+\SDcES{N^{-D}\Psi_N^D\Phi_N},\\
Z_{ijk}^{(S)}[d]&=\frac{1}{N}\Tr (\vg_i\vg_i^\top-\vSigma)\vR^{(S)}
\vg_j\vg_k^\top \vR^{(S)}\vA
\sum_{d'=0}^{\lceil D/2 \rceil} \alpha(d,d',D)[C_{ll}^{(Sl)}Q_l^{(Sl)}]^{d'}
\vOmega^{(Sl)}[\vSigma\vOmega^{(Sl)}]^{d+d'}\\
&\hspace{1in}+\SDcES{N^{-D}\Psi_N^D\Phi_N}.
\end{align*}
Finally, applying the Sherman-Morrison formula \eqref{eq:FAShermanMorrison}
to expand each copy of $\vR^{(S)}$, and re-indexing the summations by
$d+d' \mapsto d'$, we get \eqref{eq:FAYexpand} and \eqref{eq:FAZexpand}.
\end{proof}

\begin{lemma}\label{lemm:FAexpansion}
Fix any $L,D \geq 1$. Uniformly over $S \subset [N]$ with $|S| \leq L$ and over
$i \in S$, the following holds: Denote $\bar S=S \setminus \{i\}$. Then there exists
a collection of monomials $\cM_{i,S}$ such that
$Y_i^{(i)}[0]$ can be approximated as
\begin{align}
Y_{i}^{(i)}[0]&=\sum_{q \in \cM_{i,S}}
q\bigg(\{Y_i^{(S)}[d]\}_{d \leq \lfloor D/2 \rfloor},
\{Z_{ijk}^{(S)}[d]\}_{j,k \in \bar S, d \leq \lfloor D/2 \rfloor},
\{B_{jk}^{(S)}\}_{j \neq k \in \bar S},\nonumber\\
&\hspace{2in}\{C_{jk}^{(S)}\}_{j,k\in \bar S},\{Q_j^{(S)}\}_{j \in \bar S}\bigg)
+\SDcEL{N^{-D}\Psi_N^D\Phi_N}.\label{eq:FAexpansionrepr}
\end{align}
Each monomial $q \in \cM_{i,S}$ is a product of a real-valued scalar coefficient
and one or more factors of the form
$Y_i^{(S)}[d]$, $Z_{ijk}^{(S)}[d]$, $B_{jk}^{(S)}$ with $j \neq k$,
$C_{jk}^{(S)}$, $Q_j^{(S)}$ for $j,k \in \bar S$ and $d \leq \lfloor D/2 \rfloor$.
We have $q=\SDcEL{\Phi_N}$ uniformly over $q \in \cM_{i,S}$,
and the number of monomials $|\cM_{i,S}|$ is most a constant depending on
$L,D$. Furthermore:
\begin{enumerate}[label=(\alph*)]
\item There is exactly one factor of the form $Y_i^{(S)}[d]$ or
$Z_{ijk}^{(S)}[d]$ appearing in $q$.
\item The number of factors $Z_{ijk}^{(S)}[d]$,
$B_{jk}^{(S)}$, and $C_{jk}^{(S)}$ appearing in $q$ is no less than the number
of distinct indices of $\bar S$ (not including $i$) that appear as lower
indices across all factors of $q$.
\end{enumerate}
\end{lemma}
\begin{proof}
We arbitrarily order the indices of $\bar S=S \setminus \{i\}$ as
$l_1,l_2,\ldots,l_{|S|-1}$. Beginning with the monomial $Y_i^{(i)}[0]$,
iteratively for $j=1,2,\ldots,|S|-1$, we
replace all factors with superscript $(il_1\ldots l_{j-1})$
by a sum of terms with superscript $(il_1\ldots l_j)$,
using the recursions \eqref{eq:FAYexpand}--\eqref{eq:FAQexpand}. It is then
direct to check that
this gives a representation of the form \eqref{eq:FAexpansionrepr}, where:
\begin{itemize}
\item Each application of \eqref{eq:FAYexpand}--\eqref{eq:FAZexpand} replaces a
factor $Y_i^{(\ldots)}[d]$ or $Z_{ijk}^{(\ldots)}[d]$ by terms
having exactly one such factor. Thus, each monomial $q \in \cM_{i,S}$
has exactly one factor $Y_i^{(S)}[d]$ or $Z_{ijk}^{(S)}[d]$.
\item The number of total applications of
\eqref{eq:FAYexpand}--\eqref{eq:FAQexpand} is bounded by a constant depending on
$L,D$, so $|\cM_{i,S}|$ and the scalar coefficient of each $q \in \cM_{i,S}$
are both bounded by constants depending on $L,D$. Then, by the bounds of
\eqref{eq:FAbounds}, each $q \in \cM_{i,S}$ satisfies
$q=\SDcEL{\Phi_N}$, and the remainder in \eqref{eq:FAexpansionrepr} is at most
$\SDcEL{N^{-D}\Psi_N^D\Phi_N}$. If $q$ has the term $Y_i^{(S)}[d]$ or
$Z_i^{(S)}[d]$, then it also has combined power of
$\{C_{jk}^{(S)}\}_{j,k \in \bar S}$ equal to $d$, and hence may be
absorbed into the remainder of \eqref{eq:FAexpansionrepr} if $d>D/2$.
\item Each term on the right side of \eqref{eq:FAYexpand}--\eqref{eq:FAQexpand}
that contains the new lower index $l$ has at least one more factor of the
form $Z_{ijk}^{(\ldots)}[d]$, $B_{jk}^{(\ldots)}$, or $C_{jk}^{(\ldots)}$ than the left side. Thus, each monomial $q \in \cM_{i,S}$ is such
that the number of distinct lower indices of $\bar S$ across all of its factors
is no greater than the number of its factors of the form
$Z_{ijk}^{(\ldots)}[d]$, $B_{jk}^{(\ldots)}$, or $C_{jk}^{(\ldots)}$.
\end{itemize}
Combining these observations yields the lemma.
\end{proof}

\begin{proof-of-lemma}[\ref{lemma:fluctuationavg}]
For each $\eps,D>0$, let us fix an even integer $L=L(\eps,D)>D/\eps$.
The assumption of this lemma guarantees
$\bI\{\cE(S)^c\} \prec 0$ uniformly over $S \subset [N]$ with $|S|
\leq L$. Since the number of such subsets is at most $N^L$, we may take a union
bound (c.f.\ Proposition \ref{prop:domination}(a))
to obtain $\bI\{\cE_L^c\} \prec 0$ for the intersection event
$\cE_L$ of \eqref{eq:EL}. Noting that
$(1-\E_{\vg_i})[\vg_i^\top \vR^{(i)}\vA\vOmega\vg_i]=Y_i^{(i)}[0]$,
to prove the lemma, it suffices to show for any $\eps,D>0$ and
all sufficiently large $N$ that
\begin{equation}\label{eq:FAgoal}
\P\left[\left(\frac{1}{N}\sum_{i=1}^N 
Y_i^{(i)}[0]\right)
\bI\{\cE_L\}>\max\left(\frac{\Psi_N}{N},\frac{1}{\sqrt{N}}\right)\Phi_N
\cdot N^\eps \right]<N^{-D}.
\end{equation}

In anticipation of applying Markov's inequality, we analyze
\begin{equation}\label{eq:EYQL}
\E\left[\left(\sum_{i=1}^{N} Y_i^{(i)}[0]\right)^L
\bI\{\cE_L\}\right]
=\sum_{i_1,\ldots,i_L=1}^{N} \underbrace{\E\left[\prod_{l=1}^L Y_{i_l}^{(i_l)}
[0] \bI\{\cE_L\}\right]}_{:=\E[m(i_1,\ldots,i_L)]}.
\end{equation}
Fix any index tuple $(i_1,\ldots,i_L)$.
Letting $S=\{i_1,\ldots,i_L\}$ be the set of distinct indices in this tuple,
we apply Lemma \ref{lemm:FAexpansion} to each term
$Y_{i_l}^{(i_l)}[0]$, with this set $S$ and with $D=L$. This gives
\begin{equation}\label{eq:mexpand}
m(i_1,\ldots,i_L)=\sum_{q^{(1)} \in \cM(i_1,S)} \ldots
\sum_{q^{(l)} \in \cM(i_l,S)} \prod_{l=1}^L q^{(l)} \cdot \bI\{\cE_L\}
+\SD{(N^{-1}\Psi_N)^L\Phi_N^L},
\end{equation}
where each $\cM(i_l,S)$ is the collection of monomials arising in the
approximation of $Y_{i_l}^{(i_l)}[0]$,
and we have applied $q^{(l)}=\SDcEL{\Phi_N}$ to bound the remainder.
Observe that by \eqref{eq:FAEbounds} and Holder's inequality, we have
$\E[|m(i_1,\ldots,i_L)|^2] \leq N^C$ and
$\E[|\prod_{l=1}^L q^{(l)} \cdot \bI\{\cE_L\}|^2] \leq N^C$
for all $q^{(1)},\ldots,q^{(L)}$ and a constant $C>0$.
By this and the given condition $\Psi_N,\Phi_N \geq N^{-\upsilon}$,
we may take expectations in \eqref{eq:mexpand} using
Proposition \ref{prop:domination}(d) to get
\begin{equation}\label{eq:miLbound}
\E[m(i_1,\ldots,i_L)]=\sum_{q^{(1)} \in \cM(i_1,S)} \ldots
\sum_{q^{(l)} \in \cM(i_l,S)} \E\left[\prod_{l=1}^L q^{(l)} \cdot
\bI\{\cE_L\}\right]+\SD{(N^{-1}\Psi_N)^L\Phi_N^L}.
\end{equation}

Now to bound $\E[\prod_{l=1}^L q^{(l)} \cdot \bI\{\cE_L\}]$,
we consider separately two cases, focusing on those indices $i_l$ which
appear exactly once in $(i_1,\ldots,i_L)$.
In the first case, suppose there is some such index $i_l$ that does not appear
as a lower index of $q^{(l')}$ for any $l' \neq l$. Fixing this set
$S=\{i_1,\ldots,i_L\}$ and index $i_l \in S$, let us introduce
\begin{align*}
\cE'&=\bigg\{\vR^{(S)} \text{ exists, } \|\vR^{(S)}\vA\|_F \leq \Phi_N,\;
\|\vR^{(S)}\|_F \leq \Psi_N,\;
\|(z^{-1}\vGamma+\tilde m^{(S)} \vSigma)^{-1}\| \leq C_0,\\
&\hspace{2in} \text{ and }
|1+N^{-1} \vg_j^\top \vR^{(S)}\vg_j| \geq c_0 \text{ for all } j \in S \setminus
\{i_l\} \bigg\}.
\end{align*}
Comparing
with the definition of $\cE(S)$ from \eqref{eq:ESdef}, observe that
only the last condition defining $\cE'$ is different (where we do not require
the bound for $j=i_l$), so that this event $\cE'$ is independent of $\vg_{i_l}$.
Then $\cE_L \subseteq \cE(S) \subseteq \cE'$, and
\begin{equation}\label{eq:eventdecomp}
\E\left[\prod_{l=1}^L q^{(l)} \cdot \bI\{\cE_L\}\right]
=\E\left[\prod_{l=1}^L q^{(l)} \cdot \bI\{\cE'\}\right]
-\E\left[\prod_{l=1}^L q^{(l)} \cdot \bI\{\cE'\}\bI\{\cE_L^c\}
\right].
\end{equation}
For the first term of \eqref{eq:eventdecomp}, observe that
both $\{q^{(l')}:l' \neq l\}$ and $\cE'$ are independent of
$\vg_{i_l}$, and only the one
factor $Y_{i_l}^{(S)}[d]$ or $Z_{i_ljk}^{(S)}[d]$ in $q^{(l)}$ depends on
$\vg_{i_l}$. Then, noting that $\E_{\vg_i}[Y_i^{(S)}[d]]=0$ and
$\E_{\vg_i}[Z_{ijk}^{(S)}[d]]=0$, the first term of \eqref{eq:eventdecomp}
is 0. For the second term of \eqref{eq:eventdecomp},
observe that all statements of \eqref{eq:FAEbounds} continue to hold with
$\cE(S)$ replaced by $\cE'$, except for the bound on $Q_{i_l}^{(S)}$.
But $Q_{i_l}^{(S)}$ appears neither in
$\{q^{(l')}:l' \neq l\}$ nor in $q^{(l)}$, so we may apply Holder's inequality
to get $\E[|\prod_{l=1}^L q^{(l)}|^2 \bI\{\cE'\}] \leq N^C$ for a constant
$C>0$. Then, applying Cauchy-Schwarz and $\bI\{\cE_L^c\} \prec 0$, the second
term of \eqref{eq:eventdecomp} is bounded by $N^{-D'}$ for any fixed constant
$D'>0$ and all large $N$. Thus,
\begin{equation}\label{eq:FAcase1}
\E\left[\prod_{l=1}^L q^{(l)} \cdot \bI\{\cE_L\}\right] \leq N^{-D'}.
\end{equation}

In the second case,
every index $i_l$ that appears exactly once in $(i_1,\ldots,i_L)$ appears
as a lower index of $q^{(l')}$ for some $l' \neq l$. Call the number
of such indices $K$. Then condition (b) of Lemma \ref{lemm:FAexpansion} implies
that the total number of factors of the forms
$Z_{ijk}^{(S)}[d]$, $B_{jk}^{(S)}$ for $j \neq k$, and $C_{jk}^{(S)}$
across all monomials $q^{(1)},\ldots,q^{(L)}$ is at least $K$. Then, 
by the bounds of Lemma \ref{lemm:bounds} and
Proposition \ref{prop:domination}(d), we have
\begin{equation}\label{eq:FAcase2}
\E\left[\prod_{l=1}^L q^{(l)} \cdot \bI\{\cE_L\}\right]
\prec (N^{-1}\Psi_N)^K\Phi_N^L.
\end{equation}

Under the given condition $\Phi_N,\Psi_N \geq N^{-\upsilon}$, we have
$N^{-D'} \leq (N^{-1}\Psi_N)^K\Phi_N^L$ for large enough $D'$. Then,
combining the two cases \eqref{eq:FAcase1} and \eqref{eq:FAcase2} and applying
this back to \eqref{eq:miLbound}, we get
\begin{equation}\label{eq:miLboundfinal}
\E[m(i_1,\ldots,i_L)] \prec (N^{-1}\Psi_N)^K\Phi_N^L
\end{equation}
where $K$ is the number of indices in $S=\{i_1,\ldots,i_L\}$ that appear exactly
once in $(i_1,\ldots,i_L)$.
Let $J$ be the number of distinct indices in $S=\{i_1,\ldots,i_L\}$ that appear
at least twice in $(i_1,\ldots,i_L)$. Then $2J+K \leq L$, and the number of
index tuples $(i_1,\ldots,i_L) \in [N]^L$ with these values of $(J,K)$ is at
most $CN^{J+K}$, for a constant $C=C(J,K)>0$. Then,
applying \eqref{eq:miLboundfinal} back to \eqref{eq:EYQL} yields
\begin{align*}
\E\left[\left(\sum_{i=1}^{N} Y_i^{(i)}[0]\right)^L
\bI\{\cE_L\}\right] &\prec \max_{J,K \geq 0:\,2J+K \leq L}
N^{J+K} \cdot (N^{-1}\Psi_N)^K\Phi_N^L\\
&=\max_{J,K \geq 0:\,2J+K \leq L}
(\sqrt{N})^{2J}\Psi_N^K\Phi_N^L \leq \max(\Psi_N,\sqrt{N})^L \Phi_N^L.
\end{align*}
Finally, by Markov's inequality, the probability in \eqref{eq:FAgoal} is at most
\begin{align*}
\max(\Psi_N,\sqrt{N})^{-L}\Phi_N^{-L}N^{-\eps L}
\cdot \E\left[\left(\sum_{i=1}^{N} Y_i^{(i)}[0]\right)^L
\bI\{\cE_L\}\right] \prec N^{-\eps L},
\end{align*}
and \eqref{eq:FAgoal} follows as desired under our initial
choice $L=L(\eps,D)>D/\eps$.
\end{proof-of-lemma}

\subsection{No eigenvalues outside the support}\label{sec:nooutliers}

We now prove Theorem \ref{thm:no_outlier}.
Let $m_N(z),\tilde m_N(z)$ be the Stieltjes
transform of the $N$-dependent deterministic measures $\mu_N,\tilde \mu_N$.
For each $z \in \C^+$, $\tilde m_N(z)$ is the unique root in $\C^+$ to the equation
\begin{align}
z&=-\frac{1}{\tilde m_N(z)}+\gamma_N \int \frac{\lambda}{1+\lambda \tilde
m_N(z)} d\nu_N(\lambda),\label{eq:tildemn}
\end{align}
and $m_N(z),\tilde m_N(z)$ are related by
$\tilde m_N(z)=\gamma_N m_N(z)+(1-\gamma_N)({-}1/z)$.
Define the discrete set
\begin{equation}\label{eq:zNdomain}
\cT_N=\{0\} \cup \{-1/\lambda:\lambda \in \supp{\nu_N}\}.
\end{equation}
On the domain $\C \setminus \cT_N$,
we may define the formal inverse of (\ref{eq:tildemn}),
\begin{align}
z_N(\tilde m)&=-\frac{1}{\tilde m}+\gamma_N \int \frac{\lambda}{1+\lambda \tilde
m} d\nu_N(\lambda),\label{eq:z_N_def}
\end{align}
which is a finite-$N$ analogue of (\ref{eq:inv_z}).
Let $\cS_N$ be the deterministic support defined in \eqref{eq:support}, and let
$U_N(\eps)$ be the spectral domain \eqref{eq:U_e}. The following basic
properties of $\cS_N$ and $\tilde m_N(z)$ are known.

\begin{proposition}\label{prop:basic_mn}
Suppose Assumption \ref{assump:G}(a) holds, and fix any $\eps>0$. Then there
exist constants $C_0,c_0>0$, depending only on $\eps$ and the constants $C,c$ of
Assumption \ref{assump:G}(a), such that for all $x \in \cS_N$ we have $|x| \leq C$,
and for all $z=x+i\eta \in U_N(\eps)$ we have
\[c<|\tilde m_N(z)|<C, \qquad c\eta \leq |\Im \tilde m_N(z)| \leq C\eta,
\qquad \min_{\lambda \in \supp{\nu_N}}
|1+\lambda\,\tilde m_N(z)| \geq c\]
\end{proposition}
\begin{proof}
See \cite[Propositions A.3, B.1, B.2]{fan2022tracy}.
\end{proof}

Let $m_{\tilde \vK}(z)=N^{-1}\Tr (\widetilde \vK-z\vI)^{-1}$ be the Stieltjes
transform of the empirical eigenvalue distribution of $\widetilde \vK=\vG\vG^\top$.
Since $\widetilde \vK$ and $\vK=\vG^\top \vG$ have the same
eigenvalues up to $|N-n|$ 0's, we have
\begin{equation}\label{eq:m_tildeK}
m_{\tilde \vK}(z)=\gamma_N m_{\vK}(z)+(1-\gamma_N)(-1/z),
\end{equation}
so in particular $m_{\tilde \vK}$ coincides with $\tilde m_{\vK}^{(\emptyset)}$
from \eqref{eq:leaveoneout}.
We begin with a preliminary estimate for the Stieltjes transform
$m_{\tilde \vK}(z)$ when $\Im z \geq N^{-1/11}$. Similar statements have been shown
in \citep{silverstein1995strong,bai1998no}, and we provide an argument
here following ideas of \cite[Section 3]{bai1998no} for later reference.

\begin{lemma}\label{lemma:prelimestimate}
Fix any $\eps>0$, and suppose Assumption \ref{assump:G} holds.
Then, uniformly over $z=x+i\eta \in U_N(\eps)$ with $\Im z \geq N^{-1/11}$,
\[m_{\tilde \vK}(z)-\tilde m_N(z) \prec \frac{1}{\sqrt{N}\eta^4}.\]
\end{lemma}
\begin{proof}
Let $\vR^{(i)}$ and $\tilde m_{\vK}^{(i)}$ be as defined in
\eqref{eq:leaveoneout} with $\vGamma=z\vI$.
Applying the Sherman-Morrison formula
\begin{equation}\label{eq:ShermanMorrison}
\vR=\vR^{(i)}-\frac{N^{-1} \vR^{(i)}\vg_i\vg_i^\top \vR^{(i)}}
{1+N^{-1}\vg_i^\top \vR^{(i)}\vg_i},
\end{equation}
for any matrix $\vB \in \C^{n \times n}$ we have
\begin{align}
\Tr \vB=\Tr (\vK-z\vI)\vR\vB
&=-z\Tr \vR\vB+\frac{1}{N}\sum_{i=1}^N \vg_i^\top \vR\vB\vg_i\nonumber\\
&=-z\Tr \vR\vB+\frac{1}{N}\sum_{i=1}^N \frac{\vg_i^\top \vR^{(i)}\vB\vg_i}
{1+N^{-1}\vg_i^\top \vR^{(i)}\vg_i}.\label{eq:fundamentalidentity}
\end{align}

Choosing $\vB=\vI$ in \eqref{eq:fundamentalidentity},
applying $\Tr \vR=n\,m_{\vK} =N\,m_{\tilde \vK}+(n-N)(-1/z)$, and rearranging,
we obtain the identity
\begin{equation}\label{eq:mnidentity}
m_{\tilde \vK}=-\frac{1}{Nz}\sum_{i=1}^N \frac{1}{1+N^{-1} \vg_i^\top
\vR^{(i)}\vg_i}.
\end{equation}
Now fix any deterministic matrix $\vA \in \C^{n \times n}$, define
\[d_i=\frac{1}{N}\,\vg_i^\top \vR^{(i)}\vA
(\vI+m_{\tilde \vK}\vSigma)^{-1}\vg_i
-\frac{1}{N}\Tr \vR\vA(\vI+m_{\tilde \vK}\vSigma)^{-1}\vSigma,\]
and choose $\vB=\vA(\vI+m_{\tilde \vK}\vSigma)^{-1}$
in \eqref{eq:fundamentalidentity}. Then, applying also the identity
\eqref{eq:mnidentity}, we get
\begin{align}
&\Tr \vA(\vI+m_{\tilde \vK}\vSigma)^{-1}\nonumber\\
&=-z\Tr \vR\vA(\vI+m_{\tilde \vK}\vSigma)^{-1}
-zm_{\tilde \vK} \Tr \vR\vA(\vI+m_{\tilde \vK}\vSigma)^{-1}\vSigma
+\sum_{i=1}^N \frac{d_i}{1+N^{-1}\vg_i^\top \vR^{(i)}\vg_i}\nonumber\\
&=-z \Tr \vR\vA+\sum_{i=1}^N \frac{d_i}{1+N^{-1}\vg_i^\top \vR^{(i)}\vg_i}.
\label{eq:detequividentity}
\end{align}

We proceed to bound $d_i$, where (for later purposes)
we derive estimates in terms of the Frobenius norms of
$\vR,\vR\vA,\vR^{(i)},\vR^{(i)}\vA$ rather than their operator norms.
Note that Assumption \ref{assump:G}(c) implies, for any matrix
$\vB \in \C^{n \times n}$ independent of $\vg_i$,
\begin{equation}\label{eq:Agbound}
\|\vB\vg_i\|^2=\vg_i^\top \vB^*\vB \vg_i \prec
\Tr \vSigma\vB^*\vB+\|\vB^*\vB\|_F \prec \|\vB\|_F^2.
\end{equation}
We have also, by Assumption \ref{assump:G}(c) and the Sherman-Morrison
formula \eqref{eq:ShermanMorrison},
\begin{align}
N^{-1}|\Tr \vR\vB-\Tr \vR^{(i)}\vB|
&=N^{-2}|1+N^{-1}\vg_i^\top \vR^{(i)}\vg_i|^{-1}
|\vg_i^\top \vR^{(i)}\vB\vR^{(i)}\vg_i|\nonumber\\
&\prec N^{-2}|1+N^{-1}\vg_i^\top \vR^{(i)}\vg_i|^{-1}
\Big(|\Tr \vSigma\vR^{(i)}\vB\vR^{(i)}|+\|\vR^{(i)}\vB\vR^{(i)}\|_F\Big)\nonumber\\
&\prec N^{-2}|1+N^{-1}\vg_i^\top \vR^{(i)}\vg_i|^{-1}\|\vR^{(i)}\vB\|_F 
\|\vR^{(i)}\|_F.
\label{eq:mdiffestimate}
\end{align}

Define
$d_i=d_{i,1}+d_{i,2}+d_{i,3}+d_{i,4}$ where
\begin{equation}\label{eq:di14}
\begin{aligned}
d_{i,1}&=N^{-1}\vg_i^\top \vR^{(i)}\vA(\vI+m_{\tilde \vK}\vSigma)^{-1}
\vg_i-N^{-1}\vg_i^\top \vR^{(i)}\vA(\vI+\tilde m_{\vK}^{(i)}\vSigma)^{-1}\vg_i,\\
d_{i,2}&=N^{-1}\vg_i^\top \vR^{(i)}\vA(\vI+\tilde m_{\vK}^{(i)}\vSigma)^{-1}
\vg_i-N^{-1}\Tr \vSigma \vR^{(i)}\vA(\vI+\tilde m_{\vK}^{(i)}\vSigma)^{-1},\\
d_{i,3}&=N^{-1}\Tr \vSigma \vR^{(i)}\vA(\vI+\tilde m_{\vK}^{(i)}\vSigma)^{-1}
-N^{-1}\Tr \vSigma \vR\vA(\vI+\tilde m_{\vK}^{(i)}\vSigma)^{-1},\\
d_{i,4}&=N^{-1}\Tr \vSigma \vR\vA(\vI+\tilde m_{\vK}^{(i)}\vSigma)^{-1}
-N^{-1}\Tr \vSigma \vR\vA(\vI+m_{\tilde \vK}\vSigma)^{-1}.
\end{aligned}
\end{equation}
Applying the identity $\vA^{-1}-\vB^{-1}=\vA^{-1}(\vB-\vA)\vB^{-1}$,
the definition of $\tilde m_{\vK}^{(i)}$ in \eqref{eq:leaveoneout},
and the bounds \eqref{eq:Agbound} and \eqref{eq:mdiffestimate} (the latter with
$\vB=\vI$),
\begin{align}
|d_{i,1}| &\leq N^{-1} \|\vg_i^\top \vR^{(i)}\vA\|
\|(\vI+m_{\tilde \vK}\vSigma)^{-1}\| \|(\tilde m_{\vK}^{(i)}-m_{\tilde
\vK})\vSigma\| \|(\vI+\tilde m_{\vK}^{(i)}\vSigma)^{-1}\| \|\vg_i\|\nonumber\\
&\prec N^{-5/2} |1+N^{-1}\vg_i^\top \vR^{(i)}\vg_i|^{-1} 
\|\vR^{(i)}\vA\|_F \|\vR^{(i)}\|_F^2 \|(\vI+m_{\tilde \vK}\vSigma)^{-1}\|
\|(\vI+\tilde m_{\vK}^{(i)}\vSigma)^{-1}\|.\label{eq:di1}
\end{align}
Applying Assumption \ref{assump:G}(c),
\begin{align}
|d_{i,2}| &\prec N^{-1} \|\vR^{(i)}\vA(\vI+\tilde m_{\vK}^{(i)}\vSigma)^{-1}\|_F
\leq N^{-1} \|\vR^{(i)}\vA\|_F \|(\vI+\tilde m_{\vK}^{(i)}\vSigma)^{-1}\|.\label{eq:di2}
\end{align}
Applying the Sherman-Morrison identity \eqref{eq:ShermanMorrison},
$|\Tr \vu\vv^\top| \leq \|\vu\| \|\vv\|$, and \eqref{eq:Agbound},
\begin{align}
|d_{i,3}| & \leq N^{-2} |1+N^{-1}\vg_i^\top \vR^{(i)}\vg_i|^{-1}
\|\vSigma \vR^{(i)}\vg_i\| \|\vg_i^\top \vA
\vR^{(i)}(\vI+\tilde m_{\vK}^{(i)}\vSigma)^{-1}\|\nonumber\\
&\prec N^{-2} |1+N^{-1}\vg_i^\top \vR^{(i)} \vg_i|^{-1}
\|\vR^{(i)}\vA\|_F \|\vR^{(i)}\|_F \|(\vI+\tilde m_{\vK}^{(i)}\vSigma)^{-1}\|.\label{eq:di3}
\end{align}
Finally, applying $\vA^{-1}-\vB^{-1}=\vA^{-1}(\vB-\vA)\vB^{-1}$,
\eqref{eq:mdiffestimate} (with $\vB=\vI$), and 
$|\Tr \vA\vB| \leq \|\vA\|_F\|\vB\|_F \leq \sqrt{N} \|\vA\|_F\|\vB\|$,
\begin{align}
|d_{i,4}| &=N^{-1} \Big|\Tr \vSigma \vR\vA
(\vI+\tilde m_{\vK}^{(i)}\vSigma)^{-1} (\tilde m_{\vK}^{(i)}-m_{\tilde \vK})\vSigma
(\vI+m_{\tilde \vK}\vSigma)^{-1}\Big|\nonumber\\
&\prec N^{-5/2} |1+N^{-1}\vg_i^\top \vR^{(i)}\vg_i|^{-1} 
\|\vR\vA\|_F \|\vR\|_F^2
\|(\vI+\tilde m_{\vK}^{(i)}\vSigma)^{-1}\| 
\|(\vI+m_{\tilde \vK}\vSigma)^{-1}\|.\label{eq:di4}
\end{align}

For the current proof, we apply \eqref{eq:detequividentity} and the definitions
\eqref{eq:di14} with $\vA=\vI$.
Recalling $\Tr \vR=n\,m_{\vK}=N\,m_{\tilde \vK}+(n-N)(-1/z)$ and
rearranging \eqref{eq:detequividentity} with $\vA=\vI$, we get the identity
\begin{equation}\label{eq:Deltaidentity}
z_N(m_{\tilde \vK})-z={-}\frac{1}{m_{\tilde \vK}}
\cdot \frac{1}{N}\sum_{i=1}^N \frac{d_i}{1+N^{-1}\vg_i^\top
\vR^{(i)}\vg_i}
\end{equation}
where $z_N(m)=-(1/m)+N^{-1}\Tr \vSigma(\vI+m\vSigma)^{-1}$ is the function
defined in \eqref{eq:z_N_def}. For any $z=x+i\eta$ with $\eta>0$, we have
\begin{align}
|z(1+N^{-1}\vg_i^\top \vR^{(i)}\vg_i)|
&\geq \Im [z(1+N^{-1}\vg_i^\top \vR^{(i)}\vg_i)] \geq \Im z=\eta,
\label{eq:prelimbound1}\\
\max(\|\vR\|_F,\|\vR^{(i)}\|_F) &\leq N^{1/2}\max(\|\vR\|,\|\vR^{(i)}\|)
\leq N^{1/2}\eta^{-1}.\label{eq:prelimbound2}
\end{align}
Here, the second inequalities of both \eqref{eq:prelimbound1} and
\eqref{eq:prelimbound2} follow from the spectral representations of
$\vR,\vR^{(i)}$, i.e.\ writing $(\lambda_j,\vv_j)_{j=1}^n$ for the
eigenvalues and unit eigenvectors of $\vK^{(i)}$, we have
\begin{align*}
\Im [z\vg_i^\top \vR^{(i)}\vg_i]
&=\Im \left[z \vg_i^\top\left(\sum_{j=1}^n
\frac{1}{\lambda_j-z}\vv_j\vv_j^\top\right)
\vg_i\right]
=\sum_{j=1}^n \Im \frac{z}{\lambda_j-z} \cdot (\vg_i^\top \vv_j)^2\\
&=\sum_{j=1}^n \frac{\lambda_j \Im z}{|\lambda_j-z|^2} \cdot (\vg_i^\top \vv_j)^2
\geq 0,\\
\|\vR^{(i)}\|&=\left\|\sum_{j=1}^n
\frac{1}{\lambda_j-z}\vv_j\vv_j^\top\right\|=\max_{j=1}^n |\lambda_j-z|^{-1}
\leq \eta^{-1},
\end{align*}
and similarly for $\|\vR\|$. In particular, \eqref{eq:prelimbound1} and
\eqref{eq:prelimbound2} imply
\begin{equation}\label{eq:prelimbounds12}
(1+N^{-1}\vg_i^\top \vR^{(i)}\vg_i)^{-1} \prec \eta^{-1},
\qquad \|\vR\|_F,\|\vR^{(i)}\|_F \prec N^{1/2}\eta^{-1}.
\end{equation}
Next, observe that if $m(z)=\int
\frac{1}{\lambda-z}d\mu(\lambda)$ is the Stieltjes transform of any probability
measure $\mu$ supported on $[-B,B]$, then for $z=x+i\eta$ with $\eta>0$ and
$|z| \leq \eps^{-1}$, we have
\[\Im m(z)=\int \frac{\eta}{|\lambda-z|^2}d\mu(\lambda) \geq c\eta, \quad
|\Re m(z)| \leq \int \frac{|\lambda-x|}{|\lambda-z|^2}d\mu(\lambda)
\leq (C/\eta)\Im m(z)\]
for some constants $C,c>0$ depending on $\eps,B$.
Consequently, for any $\lambda \geq 0$, either
$\lambda \cdot |\Re m(z)|<1/2$ or $\lambda \cdot \Im m(z) \geq 2\eta/C$,
so $|1+\lambda m(z)| \geq \max(2,2\eta/C)$. By Assumption \ref{assump:G}(b)
and Weyl's inequality, we have $\bI\{\|\vK\|>B\} \prec 0$ and
$\bI\{\|\vK^{(i)}\|>B\} \prec 0$, and on the event where $\|\vK\|,\|\vK^{(i)}\|
\leq B$, we have that $m_{\tilde \vK},\tilde m_{\vK}^{(i)}$ are Stieltjes
transforms of probability measures supported on $[-B,B]$. Thus, this implies
\begin{equation}\label{eq:prelimbound3}
|m_{\tilde \vK}|^{-1} \leq |\Im m_{\tilde \vK}|^{-1} \prec \eta^{-1},
\quad \max(\|(\vI+m_{\tilde \vK}\vSigma)^{-1}\|,
\|(\vI+\tilde m_{\vK}^{(i)}\vSigma)^{-1}\|) \prec \eta^{-1}.
\end{equation}
Applying these bounds \eqref{eq:prelimbounds12} and
\eqref{eq:prelimbound3} to \eqref{eq:di1}--\eqref{eq:di4},
we get $d_i \prec N^{-1}\eta^{-6}+N^{-1/2}\eta^{-2} \leq 2N^{-1/2}\eta^{-2}$ for
$\eta \geq N^{-1/11}$. Then, applying these bounds 
\eqref{eq:prelimbounds12} and \eqref{eq:prelimbound3} also to
\eqref{eq:Deltaidentity}, we get
\begin{equation}\label{eq:zNestimateprelim}
z_N(m_{\tilde \vK})-z \prec \frac{1}{\sqrt{N}\eta^4}.
\end{equation}

The proof is completed by the following stability argument: When
$\eta \geq N^{-1/11}$, we have $1/(\sqrt{N} \eta^4) \ll \eta=\Im z$,
so \eqref{eq:zNestimateprelim} implies in particular that
\begin{equation}\label{eq:stabilityevent}
\bI\{z_N(m_{\tilde \vK}) \notin \C^+\} \prec 0.
\end{equation}
On the event $z_N(m_{\tilde \vK}) \in \C^+$, recalling the implicit definition
of $\tilde m_N:\C^+ \to \C^+$ by \eqref{eq:tildemn}, 
the value $\tilde m_N(z_N(m_{\tilde \vK}))$ must be the unique root
$u \in \C^+$ to the equation
\[z_N(m_{\tilde \vK})=-\frac{1}{u}+\gamma_N \int \frac{\lambda}{1+\lambda
u}d\nu_N(\lambda),\]
i.e.\ to the equation $z_N(m_{\tilde \vK})=z_N(u)$. This equation is
satisfied by $u=m_{\tilde \vK} \in \C^+$, so we deduce that
$\tilde m_N(z_N(m_{\tilde \vK}))=m_{\tilde \vK}$.
Then, applying that $z \in U_N(\eps)$ and that $\tilde m_N:\C^+ \to \C^+$
is $(4/\eps^2)$-Lipschitz over the domain $U_N(\eps/2)$,
we obtain from \eqref{eq:zNestimateprelim} that
\[\bI\{z_N(m_{\tilde \vK}) \in \C^+\}\Big(m_{\tilde \vK}-\tilde m_N(z)\Big)
=\bI\{z_N(m_{\tilde \vK}) \in \C^+\}\Big(\tilde m_N(z_N(m_{\tilde \vK}))-\tilde
m_N(z)\Big) \prec \frac{1}{\sqrt{N} \eta^4}.\]
Together with \eqref{eq:stabilityevent}, this yields the lemma.
\end{proof}

\begin{coro}\label{coro:Fnormbound}
Fix any $\eps>0$, and suppose Assumption \ref{assump:G} holds.
Then there is a constant $C>0$ such that uniformly over $z \in U_N(\eps)$
with $\Im z \geq N^{-1/11}$,
\[\bI\{\|\vR(z)\|_F>C\sqrt{N}\} \prec 0.\]
\end{coro}
\begin{proof}
Since $m_{\tilde \vK}(z)=\gamma_N m_{\vK}(z)+(1-\gamma_N)(-1/z)$ and
$\tilde m_N(z)=\gamma_N m_N(z)+(1-\gamma_N)(-1/z)$, Lemma
\ref{lemma:prelimestimate} implies also
\[m_{\vK}(z)-m_N(z) \prec \frac{1}{\sqrt{N}\eta^4} \ll \eta.\]
Observe that
$\Im m_N(z)=\int \eta/|\lambda-z|^2 d\mu_N(\lambda) \leq \eta \eps^{-2}$
for $z \in U_N(\eps)$, so
$\bI\{\Im m_{\vK}(z)>(1+\eps^{-2})\eta\} \prec 0$.
Then by the identity $\|\vR(z)\|_F^2=
\sum_i 1/|z-\lambda_i(\vK)|^2=(n/\eta)\Im m_{\vK}(z)$, we get
$\bI\{\|\vR(z)\|_F>C\sqrt{N}\} \prec 0$ for a constant $C=C(\eps)>0$, as desired.
\end{proof}

We may now apply Corollary \ref{coro:Fnormbound} and the fluctuation averaging
result of Lemma \ref{lemma:fluctuationavg} 
to improve the estimate of Lemma \ref{lemma:prelimestimate} to the following
result.

\begin{lemma}\label{lemma:improvedestimate}
Fix any $\eps>0$, and suppose Assumption \ref{assump:G} holds.
Then, uniformly over $z=x+i\eta \in U_N(\eps)$ with $\Im z \geq N^{-1/11}$,
\[m_{\tilde \vK}(z)-\tilde m_N(z) \prec \frac{1}{N}.\]
\end{lemma}
\begin{proof}
We derive an improved estimate for \eqref{eq:Deltaidentity}. First,
combining Lemma \ref{lemma:prelimestimate} with the bounds for $\tilde m_N(z)$
in Proposition \ref{prop:basic_mn}, there are constants $C_0,c_0>0$ for which
\begin{equation}\label{eq:tildembounds}
\bI\{|m_{\tilde \vK}|>C_0\} \prec 0,
\qquad \bI\{|m_{\tilde \vK}|<c_0\} \prec 0,
\qquad \bI\{\|(\vI+m_{\tilde \vK}\vSigma)^{-1}\|>C_0\} \prec 0
\end{equation}
uniformly over $z \in U_N(\eps)$ with $\Im z \geq N^{-1/11}$. Next, applying
Assumption \ref{assump:G}(c), we have also uniformly over $i \in [N]$,
\begin{align*}
N^{-1}\vg_i^\top \vR^{(i)}\vg_i
&=N^{-1}\Tr \vSigma \vR^{(i)}+\SD{N^{-1}\|\vR^{(i)}\|_F}\\
&=N^{-1}\Tr \vSigma \vR+\SD{N^{-2}|1+N^{-1}\vg_i^\top \vR^{(i)}\vg_i|^{-1}
\|\vR^{(i)}\|_F^2}+\SD{N^{-1}\|\vR^{(i)}\|_F}
\end{align*}
where the second line follows from \eqref{eq:mdiffestimate} applied with
$\vB=\vSigma$. Applying $\|\vR^{(i)}\|_F \prec N^{1/2}$ by Corollary
\ref{coro:Fnormbound} and the estimate $|1+N^{-1}\vg_i^\top \vR^{(i)}\vg_i|^{-1}
\prec \eta^{-1}$ from \eqref{eq:prelimbounds12}, this gives
\begin{equation}\label{eq:1plusgRgbound}
1+N^{-1}\vg_i^\top \vR^{(i)}\vg_i=1+N^{-1}\Tr \vSigma\vR+\SD{N^{-1/2}}.
\end{equation}
Then, applying this and 
$|1+N^{-1}\vg_i^\top \vR^{(i)}\vg_i|^{-1} \prec \eta^{-1}$ 
to \eqref{eq:mnidentity},
\[m_{\tilde \vK}=-\frac{1}{z} \cdot \frac{1}{1+N^{-1}\Tr \vSigma\vR}
+\SD{N^{-1/2}\eta^{-2}}.\]
Together with the first bound of \eqref{eq:tildembounds} and the bound $|z| \leq
\eps^{-1}$ for $z \in U_N(\eps)$, this implies
for a constant $c_0>0$ that $\bI\{|1+N^{-1}\Tr \vSigma\vR|<c_0\} \prec 0$,
and thus $\bI\{|1+N^{-1}\vg_i^\top \vR^{(i)}\vg_i|<c_0\} \prec 0$.

Applying Corollary \ref{coro:Fnormbound} and the above arguments now
for $\vK^{(S)}$ and $\vR^{(S)}$ in place of $\vK$
and $\vR$, we obtain for any fixed $L \geq 1$ and
some constants $C_0,c_0>0$, uniformly over $S \subset [N]$ with $|S| \leq L$,
over $i \in S$, and over $z \in U_N(\eps)$ with $\Im z \geq N^{-1/11}$,
\begin{equation}\label{eq:checkFAconditions}
\begin{gathered}
\bI\{|\tilde m_{\vK}^{(S)}|>C_0\} \prec 0,
\quad \bI\{|\tilde m_{\vK}^{(S)}|<c_0\} \prec 0,
\quad \bI\{\|(\vI+\tilde m_{\vK}^{(S)}\vSigma)^{-1}\|>C_0\} \prec 0,\\
\bI\{\|\vR^{(S)}\|_F>C\sqrt{N}\} \prec 0,
\quad \bI\{|1+N^{-1}\Tr \vSigma\vR^{(S)}|<c_0\} \prec 0,\\
\bI\{|1+N^{-1}\vg_i^\top \vR^{(S)}\vg_i|<c_0\} \prec 0.
\end{gathered}
\end{equation}
(We remark that a direct application of the above arguments for $\vK^{(S)}$
yields the first three estimates of \eqref{eq:checkFAconditions} for the
quantity $\frac{N}{N-|S|} \tilde m_{\vK}^{(S)}=\frac{1}{N-|S|}\Tr
\vR^{(S)}+\frac{n}{N-|S|}(-1/z)$ in place of $\tilde m_{\vK}^{(S)}$,
and the estimates for $\tilde m_{\vK}^{(S)}$ then follow for slightly modified
constants $C_0,c_0>0$ because $|S| \leq L$.)

Finally, applying \eqref{eq:1plusgRgbound} and \eqref{eq:checkFAconditions}
back to \eqref{eq:Deltaidentity} and \eqref{eq:di1}--\eqref{eq:di4} with
$\vA=\vI$, we get $|d_{i,1}|,|d_{i,3}|,|d_{i,4}| \prec N^{-1}$,
$|d_{i,2}| \prec N^{-1/2}$, and
\begin{align*}
|z_N(m_{\tilde \vK})-z| &\prec \left|\frac{1}{N}\sum_{i=1}^N \frac{d_{i,2}}
{1+N^{-1}\vg_i^\top \vR^{(i)}\vg_i}\right|+\SD{N^{-1}}\\
&=\frac{1}{N} \cdot \frac{1}{1+N^{-1}\Tr \vSigma\vR}
\cdot \left|\sum_{i=1}^N d_{i,2}\right|+\SD{N^{-1}}.
\end{align*}
The statements of \eqref{eq:checkFAconditions} verify the needed assumptions
of Lemma \ref{lemma:fluctuationavg} with $\vA=\vI$, $\vGamma=z\vI$, and
$\Phi_N=\Psi_N=C\sqrt{N}$. Then Lemma \ref{lemma:fluctuationavg}
gives $\sum_{i=1}^N d_{i,2} \prec 1$, and hence
\[|z_N(m_{\tilde \vK})-z| \prec N^{-1}.\]
The proof is then completed by the same stability argument as in the conclusion
of the proof of Lemma \ref{lemma:prelimestimate}.
\end{proof}

\begin{proof-of-theorem}[\ref{thm:no_outlier}]
We apply the idea of \cite[Section 6]{bai1998no}. Let $z=x+i\eta$, where
$\dist{x,\cS_N} \geq \eps$ and $\eta=N^{-1/11}$.
Taking imaginary part in the estimate $m_{\tilde \vK}(z)-\tilde m_N(z) \prec
N^{-1}$ of Lemma \ref{lemma:improvedestimate} and multiplying by
$\eta$ gives
\[\frac{1}{N}\sum_{j=1}^N \frac{\eta^2}{(\lambda_j(\widetilde \vK)-x)^2+\eta^2}
-\int \frac{\eta^2}{(\lambda-x)^2+\eta^2} d\tilde\mu_N(\lambda) \prec
\frac{\eta}{N}.\]
Fix any integer $P \geq 1$, and apply this instead
at the point $z=x+i\sqrt{p}\,\eta$ for each $p=1,\ldots,P$. Then
\[\frac{1}{N}\sum_{j=1}^N \frac{\eta^2}{(\lambda_j(\widetilde \vK)-x)^2+p\eta^2}
-\int \frac{\eta^2}{(\lambda-x)^2+p\eta^2} d\tilde\mu_N(\lambda) \prec
\frac{\eta}{N} \text{ for all } p=1,\ldots,P.\]
Taking successive finite differences using
\[\frac{1}{r-q+1}\left(\frac{1}{\prod_{p=q}^r (\lambda-x)^2+p\eta^2}
-\frac{1}{\prod_{p=q+1}^{r+1} (\lambda-x)^2+p\eta^2}\right)
=\frac{\eta^2}{\prod_{p=q}^{r+1} (\lambda-x)^2+p\eta^2},\]
we then obtain
\begin{equation}\label{eq:finitediffbound}
\frac{1}{N}\sum_{j=1}^N \frac{\eta^{2P}}{\prod_{p=1}^P
[(\lambda_j(\widetilde \vK)-x)^2+p\eta^2]}
-\int \frac{\eta^{2P}}{\prod_{p=1}^P[(\lambda-x)^2+p\eta^2]}
d\tilde\mu_N(\lambda) \prec \frac{\eta}{N}.
\end{equation}
Since $\dist{x,\cS_N} \geq \eps$, the second integral term of
\eqref{eq:finitediffbound} is bounded by $C\eta^{2P}$ for a
constant $C:=C(\eps,P)>0$. Thus, we get
\[\frac{1}{N}\sum_{j=1}^N \bI\{\lambda_j(\widetilde \vK) \in (x-\eta,x+\eta)\}
\leq \frac{C}{N}\sum_{j=1}^N \frac{\eta^{2P}}{\prod_{p=1}^P
[(\lambda_j(\widetilde \vK)-x)^2+p\eta^2]} \prec \frac{\eta}{N}+\eta^{2P}\]
where the first inequality holds
for a constant $C:=C(P)>0$. Finally, recalling $\eta=N^{-1/11}$ and taking
any $P \geq 6$, we get $\eta/N+\eta^{2P} \ll 1/N$, hence
\[\bI\big\{\text{there exists an eigenvalue of }\widetilde \vK
\text{ in } (x-\eta,x+\eta)\big\} \prec 0.\]
Recalling Assumption \ref{assump:G}(b) and
taking a union bound over $x$ belonging to a $\eta$-net of
$[-B,B] \setminus (\cS_N+(-\eps,\eps))$ (with cardinality at most $CN^{1/11}$),
we obtain
\[\bI\big\{\text{there exists an eigenvalue of }\widetilde \vK
\text{ in } \cS_N+(-\eps,\eps)\big\} \prec 0.\]
The theorem follows from the observation that $\vK$ has the same non-zero
eigenvalues as $\widetilde \vK$, and all 0 eigenvalues belong by definition to
$\cS_N$.
\end{proof-of-theorem}

\subsection{Deterministic equivalent for the
resolvent}\label{sec:deterministicequiv}

In this section, we prove Theorem \ref{thm:deterministic_equ}. 

\begin{lemma}\label{lemma:supportcomparea}
Suppose Assumption \ref{assump:G} holds. Let
\[\gamma_N^{(S)}=\frac{n}{N-|S|}, \qquad
\mu_N^{(S)}=\rho_{\gamma_N^{(S)}}^\MP \boxtimes \nu_N, \qquad
\tilde \mu_N^{(S)}=\gamma_N^{(S)} \mu_N^{(S)}+(1-\gamma_N^{(S)})\delta_0\]
be the analogues of $\gamma_N,\mu_N,\tilde \mu_N$ defined with the dimension
$N-|S|$ in place of $N$.
Then for any fixed $\eps>0$ and $L \geq 1$, all large $N$, and all
$S \subset [N]$ with $|S| \leq L$,
\[\supp{\tilde \mu_N^{(S)}} \subseteq \supp{\tilde \mu_N}+(-\eps,\eps)\]
\end{lemma}
\begin{proof}
Let $\cT_N$ and $z_N:\C \setminus \cT_N \to \C$ be as defined by
\eqref{eq:zNdomain} and \eqref{eq:z_N_def}. Define similarly
\[z_N^{(S)}(\tilde m)=-\frac{1}{\tilde m}+\gamma_N^{(S)}
\int \frac{\lambda}{1+\lambda \tilde m}d\nu_N(\lambda), \qquad
z_N^{(S)}:\C \setminus \cT_N \to \C.\]
We recall from Proposition \ref{prop:inv_z} that
$x \in \R \setminus \supp{\tilde\mu_N}$ if and only if there exists
$\tilde m \in \R \setminus \cT_N$ where $z_N(\tilde m)=x$ and
$z_N'(\tilde m)>0$; the analogous characterization holds for
$\R \setminus \supp{\tilde \mu_N^{(S)}}$ and $z_N^{(S)}(\tilde m)$.

Now fix any $\eps,L>0$. By Proposition \ref{prop:basic_mn}, there is a constant
$C_0>0$ such that $\supp{\tilde \mu_N^{(S)}} \subseteq [-C_0,C_0]$ for all $|S|
\leq L$ and all large $N$. Consider any $x \in [-C_0,C_0] \setminus
(\supp{\tilde \mu_N}+(-\eps,\eps))$. Then $[x-\eps/2,x+\eps/2] \subset \R
\setminus \supp{\tilde \mu_N}$, so $\tilde m_N$ is well-defined and increasing on
$[x-\eps/2,x+\eps/2]$.
Define $[\tilde m_-,\tilde m_+]=[\tilde m_N(x-\eps/2), \tilde m_N(x+\eps/2)]$.
Then Proposition \ref{prop:inv_z} implies that $z_N$ is increasing on
$[\tilde m_-,\tilde m_+]$, and $z_N([\tilde m_-,\tilde m_+])=[x-\eps/2,x+\eps/2]$.
Again by Proposition \ref{prop:basic_mn}, there is a constant $c>0$ such that,
for any such $x \in [-C_0,C_0] \setminus (\supp{\tilde \mu_N}+(-\eps,\eps))$, we
have
\[\min_{y \in [x-\eps/2,x+\eps/2]} \min_{\lambda \in
\supp{\nu_N}} |1+\lambda \tilde m_N(y)|>c.\]
This then implies that there is a constant $C>0$ for which
\[|z_N^{(S)}(\tilde m)-z_N(\tilde m)|
=|\gamma_N^{(S)}-\gamma_N| \cdot \left|\int \frac{\lambda}{1+\lambda \tilde
m}d\nu_N(\lambda)\right| \leq \frac{C}{N}<\eps/2\]
for all $\tilde m \in [\tilde m_-,\tilde m_+]$,
$|S| \leq L$, and large $N$. Then $z_N^{(S)}(\tilde m_-)<z_N(\tilde m_-)+\eps/2=x$
and $z_N^{(S)}(\tilde m_+)>z_N(\tilde m_+)-\eps/2=x$.
\cite[Theorem 4.3]{silverstein1995analysis} shows that if
$m_1,m_2 \in [\tilde m_-,\tilde m_+]$ satisfy ${z_N^{(S)}}'(m_1) \geq 0$ and
${z_N^{(S)}}'(m_2) \geq 0$, then ${z_N^{(S)}}'(m)>0$ strictly for all $m \in
[m_1,m_2]$. By this and the continuity and differentiability of
$z_N^{(S)}$ on $[\tilde m_-,\tilde m_+]$, there must be a point $\tilde m \in (\tilde
m_-,\tilde m_+)$ where $z_N^{(S)}(\tilde m)=x$ and ${z_N^{(S)}}'(\tilde m)>0$
strictly. Then Proposition \ref{prop:inv_z} implies that
$x \notin \supp{\tilde \mu_N^{(S)}}$. This holds for all
$x \in [-C_0,C_0] \setminus (\supp{\tilde \mu_N}+(-\eps,\eps))$, implying
$\supp{\tilde \mu_N^{(S)}} \subseteq \supp{\tilde \mu_N}+(-\eps,\eps)$ as desired.
\end{proof}

The following now applies Lemma \ref{lemma:supportcomparea} and
Theorem \ref{thm:no_outlier} to extend the
estimates (\ref{eq:checkFAconditions}) previously obtained over
$\{z \in U_N(\eps):\Im z \geq N^{-1/11}\}$ to all of $U_N(\eps)$.
\begin{lemma}\label{lemma:FAconditions}
Fix any $\eps>0$ and $L \geq 1$. Then for some constants $C_0,c_0>0$, uniformly
over $z \in U_N(\eps)$, $S \subset [N]$ with $|S| \leq L$, and $i \in S$,
we have
\[\begin{gathered}
\bI\{|\tilde m_{\vK}^{(S)}(z)|>C_0\} \prec 0,
\quad \bI\{|\tilde m_{\vK}^{(S)}(z)|<c_0\} \prec 0,
\quad \bI\{\|(\vI+\tilde m_{\vK}^{(S)}(z)\vSigma)^{-1}\|>C_0\} \prec 0,\\
\bI\{\|\vR^{(S)}(z)\|>C_0\} \prec 0,
\quad \bI\{|1+N^{-1}\Tr \vSigma\vR^{(S)}(z)|<c_0\} \prec 0,\\
\bI\{|1+N^{-1}\vg_i^\top \vR^{(S)}(z)\vg_i|<c_0\} \prec 0.
\end{gathered}\]
\end{lemma}
\begin{proof}
By conjugation symmetry, it suffices to show the statements for
$z \in U_N(\eps)$ with $\Im z \geq 0$. Denote for simplicity
$\vR^{(S)}=\vR^{(S)}(z)$ and $\tilde m_{\vK}^{(S)}=\tilde m_{\vK}^{(S)}(z)$.
Let $\cS_N^{(S)}=\supp{\mu_N^{(S)}} \cup \{0\}
=\supp{\tilde \mu_N^{(S)}} \cup \{0\}$ where $\mu_N^{(S)},\tilde \mu_N^{(S)}$
are as defined in Lemma \ref{lemma:supportcomparea}.
Then Theorem~\ref{thm:no_outlier} applied to $\vK^{(S)}$ guarantees that
\[\bI\{\vK^{(S)} \text{ has an eigenvalue outside } \cS_N^{(S)}+(-\eps/4,\eps/4)\} \prec 0,\]
uniformly over all $S \subset [N]$ with $|S| \leq L$. Note that
$\cS_N^{(S)}+(-\eps/4,\eps/4) \subseteq \cS_N+(-\eps/2,\eps/2)$ by
Lemma~\ref{lemma:supportcomparea}. Then, applying the bound $\|\vR^{(S)}\|
\leq 1/\dist{z,\cS_N^{(S)}}$ and the condition $z \in U_N(\eps)$, we get
\begin{equation}\label{eq:operatornormbound}
\bI\{\|\vR^{(S)}\|>2/\eps\} \prec 0.
\end{equation}

The remaining statements have already been shown for $z \in U_N(\eps)$ with
$\Im z \geq N^{-1/11}$ in \eqref{eq:checkFAconditions}. 
For $z=x+i\eta$ where $\eta \in [0,N^{-1/11}]$, define $z'=x+i N^{-1/11}$.
On the event that $\vK^{(S)}$ has no eigenvalues
outside $\cS_N+(-\eps/2,\eps/2)$, both $N^{-1}\Tr \vSigma \vR^{(S)}(z)$ and
$\tilde m_{\vK}^{(S)}(z)=N^{-1}\Tr \vR^{(S)}(z)+\gamma_N(-1/z)$ are
$C$-Lipschitz over $z \in U_N(\eps)$ for a constant $C=C(\eps)>0$,
and $N^{-1}\vg_i^\top \vR^{(S)}(z)\vg_i$ is $CN^{-1}\|\vg_i\|^2$-Lipschitz where
$N^{-1}\|\vg_i\|^2 \prec 1$ by Assumption \ref{assump:G}. Then
\[N^{-1}\Tr \vSigma \vR^{(S)}(z)-N^{-1}\Tr \vSigma \vR^{(S)}(z')
\prec N^{-1/11}, \quad
\tilde m_{\vK}^{(S)}(z)-\tilde m_{\vK}^{(S)}(z')
\prec N^{-1/11},\]
\[N^{-1}\vg_i^\top \vR^{(S)}(z)\vg_i-N^{-1}\vg_i^\top \vR^{(S)}(z')\vg_i
\prec N^{-1/11},\]
so the remaining statements of the lemma hold also for $z \in U_N(\eps)$
with $\Im z \in [0,N^{-1/11}]$.
\end{proof}

\begin{proof-of-theorem}[\ref{thm:deterministic_equ}]
Again by conjugation symmetry, it suffices to show the result for
$z \in U_N(\eps)$ with $\Im z \geq 0$. Denote for simplicity
$\vR^{(S)}=\vR^{(S)}(z)$ and $\tilde m_{\vK}^{(S)}=\tilde m_{\vK}^{(S)}(z)$.
The first estimate of Lemma \ref{lemma:FAconditions} implies
\begin{equation}\label{eq:FAconditions1}
\bI\{\|\vR^{(S)}\|_F>C\sqrt{N}\} \prec 0,
\qquad \bI\{\|\vR^{(S)}\vA\|_F>C\|\vA\|_F\} \prec 0
\end{equation}
uniformly over $z \in U_N(\eps)$ and $\vA \in \C^{n \times n}$. Then also, by
Assumption \ref{assump:G}(c) and \eqref{eq:mdiffestimate} applied with
$\vB=\vSigma$,
\begin{align}
1+N^{-1}\vg_i^\top \vR^{(i)}\vg_i
&=1+N^{-1}\Tr \vSigma \vR+\SD{N^{-2}|1+N^{-1}\vg_i^\top
\vR^{(i)}\vg_i|^{-1}\|\vR^{(i)}\|_F^2+\|\vR^{(i)}\|_F}\nonumber\\
&=1+N^{-1}\Tr \vSigma \vR+\SD{N^{-1/2}}.\label{eq:1plusgRgfull}
\end{align}

Let $d_i=d_{i,1}+d_{i,2}+d_{i,3}+d_{i,4}$ be as defined in \eqref{eq:di14} with
$\vA=\vI$. Then, applying \eqref{eq:FAconditions1}, \eqref{eq:1plusgRgfull},
and the bounds of Lemma \ref{lemma:FAconditions}, we obtain exactly as in the
proof of Lemma \ref{lemma:improvedestimate} (using again the fluctuation
averaging result of Lemma \ref{lemma:fluctuationavg}) that,
uniformly over $z \in
U_N(\eps)$, we have $|d_{i,1}|,|d_{i,3}|,|d_{i,4}| \prec N^{-1}$,
$|d_{i,2}| \prec N^{-1/2}$, and
\[|z_N(m_{\tilde \vK})-z| \prec
\frac{1}{N} \cdot \frac{1}{1+N^{-1}\Tr \vSigma\vR}
\cdot \left|\sum_{i=1}^N d_{i,2}\right|+\SD{N^{-1}}=\SD{N^{-1}}.\]
Fix any $\iota>0$. If $\Im z \geq N^{-1+\iota}$, then this
implies $\bI\{z_N(m_{\tilde \vK}) \notin \C^+\} \prec 0$. By the same
stability argument as in Lemma \ref{lemma:prelimestimate}, we get
$m_{\tilde \vK}(z)-\tilde m_N(z) \prec N^{-1}$
uniformly over $z \in U_N(\eps)$ with $\Im z \geq N^{-1+\iota}$. For
$\Im z \in [0,N^{-1+\iota}]$, on the event that all eigenvalues of $\vK$ belong
to $\cS_N+(-\eps/2,\eps/2)$, we may apply that both $m_{\tilde \vK}(z)$
and $\tilde m_N(z)$ are $C(\eps)$-Lipschitz over $z \in U_N(\eps)$ to
compare values at $z=x+i\eta$ and $z'=x+iN^{-1+\iota}$.
Applying $m_{\tilde \vK}(z')-\tilde m_N(z') \prec N^{-1}$, we then get for any
$D>0$, all $z \in U_N(\eps)$, some constant $C>0$, and all large $N$,
\[\P[|m_{\tilde \vK}(z)-\tilde m_N(z)|>CN^{-1+\iota}] \leq N^{-D}.\]
Since $\iota>0$ is arbitrary, this shows
$m_{\tilde \vK}(z)-\tilde m_N(z) \prec N^{-1}$ uniformly over $z \in U_N(\eps)$.
The bound $m_{\vK}(z)-m_N(z) \prec N^{-1}$ then follows from
$m_{\tilde \vK}(z)=\gamma_N m_{\vK}(z)+(1-\gamma_N)(-1/z)$ and
$\tilde m_N(z)=\gamma_N m_N(z)+(1-\gamma_N)(-1/z)$.

For the estimate of $\Tr \vR\vA$, we apply the definition of
$d_i=d_{i,1}+d_{i,2}+d_{i,3}+d_{i,4}$ from \eqref{eq:di14}
and the identity \eqref{eq:detequividentity} now with this matrix $\vA$. Then
\eqref{eq:detequividentity} gives
\[\Tr \Big[\vR\vA-({-}z\vI-zm_{\tilde \vK}\vSigma)^{-1}\vA\Big]
=\frac{1}{z}\sum_{i=1}^N \frac{d_i}{1+N^{-1}\vg_i^\top \vR^{(i)}\vg_i}.\]
Applying \eqref{eq:FAconditions1}, \eqref{eq:1plusgRgfull},
and the bounds of Lemma \ref{lemma:FAconditions} to
\eqref{eq:di1}--\eqref{eq:di4}, uniformly over $z \in U_N(\eps)$ and $\vA \in
\C^{n \times n}$, we have $|d_{i,1}|,|d_{i,3}|,|d_{i,4}|
\prec N^{-3/2}\|\vA\|_F$, $|d_{i,2}| \prec N^{-1}\|\vA\|_F$, and hence
\[\left|\sum_{i=1}^N \frac{d_i}{1+N^{-1}\vg_i^\top \vR^{(i)}\vg_i}\right|
\prec \frac{1}{1+N^{-1}\Tr \vSigma\vR}\left|\sum_{i=1}^N d_{i,2}\right|
+\SD{N^{-1/2}\|\vA\|_F}.\]
Finally, applying Lemma \ref{lemma:fluctuationavg} with $\vGamma=z\vI$,
$\Psi_N(\vGamma)=C\sqrt{N}$, and $\Phi_N(\vGamma,\vA)=C\|\vA\|_F$ (where we may
assume without loss of generality $\|\vA\|_F \in (N^{-\upsilon},N^\upsilon)$
by scale invariance of the desired estimate with respect to $\vA$),
we get $|\sum_i d_{i,2}| \prec N^{-1/2}\|\vA\|_F$. Thus,
\[\Tr \Big[\vR\vA-({-}z\vI-zm_{\tilde \vK}\vSigma)^{-1}\vA\Big]
\prec \frac{1}{\sqrt{N}}\,\|\vA\|_F.\]
\end{proof-of-theorem}


\section{Analysis of spiked eigenstructure}\label{appendix:spike}

We now consider the asymptotic setup of Appendix \ref{subsec:spikes} and
prove Corollary \ref{cor:no_outlier} and Theorem \ref{thm:spiked}.
As all the desired
statements are invariant under conjugation of $\vSigma$ by an orthogonal matrix,
we may assume without loss of generality that $\vSigma$ is diagonal and of the
form
\[\vSigma=\begin{pmatrix}
	\vSigma_r & \mathbf{0}\\
	\mathbf{0} & \vSigma_0
\end{pmatrix}, \qquad
\vSigma_r=\diag(\lambda_1(\vSigma),\ldots,\lambda_r(\vSigma)),
\qquad \vSigma_0=\diag(\lambda_{r+1}(\vSigma),\ldots,\lambda_n(\vSigma)).\]
Denote the block decomposition of $\vG$ corresponding to
$\vSigma_r,\vSigma_0$ as
\[\vG=[\vG_r,\vG_0], \qquad \vG_r \in \R^{N \times r},\qquad
\vG_0 \in \R^{N \times (n-r)}.\]
We remind the reader that $\vG_r$ and $\vG_0$ need not be independent.

\subsection{No outliers outside the limit support}

We consider first the setting of $r=0$, and prove Corollary~\ref{cor:no_outlier}
together with some uniform convergence properties of
$\tilde m_N$ and $z_N$ that will be used in the later analysis.

Recall the domain $\cT_N$ and function $z_N:\C \setminus \cT_N \to \C$ from
\eqref{eq:zNdomain} and \eqref{eq:z_N_def}, and their asymptotic
analogues $\cT$ and $z:\C
\setminus \cT \to \C$ from \eqref{eq:zdomain} and \eqref{eq:inv_z}.

\begin{lemma}\label{lemma:supportcompareb}
Suppose Assumption \ref{assump:G} holds, and Assumption \ref{assum:spike} holds
with $r=0$. Then, as $N \to \infty$,
\begin{enumerate}[label=(\alph*)]
\item $z_N(\tilde m)$ and its derivative $z_N'(\tilde m)$ converge uniformly
over compact subsets of $\C \setminus \cT$ to $z(\tilde m)$ and $z'(\tilde m)$.
\item For any $\eps>0$ and all large $N$,
\[\supp{\tilde \mu_N} \subseteq \supp{\tilde \mu}+(-\eps,\eps).\]
\item $\tilde m_N(z)$ and its derivative $\tilde m_N'(z)$ converge uniformly
over compact subsets of $\C \setminus \supp{\tilde \mu}$ to $\tilde m(z)$ and
$\tilde m'(z)$.
\end{enumerate}
\end{lemma}
\begin{proof}
For part (a), let $K \subset \C \setminus \cT$ be any fixed compact set.
Then $K$ does not intersect some sufficiently small open neighborhood of the compact
domain $\cT$. If Assumption \ref{assum:spike} holds with $r=0$, then $\cT_N$ is
contained in this open neighborhood of $\cT$ for all large $N$, so $K
\subset \C \setminus \cT_N$, and both $z_N$ and
$z$ are well-defined on $K$. The pointwise convergences $z_N(\tilde m) \to z(\tilde
m)$ and $z_N'(\tilde m) \to z'(\tilde m)$ on $K$ then follow from $\gamma_N \to
\gamma$, the weak convergence $\nu_N \to \nu$, and the uniform boundedness of the
functions $\lambda \mapsto \lambda/(1+\lambda \tilde m)$
and $\lambda \mapsto \lambda^2/(1+\lambda \tilde m)^2$ on an open
neighborhood of $\supp{\nu}$, for $\tilde m \in K$. This convergence is
furthermore uniform because $\{z_N\}$ and $\{z_N'\}$ are both equicontinuous
over $K$.

For part (b), consider any $x \notin \supp{\tilde \mu}+(-\eps,\eps)$.
Then $[x-\eps/2,x+\eps/2] \subset \R \setminus \supp{\tilde \mu}$, so $\tilde m$ is
well-defined and increasing on $[x-\eps/2,x+\eps/2]$. Let $[\tilde m_-,\tilde m_+]
=[\tilde m(x-\eps/2),\tilde m(x+\eps/2)]$. Then by Proposition~\ref{prop:inv_z},
$z'(\tilde m)>0$ for all $\tilde m \in [\tilde m_-,\tilde m_+]$, and
$z([\tilde m_-,\tilde m_+])=[x-\eps/2,x+\eps/2]$. The uniform convergence in
part (a) implies for all large $N$ that $z_N(\tilde m_-)<x$,
$z_N(\tilde m_+)>x$, and $z_N'(\tilde m)>0$ for all $\tilde m \in [\tilde
m_-,\tilde m_+]$. Then there exists $\tilde m \in [\tilde m_-,\tilde m_+]$
where $z_N(\tilde m)=x$ and $z_N'(\tilde m)>0$, implying by
Proposition \ref{prop:inv_z} that $x \notin \supp{\tilde \mu_N}$. So
$\supp{\tilde \mu_N} \subseteq \supp{\tilde \mu}+(-\eps,\eps)$ as desired.

For part (c), let $K \subset \C \setminus \supp{\tilde \mu}$ be any
fixed compact set. Then $K$ does not intersect some sufficiently small open
neighborhood of the compact set $\supp{\tilde \mu}$, so the inclusion of part
(b) implies $K \subset \C \setminus \supp{\tilde \mu_N}$ for all large $N$, and
both $\tilde m_N$ and $\tilde m$ are well-defined on $K$. The
uniform convergence $\tilde m_N(z) \to \tilde m(z)$ and $\tilde m_N'(z) \to
\tilde m'(z)$ on $K$ then follow from the weak convergence $\tilde \mu_N \to
\tilde \mu$, the uniform boundedness of the functions $\lambda \mapsto
1/(\lambda-z)$ and $\lambda \mapsto 1/(\lambda-z)^2$ on an open neighborhood of
$\supp{\tilde \mu}$ for $z \in K$,
and the equicontinuity of $\{\tilde m_N\}$ and $\{\tilde m_N'\}$ on $K$.
\end{proof}

\begin{proof-of-coro}[\ref{cor:no_outlier}]
By Lemma \ref{lemma:supportcompareb}(b), for any fixed $\eps>0$, we have
$\cS_N+(-\eps/2,\eps/2) \subseteq \cS+(-\eps,\eps)$
for all large $N$. Then by Theorem \ref{thm:no_outlier},
\begin{align*}
&\bI\{\vK \text{ has an eigenvalue in } \R \setminus (\cS+(-\eps,\eps))\\
&\leq \bI\{\vK \text{ has an eigenvalue in } \R \setminus
(\cS_N+(-\eps/2,\eps/2))\} \prec 0.
\end{align*}
\end{proof-of-coro}

\subsection{Deterministic equivalents for generalized resolvents}

We next introduce two generalized resolvents for the matrix $\vK$, and extend
Theorem \ref{thm:deterministic_equ} to establish
deterministic equivalents for these generalized resolvents.

Define the spectral domain
\[U(\eps)=\Big\{z \in \C:|z| \leq \eps^{-1},\;\dist{z,\cS} \geq \eps\Big\}\]
where $\cS$ is the limit support set defined in (\ref{eq:limitsupport}).
Given $z \in U(\eps)$ and $\alpha \in \C$, define a diagonal matrix
\begin{equation}\label{eq:Gamma}
\vGamma:=\vGamma(z,\alpha)=z\vI_n+\alpha\vV_r\vV_r^\top
=\begin{pmatrix} (z+\alpha)\vI_r & \mathbf{0} \\
\mathbf{0} & z \vI_{n-r} \end{pmatrix} \in \C^{n \times n},
\quad \vV_r=\begin{pmatrix} \vI_r \\ \mathbf{0} \end{pmatrix}
\in \R^{n \times r}.
\end{equation}
Define the first generalized resolvent
\begin{equation}\label{eq:matrix_block}
\vcR(z,\alpha)=
\begin{pmatrix}
	-\vGamma & \vG^\top\\ \vG & -\vI_N
\end{pmatrix}^{-1} \in \C^{(n+N)\times (n+N)}.
\end{equation}
This matrix inverse exists if and only if
the Schur complement $\vG^\top \vG-\vGamma=\vK-\vGamma$
for its lower right block is invertible, in which case the upper-left block of
$\vcR(z,\alpha)$ is $\vR(\vGamma)=(\vK-\vGamma)^{-1}$. The following
provides a deterministic equivalent for this block of $\vcR(z,\alpha)$.

\begin{lemma}\label{lemm:apply_fluctuation}
Under the assumptions of Theorem \ref{thm:spiked}, for any fixed $\eps>0$,
there exist $C_0,\alpha_0>0$ (depending on $\eps$) such that fixing any $\alpha
\in \C$ with $|\alpha|>\alpha_0$, the following hold:
\begin{enumerate}[label=(\alph*)]
\item The event
\[\cE=\Big\{\vcR(z,\alpha) \text{ exists and } \|\vcR(z,\alpha)\| \leq C_0
\text{ for all } z \in U(\eps)\Big\}\]
satisfies $\bI\{\cE^c\} \prec 0$.
\item Uniformly over $z \in U(\eps)$
and deterministic unit vectors $\vv_1,\vv_2 \in \R^n$,
\begin{align}
\norm{   \begin{pmatrix}
    \vv_1^\top&\mathbf{0}
    \end{pmatrix}\vcR(z,\alpha)\begin{pmatrix}
    \vv_2\\
        \mathbf{0}
    \end{pmatrix}+  \vv_1^\top \left(\vGamma+z\cdot\tilde{m}_{N,0}(z)\vSigma
\right)^{-1}\vv_2 } \prec \frac{1}{\sqrt{N}}.
\end{align} 
\end{enumerate}
\end{lemma}

In the setting of Theorem \ref{thm:spiked}(c), let
$\vu=\frac{1}{\sqrt{N}}(u_1,\ldots,u_N) \in \R^N$ be the
additional given vector for which $\{(u_j,\vg_j^\top)\}_{j=1}^N$ are
independent vectors in $\R^{n+1}$. For $z \in U(\eps)$ and $\alpha \in \C$, define
\begin{align}
\widetilde{\vSigma}&=\begin{pmatrix}
		\E[u^2] & \E[u\vg]^\top\\
		\E[u\vg] & \vSigma
	\end{pmatrix}	\in\R^{(n+1)\times (n+1)}\label{eq:tilde_Sigma},\\
\widetilde \vGamma&=\widetilde \vGamma(z,\alpha)
=\begin{pmatrix} z+\alpha & \mathbf{0} \\
\mathbf{0} & \vGamma \end{pmatrix} \in \C^{(n+1) \times (n+1)}
\nonumber
\end{align}
where $\E[u^2]$ and $\E[u\vg]$ denote the common values of $\E[u_j^2]$ and
$\E[u_j\vg_j]$ for $j=1,\ldots,N$. Define the second generalized resolvent
\begin{equation}\label{eq:tildeR}
\widetilde \vcR(z,\alpha)=
\begin{pmatrix} {-}\widetilde \vGamma & [\vu,\vG]^\top \\
[\vu,\vG] & -\vI \end{pmatrix}^{-1}
=\begin{pmatrix} {-}(z+\alpha) & \mathbf{0} & \vu^\top \\
\mathbf{0} & -\vGamma & \vG^\top \\
\vu & \vG & -\vI_N \end{pmatrix}^{-1}
\in \C^{(n+1+N)\times (n+1+N)}.
\end{equation}
We have the following
deterministic equivalent for the upper-left block of $\widetilde \vcR(z,\alpha)$,
which is analogous to Lemma \ref{lemm:apply_fluctuation}.

\begin{lemma}\label{lemm:apply_fluctuation2}
Under the assumptions of Theorem \ref{thm:spiked}(c), for any fixed $\eps>0$, there
exist $C_0,\alpha_0>0$ (depending on $\eps$) such that fixing any $\alpha \in
\C$ with $|\alpha|>\alpha_0$, the following hold:
\begin{enumerate}[label=(\alph*)]
\item The event
\[\widetilde \cE=\Big\{\widetilde \vcR(z,\alpha) \text{ exists and } \|\widetilde \vcR(z,\alpha)\| \leq C_0
\text{ for all } z \in U(\eps)\Big\}\]
satisfies $\bI\{\widetilde \cE^c\} \prec 0$.
\item Uniformly over $z \in U(\eps)$ and
deterministic unit vectors $\vv_1,\vv_2\in\R^{n+1}$,
\begin{align}
		\norm{   \begin{pmatrix}
			\vv_1^\top&\mathbf{0}
		\end{pmatrix}\widetilde \vcR(z,\alpha)\begin{pmatrix}
			\vv_2\\
			\mathbf{0}
		\end{pmatrix}+
\vv_1^\top \left(\widetilde\vGamma+z\cdot\tilde{m}_{N,0}(z)\widetilde\vSigma
\right)^{-1}\vv_2 } \prec \frac{1}{\sqrt{N}}.
\end{align}
\end{enumerate}
\end{lemma}

In the remainder of this section, we prove Lemmas \ref{lemm:apply_fluctuation}
and \ref{lemm:apply_fluctuation2}. Recall 
\[\mu_{N,0}=\rho_{\gamma_{N,0}}^{\MP} \boxtimes\nu_{N,0}, \qquad
\tilde \mu_{N,0}=\gamma_{N,0}\mu_{N,0}+(1-\gamma_{N,0})\delta_0.\]
Define the bulk components of the sample covariance and Gram matrices
\begin{equation}\label{eq:K0def}
\vK_0=\vG_0^\top \vG_0 \in \R^{(n-r) \times (n-r)},
\qquad \widetilde \vK_0=\vG_0\vG_0^\top \in \R^{N \times N}.
\end{equation}
Define also the $N$-dependent bulk spectral support and spectral domain
\begin{align}
\cS_{N,0}&=\supp{\mu_{N,0}} \cup \{0\}
=\supp{\tilde\mu_{N,0}} \cup \{0\},\nonumber\\
U_{N,0}(\eps)&=\{z \in \C:|z| \leq \eps^{-1},\dist{z,\cS_{N,0}} \geq \eps\}.\label{eq:UN0}
\end{align}
Lemma \ref{lemma:supportcompareb}(b) shows
$\cS_{N,0} \subseteq \cS+(-\eps/2,\eps/2)$ for any fixed $\eps>0$ and all large $N$,
so also $U(\eps) \subseteq U_{N,0}(\eps/2)$ for all large $N$. Thus, the results of
Appendix \ref{appendix:resolvent} applied to $\vK_0$, which hold uniformly over
$z \in U_{N,0}(\eps/2)$ for any fixed $\eps>0$,
also hold uniformly over $z \in U(\eps)$. In particular,
the following is an immediate consequence of Corollary \ref{cor:no_outlier} and
Theorem \ref{thm:deterministic_equ},
which we record here for future reference.

\begin{lemma}\label{lemm:support_G0}
Suppose Assumptions~\ref{assump:G} and \ref{assum:spike} hold.
Then for any fixed $\eps>0$,
\begin{align*}
\bI\{\vK_0 \text{ has an eigenvalue outside }
\cS+(-\eps,\eps)\} \prec 0.
\end{align*}
Furthermore, uniformly over $z \in U(\eps)$,
\[m_{\vK_0}-m_{N,0}(z)\prec 1/N, \quad  m_{\tilde\vK_0}-\tilde m_{N,0}(z)
\prec 1/N.\]
\end{lemma}

We now check that for sufficiently large $|\alpha|$, the generalized resolvent
$\vcR(z,\alpha)$ exists and has bounded operator norm with high probability.\\

\begin{proof-of-lemma}[\ref{lemm:apply_fluctuation}(a)]
Let
\[\cE'=\left\{\text{ all eigenvalues of } \vK_0 \text{ belong to }
\cS+(-\eps/2,\eps/2),\text{ and }\norm{\vG}<\sqrt{B}\right\}.\]
By Assumption \ref{assump:G}(b) and Lemma \ref{lemm:support_G0},
$\bI\{{\cE'}^c\} \prec 0$, so it suffices to show $\cE' \subseteq \cE$. 
On this event $\cE'$, for any $z \in U(\eps)$, we have that each
eigenvalue of $\vK_0$ is separated by at least $\eps/2$ from $z$.
Then,
	\begin{equation}\label{eq:block_G0}
		\vcR_0(z):=\begin{pmatrix}
		-z\vI_{n-r} & \vG_0^{\top}\\
			\vG_0 & -\vI_N
		\end{pmatrix}^{-1} \in \C^{(n-r+N) \times (n-r+N)}
	\end{equation}
exists for all $z \in U(\eps)$
because the Schur complement $\vK_0-z\vI_{n-r}$ of its lower-right block
is invertible.
Furthermore, denoting $\vR_0=(\vK_0-z\vI_{n-r})^{-1}$, we have
$\|\vR_0\| \leq 2/\eps$ and $\|\vG_0\| \leq \|\vG\|<\sqrt{B}$, so
	\begin{align}
\|\vcR_0(z)\|=\left\|\begin{pmatrix} \vR_0 & \vR_0\vG_0^\top
\\ \vG_0\vR_0 &
\vG_0\vR_0\vG_0^\top-\vI_N\end{pmatrix}\right\|
	\leq C_1\label{eq:R(z)_bound}
	\end{align}
for some constant $C_1$ depending only on $\eps,B$.

Now write $\vcR(z,\alpha)$ as defined in \eqref{eq:matrix_block} in its block
decomposition with blocks of sizes $r$ and $n-r+N$. Then the Schur complement
of the upper left block of size $r \times r$ is given by
	\begin{equation}\label{eq:schur_com}
	\vS=-\begin{pmatrix}
		\mathbf{0}&
		\vG_r^\top
	\end{pmatrix}\vcR_0(z)
	\begin{pmatrix}
		\mathbf{0}\\
		\vG_r
	\end{pmatrix}	-(\alpha+z)\vI_{r}.
	\end{equation}
	Notice that \begin{align}
	\vS\vS^*=~& |\alpha+z|^2\vI_{r}+\begin{pmatrix}
		\mathbf{0}&
		\vG_r^\top
	\end{pmatrix}\vcR_0(z)
	\begin{pmatrix}
		\mathbf{0}\\
		\vG_r
	\end{pmatrix}\begin{pmatrix}
		\mathbf{0}&
	\vG_r^\top
	\end{pmatrix}\overline{\vcR_0(z)}
	\begin{pmatrix}
		\mathbf{0}\\
		\vG_r
	\end{pmatrix}\\
	~&+(\bar \alpha+\bar z)\begin{pmatrix}
		\mathbf{0}&
		\vG_r^\top
	\end{pmatrix}\vcR_0(z)
	\begin{pmatrix}
		\mathbf{0}\\
		\vG_r
	\end{pmatrix}+ (\alpha+z)\begin{pmatrix}
		\mathbf{0}&
		\vG_r^\top
	\end{pmatrix}\overline{\vcR_0(z)}
	\begin{pmatrix}
		\mathbf{0}\\
		\vG_r
	\end{pmatrix}
	\end{align}
where the first two terms are positive semi-definite.
	Therefore, applying \eqref{eq:R(z)_bound} and $\|\vG_r\| \leq
\|\vG\|<\sqrt{B}$ on the event $\cE'$,
there exist $\alpha_0,c_0>0$ depending only on $\eps,B$, such that 
	\begin{align}
	\lambda_{\min}(\vS\vS^*) \ge
|\alpha+z|^2-2(|\alpha|+|z|)\norm{\vG_r}^2\norm{\vcR_0( z)}>c_0
	\end{align}
for any $z\in U(\eps)$ and $|\alpha|>\alpha_0$. Consequently, under the event
$\cE'$, the Schur complement $\vS$ in \eqref{eq:schur_com} is
invertible with $\|\vS^{-1}\|<c_0^{-1/2}$.
Then $\vcR(z,\alpha)$ exists, and
	\begin{align}
	\|\vcR(z,\alpha)\|
	&=\left\|\begin{pmatrix}
		\vS^{-1}& -\vS^{-1}\begin{pmatrix}
			\mathbf{0}&
			\vG_r^\top
		\end{pmatrix} \vcR_0(z)\\
		-\vcR_0(z)\begin{pmatrix}
			\mathbf{0}\\
			\vG_r
		\end{pmatrix} \vS^{-1} & \vcR_0(z)+\vcR_0(z)\begin{pmatrix}
			\mathbf{0}\\
			\vG_r
		\end{pmatrix}\vS^{-1}\begin{pmatrix}
			\mathbf{0}&
			\vG_r^\top
		\end{pmatrix}\vcR_0(z)
	\end{pmatrix}\right\| \leq C_0
	\end{align}
for a constant $C_0>0$ depending only on $\eps,B$.
This shows $\cE' \subseteq \cE$ as desired.
        \end{proof-of-lemma}

For the matrix $\vGamma=\vGamma(z,\alpha)$ in \eqref{eq:Gamma}, recall the definitions of
$\vR^{(S)}(\vGamma)$ and $\tilde m_{\vK}^{(S)}(\vGamma)$ from
\eqref{eq:leaveoneout}. The following provides an analogue of
Lemma \ref{lemma:FAconditions} for these quantities.

\begin{lemma}\label{lemm:bounds_R(z)}
Fix any $\eps>0$ and $L \geq 1$. Then there exist $C_0,c_0,\alpha_0>0$ such that
for any fixed $\alpha \in \C$ with $|\alpha|>\alpha_0$,
uniformly over $S \subset [N]$ with $|S| \leq L$, over $j \in S$, and over
$z \in U(\eps)$,
\[\begin{gathered}
\bI\{|\tilde m_{\vK}^{(S)}(\vGamma)|>C_0\} \prec 0,
\quad \bI\{|\tilde m_{\vK}^{(S)}(\vGamma)|<c_0\} \prec 0,
\quad \bI\{\|(z^{-1}\vGamma+\tilde m_{\vK}^{(S)}(\vGamma)\vSigma)^{-1}\|>C_0\} \prec 0,\\
\bI\{\|\vR^{(S)}(\vGamma)\|>C_0\} \prec 0,
\quad \bI\{|1+N^{-1}\Tr \vSigma\vR^{(S)} (\vGamma)|<c_0\} \prec 0,\\
\bI\{|1+N^{-1}\vg_j^\top \vR^{(S)}(\vGamma) \vg_j|<c_0\} \prec 0.
\end{gathered}\]
\end{lemma}
\begin{proof}
Suppose $|\alpha|$ is large enough so that Lemma \ref{lemm:apply_fluctuation}(a)
holds. Since $\vR(\vGamma)$ is the upper-left block of $\vcR(z,\alpha)$,
Lemma~\ref{lemm:apply_fluctuation}(a) applied with $\vG^{(S)}$ in place of $\vG$
shows that $\bI\{\|\vR^{(S)}(\vGamma)\|>C_0\} \prec 0$ for a constant $C_0>0$,
uniformly over $S \subset [N]$ with $|S| \leq L$ and over $z \in U(\eps)$.
For the remaining statements,
let $\vG_0^{(S)} \in \R^{(N-|S|) \times (n-r)}$ be the submatrix of $\vG_0$ with
the rows of $S$ removed, and define
\[\vK_0^{(S)}={\vG_0^{(S)}}^\top\vG_0^{(S)},
\quad \widetilde \vK_0^{(S)}={\vG_0^{(S)}}{\vG_0^{(S)}}^\top,\]
\[\vR_0^{(S)}=(\vK_0^{(S)}-z\vI_{n-r})^{-1},
\quad m_{\vK_0}^{(S)}=\frac{1}{n-r}\Tr \vR_0^{(S)},
\quad \tilde
m_{\vK_0}^{(S)}=\gamma_{N,0}m_{\vK_0}^{(S)}+(1-\gamma_{N,0})\Big({-}\frac{1}{z}\Big).\]
Then by Lemma \ref{lemm:support_G0} applied to $\vK_0^{(S)}$, also
$\bI\{\|\vR_0^{(S)}\|>C_0\} \prec 0$ for a constant $C_0>0$.

Using these bounds, we first show the comparisons
\begin{equation}\label{eq:diff_m}
	|\tilde m_{\vK}^{(S)}(\vGamma)-\tilde m_{\vK_0}^{(S)}| \prec
1/N, \quad
	\left| N^{-1}\Tr \vSigma\vR^{(S)}(\vGamma)-N^{-1}
\Tr \vSigma_0\vR_0^{(S)}\right| \prec 1/N.
\end{equation}
	For the first comparison, notice that in the decompositions into blocks
of sizes $r$ and $n-r$,
\[\frac{n-r}{n}\,m_{ \vK_0}^{(S)}= \frac{1}{n}\Tr \begin{pmatrix}
		\mathbf{0} & \mathbf{0}\\
		\mathbf{0} & \vR_0^{(S)}
	\end{pmatrix}\]
	and 
	\begin{align*}
		m_{\vK}^{(S)}(\vGamma)=~&\frac{1}{n}\Tr \vR^{(S)}(\vGamma)\\
=~&\frac{1}{n}\Tr \begin{pmatrix}
		\vG_r^{(S)\top}\vG_r^{(S)}-(\alpha+z)\vI_r & \vG_r^{(S)\top}\vG_0^{(S)}\\
		\vG_0^{(S)\top}\vG_r^{(S)} & \vK_0^{(S)}-z\vI_{N-|S|}
	\end{pmatrix}^{-1}\\=
	~&\frac{1}{n}\Tr \begin{pmatrix}
		(\vS_r^{(S)})^{-1} & -(\vS_r^{(S)})^{-1}\vG_r^{(S)\top}\vG_0^{(S)}\vR_0^{(S)}\\
		-\vR_0^{(S)}\vG_0^{(S)\top}\vG_r^{(S)}(\vS_r^{(S)})^{-1} & \vR_0^{(S)}+\vR_0^{(S)}\vG_0^{(S)\top}\vG_r^{(S)}(\vS_r^{(S)})^{-1}\vG_r^{(S)\top}\vG_0^{(S)}\vR_0^{(S)}
	\end{pmatrix}, 
	\end{align*}
	where 
	\begin{equation}\label{eq:Schur_r}
		\vS_r^{(S)}:= \vG_r^{(S)\top}\vG_r^{(S)}-(\alpha+z)\vI_r-  \vG_r^{(S)\top}\vG_0^{(S)}\vR_0^{(S)}\vG_0^{(S)\top}\vG_r^{(S)}
	\end{equation}
	is the Schur complement of the lower-right block. We have
$\snorm{(\vS_r^{(S)})^{-1}}\le \snorm{\vR^{(S)}(\vGamma)}\prec 1$,
$\snorm{\vR_0^{(S)}}\prec 1$,
and by Assumption~\ref{assump:G}, $\snorm{\vG_0^{(S)}}\prec 1$ and
$\snorm{\vG_r^{(S)}}\prec 1$. Combining these bounds and using
$|\Tr \vA| \leq r\|\vA\|$ when $\vA$ has rank $r$ (as follows from the
von Neumann trace inequality),
	\begin{align*}
		&\left|m_{ \vK}^{(S)}(\vGamma)-\frac{n-r}{n}\,m_{
\vK_0}^{(S)}\right|\\
&=\left|\frac{1}{n}\Tr\begin{pmatrix}
			(\vS_r^{(S)})^{-1} & -(\vS_r^{(S)})^{-1}\vG_r^{(S)\top}\vG_0^{(S)}\vR_0^{(S)}\\
			-\vR_0^{(S)}\vG_0^{(S)\top}\vG_r^{(S)}(\vS_r^{(S)})^{-1} & \vR_0^{(S)}\vG_0^{(S)\top}\vG_r^{(S)}(\vS_r^{(S)})^{-1}\vG_r^{(S)\top}\vG_0^{(S)}\vR_0^{(S)}
		\end{pmatrix}\right|\\
		&\le \frac{1}{n}|\Tr (\vS_r^{(S)})^{-1}|+\frac{1}{n}|\Tr \vR_0^{(S)}\vG_0^{(S)\top}\vG_r^{(S)}(\vS_r^{(S)})^{-1}\vG_r^{(S)\top}\vG_0^{(S)}\vR_0^{(S)}|\\
		&\le \frac{r}{n}\snorm{(\vS_r^{(S)})^{-1}} +
\frac{r}{n}\snorm{\vG_r^{(S)}}^2
\snorm{\vG_0^{(S)}}^2\snorm{\vR_0^{(S)}}^2\snorm{(\vS_r^{(S)})^{-1}}
\prec 1/N.
	\end{align*}
Then also $|m_{\vK}^{(S)}(\vGamma)-m_{\vK_0}^{(S)}| \prec 1/N$ since
$|m_{\vK_0}^{(S)}|\leq \|\vR_0^{(S)}\| \prec 1$ and $(n-r)/n=1+\SD{1/N}$.
Hence $|\tilde m_{\vK}^{(S)}(\vGamma)-\tilde m_{\vK_0}^{(S)}| \prec 1/N$
from the definitions $\tilde m_{\vK}^{(S)}(\vGamma)=\gamma_N
m_{\vK}^{(S)}(\vGamma)+(1-\gamma_N)(-1/z)$ and
$\tilde m_{\vK_0}^{(S)}=\gamma_{N,0}m_{\vK_0}^{(S)}+(1-\gamma_{N,0})(-1/z)$,
as $|1/z| \leq \eps$ for $z \in U(\eps)$ and $\gamma_{N,0}=\gamma_N+\SD{1/N}$.
The proof of the second comparison of \eqref{eq:diff_m} is analogous,
considering in addition
	\begin{equation}\label{eq:SigmaRcompare}
\frac{1}{n}\Tr \left(\vSigma-\begin{pmatrix}
		\mathbf{0} & \mathbf{0}\\
		\mathbf{0} & \vSigma_0
	\end{pmatrix}\right)\vR^{(S)}(\vGamma)
=\frac{1}{n}\Tr \begin{pmatrix}
		\vSigma_r & \mathbf{0}\\
		\mathbf{0} & \mathbf{0}
	\end{pmatrix}\vR^{(S)}(\vGamma)
\leq \frac{r}{n}\|\vSigma_r\|
\|\vR^{(S)}(\vGamma)\| \prec \frac{1}{N}.
\end{equation}

Now, Lemma~\ref{lemma:FAconditions} applied with $\vK_0$ shows, uniformly over $S
\subset [N]$ with $|S| \leq L$ and over $z \in U(\eps)$,
\[\bI\{|\tilde m_{\vK_0}^{(S)}|>C_0\} \prec 0, \quad
	\bI\{|\tilde m_{\vK_0}^{(S)}|<c_0\} \prec 0, \quad
\bI\{|1+N^{-1}\Tr \vSigma_0\vR_0^{(S)}|<c_0\} \prec 0,\]
which together with \eqref{eq:diff_m} implies
	\[\bI\{|\tilde m_{\vK}^{(S)}(\vGamma)|>C_0\} \prec 0, \quad
	\bI\{|\tilde m_{\vK}^{(S)}(\vGamma)|<c_0\} \prec 0, \quad
\bI\{|1+N^{-1}\Tr \vSigma\vR^{(S)}(\vGamma)|<c_0\} \prec 0\]
for adjusted constants $C_0,c_0>0$.
Also by Assumption~\ref{assump:G}, uniformly over $j \in S$,
	\begin{equation}\label{eq:concnetr_g_i}
	N^{-1}\vg_j^\top \vR^{(S)}(\vGamma)\vg_j-N^{-1}\Tr
\vSigma\vR^{(S)}(\vGamma)\prec N^{-1}\|\vR^{(S)}(\vGamma)\|_F \leq
N^{-1/2}\|\vR^{(S)}(\vGamma)\| \prec N^{-1/2},
	\end{equation}
	so $\bI\{|1+N^{-1}\vg_j^\top \vR^{(S)}(\vGamma) \vg_j|<c\} \prec 0$
for a constant $c>0$.
Lastly, from the definition of $\vGamma=\vGamma(z,\alpha)$ in \eqref{eq:Gamma}, we have
	\begin{align}
	    z^{-1}\vGamma+\tilde m_{\vK}^{(S)}(\vGamma)\vSigma&= \begin{pmatrix}
		\tilde m_{\vK}^{(S)}(\vGamma)\vSigma_r+(\frac{\alpha}{z}+1)\vI_r& \mathbf{0}\\
		\mathbf{0} & \tilde m_{\vK}^{(S)}(\vGamma)\vSigma_0+\vI_{n-r}\label{eq:decomp_gamma+sigma}
	\end{pmatrix}.
	\end{align}
By \eqref{eq:diff_m} and Lemma \ref{lemma:FAconditions}, we have
	\begin{equation}\label{eq:lowerblockImSigma}
\bI\left\{\norm{\left(\tilde m_{\vK}^{(S)}(\vGamma)\vSigma_0+\vI_{n-r}\right)^{-1}}>C\right\}\prec 0
\end{equation}
	for some constant $C>0$. We have already proved $\bI\{|\tilde
m_{\vK}^{(S)}(\vGamma)|>C_0\} \prec 0$, and applying
$\|\vSigma_r\|\le C $ under Assumption~\ref{assump:G}, we can deduce for the
smallest singular value that
\begin{equation}\label{eq:sigmamin1plusmSigma}
		\sigma_{\min} \left(\tilde
m_{\vK}^{(S)}(\vGamma)\vSigma_r+(\alpha/z+1)\vI_{r}\right)\ge
\frac{|\alpha|}{|z|}-1-|\tilde m_{\vK}^{(S)}(\vGamma)|\norm{\vSigma_r} \geq
c
\end{equation}
	on the event $\{|\tilde m_{\vK}^{(S)}(\vGamma)|\le C_0\}$,
for any $z \in U(\eps)$, $|\alpha| \geq \alpha_0$, and some $\alpha_0,c>0$
depending on $\eps,C_0$.
Thus also
\begin{equation}\label{eq:extended1plusmSigma}
\bI\{\|(z^{-1}\vGamma+m_{\tilde
\vK}^{(S)}(\vGamma)\vSigma)^{-1}\|>C\} \prec 0
\end{equation}
for a constant $C>0$, showing all statements of the lemma.
\end{proof}

\begin{proof-of-lemma}[\ref{lemm:apply_fluctuation}(b)]
	Recalling the form of $\vcR(z,\alpha)$ in \eqref{eq:matrix_block},
the quantity we wish to approximate is
	\begin{align*}
		\begin{pmatrix}
			\vv_1^\top&\mathbf{0}
		\end{pmatrix}\vcR(z,\alpha)\begin{pmatrix}
			\vv_2\\
			\mathbf{0}
		\end{pmatrix}=\vv_1^\top\vR(\vGamma)\vv_2=\vv^\top_1\left(\vK-\vGamma\right)^{-1}\vv_2.
	\end{align*}
	Analogous to \eqref{eq:fundamentalidentity} in the proof of Lemma~\ref{lemma:prelimestimate}, 
	for any matrix $\vB\in \C^{n\times n}$, we have
	\begin{align}
		\Tr \vB=
\Tr (\vK-\vGamma)\vR(\vGamma)\vB={-}\Tr \vR(\vGamma)\vB\vGamma+\frac{1}{N}\sum_{i=1}^N
\frac{\vg_i^\top\vR^{(i)}(\vGamma)\vB\vg_i}{1+N^{-1}\vg_i^\top\vR^{(i)}(\vGamma)\vg_i}.\label{eq:S-M-v1Rv2}
	\end{align}
Applying the definition
$m_{\tilde \vK}(\vGamma)=N^{-1}\Tr \vR(\vGamma) +(1-\gamma_N)(-1/z)$
and the identity \eqref{eq:S-M-v1Rv2} with $\vB=\vI$, we obtain analogously to
\eqref{eq:mnidentity} that
	\begin{align*}
m_{\tilde \vK}(\vGamma)
&=(1-\gamma_N)(-1/z)+\frac{1}{Nz}\Tr \vR(\vGamma)\vGamma
-\frac{1}{N}\Tr (z^{-1}\vGamma-\vI)\vR(\vGamma)\notag\\
&={-}\frac{1}{Nz}\sum_{i=1}^N\frac{1}{1+N^{-1}\vg_i^\top\vR^{(i)}(\vGamma)\vg_i}
-\frac{1}{N}\Tr (z^{-1}\vGamma-\vI)\vR(\vGamma).
	\end{align*}
Then, noting that $z^{-1}\vGamma-\vI$ has rank $r$ and hence
$|N^{-1}\Tr(z^{-1}\vGamma-\vI)\vR(\vGamma)|
\leq \frac{r}{N}\frac{|\alpha|}{|z|}\|\vR(\vGamma)\| \prec N^{-1}$, this gives
\begin{equation}\label{eq:S-M-TrR}
m_{\tilde \vK}(\vGamma)
={-}\frac{1}{Nz}\sum_{i=1}^N \frac{1}{1+N^{-1}\vg_i^\top\vR^{(i)}(\vGamma)\vg_i}
+\SD{N^{-1}}.
\end{equation}

Fixing the unit vectors $\vv_1,\vv_2 \in \R^n$,
let us now choose $\vA=\vv_2\vv_1^\top$ and
$\vB=\vA(z^{-1}\vGamma+m_{\tilde \vK}(\vGamma)\cdot\vSigma )^{-1}$
in \eqref{eq:S-M-v1Rv2}, and define
	\begin{align*}
d_i&=\frac{1}{N}\,\vg_i^\top \vR^{(i)}(\vGamma)\vB\vg_i
	-\frac{1}{N}\Tr \vR(\vGamma)\vB\vSigma\\
&=\frac{1}{N}\,\vg_i^\top \vR^{(i)}(\vGamma)\vA
	\left(z^{-1}\vGamma+m_{\tilde \vK}(\vGamma)\vSigma\right)^{-1}\vg_i
	-\frac{1}{N}\Tr \vR(\vGamma)\vA\left(z^{-1}\vGamma+m_{\tilde
\vK}(\vGamma)\vSigma\right)^{-1}\vSigma.
\end{align*}
Then, combining \eqref{eq:S-M-v1Rv2} and \eqref{eq:S-M-TrR}, we get
	\begin{align}
&\vv_1^\top (z^{-1}\vGamma+m_{\tilde \vK}(\vGamma)\vSigma)^{-1}\vv_2=\Tr
\vB\notag\\
&= {-}\Tr \vR(\vGamma)\vB\vGamma + \Tr
\vR(\vGamma)\vB\vSigma \cdot \left(-zm_{\tilde
\vK}(\vGamma)+\SD{N^{-1}}\right)
+\sum_{i=1}^N\frac{d_i}{1+N^{-1}\vg_i^\top\vR^{(i)}(\vGamma)\vg_i} \notag\\
&= {-}z\cdot\vv_1^\top \vR(\vGamma)\vv_2 +
\sum_{i=1}^N\frac{d_i}{1+N^{-1}\vg_i^\top\vR^{(i)}(\vGamma)\vg_i}+\SD{N^{-1}},\label{eq:J1+J2}
	\end{align} 
where the last equality applies the definition of $\vB$ to combine the first two
terms, and applies also
$|\Tr \vR(\vGamma)\vB\vSigma|
\leq \|(z^{-1}\vGamma+m_{\tilde \vK}(\vGamma)\vSigma)^{-1}\vSigma\vR(\vGamma)\|
\prec 1$ by Lemma~\ref{lemm:bounds_R(z)} to obtain the
$\SD{N^{-1}}$ remainder.

	Considering a similar decomposition as in Lemma~\ref{lemma:prelimestimate}, we define
	$d_i=d_{i,1}+d_{i,2}+d_{i,3}+d_{i,4}$ where
	\begin{equation}\label{eq:di14_gamma}
	\begin{aligned}
	d_{i,1}&=\frac{1}{N}\vg_i^\top
\vR^{(i)}(\vGamma)\vA(z^{-1}\vGamma+m_{\tilde \vK}(\vGamma)\vSigma)^{-1}
	\vg_i-\frac{1}{N}\vg_i^\top
\vR^{(i)}(\vGamma)\vA(z^{-1}\vGamma+m_{\tilde \vK}^{(i)}(\vGamma)\vSigma)^{-1}\vg_i,\\
	d_{i,2}&=\frac{1}{N}\vg_i^\top
\vR^{(i)}(\vGamma)\vA(z^{-1}\vGamma+m_{\tilde \vK}^{(i)}(\vGamma)\vSigma)^{-1}
	\vg_i-\frac{1}{N}\Tr \vSigma
\vR^{(i)}(\vGamma)\vA(z^{-1}\vGamma+m_{\tilde \vK}^{(i)}(\vGamma)\vSigma)^{-1},\\
	d_{i,3}&=\frac{1}{N}\Tr \vSigma
\vR^{(i)}(\vGamma)\vA(z^{-1}\vGamma+m_{\tilde \vK}^{(i)}(\vGamma)\vSigma)^{-1}
	-\frac{1}{N}\Tr \vSigma \vR(\vGamma)\vA(z^{-1}\vGamma+m_{\tilde
\vK}^{(i)}(\vGamma)\vSigma)^{-1},\\
	d_{i,4}&=\frac{1}{N}\Tr \vSigma \vR(\vGamma)\vA(z^{-1}\vGamma+m_{\tilde
\vK}^{(i)}(\vGamma)\vSigma)^{-1}
	-\frac{1}{N}\Tr \vSigma \vR(\vGamma)\vA(z^{-1}\vGamma+m_{\tilde
\vK}(\vGamma)\vSigma)^{-1}.
	\end{aligned}
	\end{equation}
	For $\vA=\vv_2\vv_1^\top$, by the bound
$\bI\{\|\vR^{(S)}(\vGamma)\|>C_0\} \prec 0$
from Lemma \ref{lemm:bounds_R(z)}, we have for a constant $C>0$ that
	\begin{equation}
	\bI\left\{\norm{\vR^{(S)}(\vGamma)}_F>C\sqrt{N}\right\}\prec 0, \quad \bI\left\{\norm{\vR^{(S)}(\vGamma)\vA}_F>C\right\}\prec 0
	\end{equation} 
	uniformly over $z\in U(\eps)$.
	Then, employing Lemma~\ref{lemm:bounds_R(z)} and the same bounds as
\eqref{eq:di1}--\eqref{eq:di4} from the proof of Lemma \ref{lemma:FAconditions}
(where here, the bounds for $\|\vR^{(i)}\vA\|_F,\|\vR\vA\|_F$ are
improved by a factor of $N^{-1/2}$ because $\vA$ is low-rank),
we conclude that $|d_{i,1}|,|d_{i,3}|,|d_{i,4}|
	\prec N^{-3/2}$ and $|d_{i,2}| \prec N^{-1}$.  
Hence, applying also $1+N^{-1}\vg_i^\top \vR^{(i)}(\vGamma)\vg_i
=1+N^{-1}\Tr \vSigma\vR(\vGamma)+\SD{N^{-1/2}}$ as follows from
\eqref{eq:concnetr_g_i} and the bound \eqref{eq:mdiffestimate},
	\[\sum_{i=1}^N\frac{d_i}{1+N^{-1}\vg_i^\top\vR^{(i)}(\vGamma)\vg_i}
=\frac{1}{1+N^{-1}\Tr \vSigma\vR(\vGamma)}
\cdot \sum_{i=1}^N d_{i,2}+\SD{N^{-1/2}}.\]
	By Lemma~\ref{lemma:fluctuationavg} applied with
$\Psi_N(\vGamma)=C\sqrt{N}$ and $\Phi_N(\vGamma,\vA)=C$ for a constant $C>0$,
we have $|\sum_i d_{i,2}| \prec N^{-1/2} $. Thus the above quantity is of size
$\SD{N^{-1/2}}$, so applying this back to \eqref{eq:J1+J2},
	\begin{align*}
		\vv_1^\top(\vGamma+zm_{\tilde \vK}(\vGamma)\cdot\vSigma
)^{-1}\vv_2+\vv_1^\top\vR(\vGamma)\vv_2 \prec N^{-1/2}.
	\end{align*}
	Finally, from \eqref{eq:diff_m} and Lemma~\ref{lemm:support_G0} 
we have $m_{\tilde \vK}(\vGamma)=\tilde m_{N,0}(z)+\SD{N^{-1}}$, and applying
this above completes the proof.
\end{proof-of-lemma}

\begin{proof-of-lemma}[\ref{lemm:apply_fluctuation2}]
The proof is similar to Lemma \ref{lemm:apply_fluctuation}, replacing
$r$ and $n$ throughout by $r+1$ and $n+1$,
$\vG_r^{(S)}$ by $[\vu^{(S)},\vG_r^{(S)}]$,
$\vSigma$ by $\widetilde\vSigma$, and
$\vR^{(S)}(\vGamma)$ and $\tilde m_{\vK}^{(S)}(\vGamma)$ by
\[\vR^{(S)}(\widetilde \vGamma)
=\big([\vu^{(S)},\vG^{(S)}]^\top [\vu^{(S)},\vG^{(S)}]-\widetilde
\vGamma\big)^{-1},
\quad \tilde m_{\vK}^{(S)}(\widetilde \vGamma)
=\frac{1}{N}\Tr \vR^{(S)}(\widetilde
\vGamma)+\Big(1-\frac{n+1}{N}\Big)\left({-}\frac{1}{z}\right).\]
The only difference here is that $\widetilde \vSigma$ is no longer
diagonal, leading to the following minor modifications of the preceding
proof: The bound 
\[\frac{1}{n+1}\Tr\left(\widetilde\vSigma
-\begin{pmatrix} \mathbf{0} & \mathbf{0} & \mathbf{0} \\
\mathbf{0} & \mathbf{0} & \mathbf{0} \\
\mathbf{0} & \mathbf{0} & \vSigma_0 \end{pmatrix}\right)
\vR^{(S)}(\widetilde\vGamma) \prec \frac{1}{N}\]
analogous to \eqref{eq:SigmaRcompare} follows upon noting that
(with $\E[u\vg]^\top=\begin{pmatrix} \E[u\vg_r]^\top & \E[u\vg_0]^\top
\end{pmatrix}$)
\[\widetilde \vSigma-\begin{pmatrix} \mathbf{0} & \mathbf{0} & \mathbf{0} \\
\mathbf{0} & \mathbf{0} & \mathbf{0} \\
\mathbf{0} & \mathbf{0} & \vSigma_0 \end{pmatrix}
=\begin{pmatrix} \E[u^2] & \E[u\vg_r]^\top & \E[u\vg_0]^\top \\
\E[u\vg_r] & \vSigma_r & \mathbf{0} \\
\E[u\vg_0] & \mathbf{0} & \mathbf{0} \end{pmatrix}\]
still is of low rank, with rank at most $r+2$.
Writing as shorthand $\tilde m_{\vK}^{(S)}=\tilde
m_{\vK}^{(S)}(\widetilde \vGamma)$, the bound
	\[\bI\{\|(z^{-1}\widetilde\vGamma+\tilde m_{\vK}^{(S)}
\widetilde\vSigma)^{-1}\|>C_0\} \prec 0\]
analogous to \eqref{eq:extended1plusmSigma} follows from
	\[(z^{-1}\widetilde\vGamma+\tilde m_{\vK}^{(S)}\widetilde\vSigma)^{-1}=~\begin{pmatrix}
        \tilde m_{\vK}^{(S)}\E[u^2]
+\frac{\alpha}{z}+1 & \tilde m_{\vK}^{(S)}
\E[u\vg_r]^\top & \tilde m_{\vK}^{(S)} \E[u\vg_0]^\top \\
\tilde m_{\vK}^{(S)} \E[u\vg_r] &
\tilde m_{\vK}^{(S)}\vSigma_r
+(\frac{\alpha}{z}+1)\vI_r & \mathbf{0} \\
	\tilde m_{\vK}^{(S)} \E[u\vg_0] & \mathbf{0} &
\tilde m_{\vK}^{(S)}\vSigma_0+\vI
\end{pmatrix}^{-1},\]
the bound $\bI\{\|\tilde m_{\vK}^{(S)}\vSigma_0+\vI\|^{-1}>C\} \prec 0$ for
the lower-right block as follows from
\eqref{eq:lowerblockImSigma}, and the bound for the inverse of
its Schur-complement
\begin{align*}
&\bI\Bigg\{\Bigg\|\Bigg[\begin{pmatrix}
 \tilde m_{\vK}^{(S)}\E[u^2]
+\frac{\alpha}{z}+1 & \tilde m_{\vK}^{(S)}\E[u\vg_r]^\top \\
\tilde m_{\vK}^{(S)}\E[u\vg_r] & \tilde m_{\vK}^{(S)}\vSigma_r
+(\frac{\alpha}{z}+1)\vI_r \end{pmatrix}\\
&\hspace{1in}
-\begin{pmatrix} \tilde m_{\vK}^{(S)}\E[u\vg_0]^\top \\ \mathbf{0} \end{pmatrix}
(\tilde m_{\vK}^{(S)}\vSigma_0+\vI)^{-1}
\begin{pmatrix} \tilde m_{\vK}^{(S)}\E[u\vg_0] & \mathbf{0}
\end{pmatrix}\Bigg]^{-1}
\Bigg\|>C\Bigg\} \prec 0
\end{align*}
which holds uniformly over $z \in U(\eps)$ for any $|\alpha|>\alpha_0$ 
sufficiently large, by an argument analogous to \eqref{eq:sigmamin1plusmSigma}.
The remainder of the proof
is identical to that of Lemma \ref{lemm:apply_fluctuation}, and we omit the
details. 
\end{proof-of-lemma}

\subsection{Analysis of outliers}

Let $\vV_r,\vGamma(z,\alpha),\vcR(z,\alpha),\widetilde
\vcR(z,\alpha)$ be as defined in the preceding section.
Consider the decomposition of
$\widetilde \vcR(z,\alpha)$ as in \eqref{eq:tildeR}
into its blocks of dimensions 1, $n$, and $N$, and define
\begin{align}
\widetilde \cR_{11}(z,\alpha)&:=\begin{pmatrix} 1 \\ 0 \\ 0 \end{pmatrix}^\top
\widetilde \vcR(z,\alpha) \begin{pmatrix} 1 \\ 0 \\ 0 \end{pmatrix}
=\frac{1}{-z-\alpha+\vu^\top \vu-\vu^\top
\vG(\vG^\top\vG-\vGamma(z,\alpha))^{-1} \vG^\top\vu},
\label{eq:R11}\\
\widetilde \vcR_{1V}(z,\alpha)&:=\begin{pmatrix} 1 \\ 0 \\ 0 \end{pmatrix}^\top
\widetilde \vcR(z,\alpha) \begin{pmatrix} 0 \\ \vV_r \\ 0 \end{pmatrix}
={-}\widetilde \cR_{11}(z,\alpha) \cdot
\vu^\top \vG\Big(\vG^\top \vG-\vGamma(z,\alpha)\Big)^{-1}\vV_r,
\label{eq:R1V}
\end{align}
where the second equalities follow from block matrix inversion of
the lower $2 \times 2$ blocks of $\widetilde \vcR(z,\alpha)$,
followed by block matrix inversion of the full matrix $\widetilde
\vcR(z,\alpha)$. Set
	\begin{equation}\label{eq:eigen-vector-iden}
	\vM_{\vK}(z,\alpha) = 
\vI_r+\alpha \begin{pmatrix} \vV_r^\top & \mathbf{0} \end{pmatrix}
\vcR(z,\alpha)\begin{pmatrix} \vV_r \\ \mathbf{0} \end{pmatrix}.
	\end{equation} 

\begin{prop}\label{prop:master}
Fix any $\eps>0$ and any $\alpha \in \R$ sufficiently large that satisfies
Lemmas \ref{lemm:apply_fluctuation} and \ref{lemm:apply_fluctuation2}.
Then on the event $\cE \cap \widetilde \cE$ of
Lemmas \ref{lemm:apply_fluctuation} and \ref{lemm:apply_fluctuation2},
for all sufficiently large $N$,
\begin{enumerate}[label=(\alph*)]
\item $\widehat \lambda
\in U(\eps) \cap \R$ is an eigenvalue of $\vG^\top \vG$ if and only if
$\det \vM_{\vK}(\widehat\lambda,\alpha)=0$, and its multiplicity as an
eigenvalue of $\vG^\top \vG$ equals the dimension of
$\ker\vM_{\vK}(\widehat\lambda,\alpha)$.
\item Let $\widehat\vv\in\R^n$ be a unit eigenvector of $\vG^\top \vG$ (i.e.\
right singular vector of $\vG$)
corresponding to an eigenvalue $\widehat\lambda \in U(\eps) \cap \R$. Then
$\vV_r^\top\widehat\vv$ is a non-zero vector in $\ker
\vM_{\vK}(\widehat\lambda,\alpha)$, and
\begin{align}\label{eq:norm_vvr_v}
\frac{1}{\alpha^2} = \widehat{\vv}^\top \vV_r
\begin{pmatrix} \vV_r \\ \mathbf{0} \end{pmatrix}^\top
\vcR(\widehat\lambda,\alpha) \begin{pmatrix} \vI_n & \mathbf{0} \\
\mathbf{0} & \mathbf{0} \end{pmatrix} 
\vcR(\widehat\lambda,\alpha) \begin{pmatrix} \vV_r \\ \mathbf{0} \end{pmatrix}
\vV_r^\top \widehat{\vv}
\end{align}
For any vector $\vv\in\R^n$, we have
\begin{equation}\label{eq:vvhat}
\vv^\top \widehat\vv+\alpha\begin{pmatrix} \vv^\top &\mathbf{0}
\end{pmatrix}\vcR(\widehat\lambda,\alpha)\begin{pmatrix} \vV_r \\ \mathbf{0} \end{pmatrix}
\vV_r^\top \widehat\vv=0.
\end{equation}
\item Let $\vu$ be as in Theorem \ref{thm:spiked}(c), and let $\hat{\vu} \in \R^N$ be
a unit eigenvector of $\vG\vG^\top$ (i.e.\ left singular vector of $\vG$)
corresponding to the eigenvalue $\widehat\lambda \in U(\eps) \cap \R$. Then
\begin{equation}\label{eq:uuhat}
	\vu^\top \widehat\vu=\frac{\alpha }{\widehat\lambda^{1/2}\,\widetilde
\cR_{11}(\widehat\lambda,\alpha)}\widetilde\vcR_{1V}(\widehat\lambda,\alpha) \vV_r^\top\widehat\vv.
\end{equation}
\end{enumerate}
\end{prop}
\begin{proof}
For part (a), note that if $\widehat\lambda$ is an eigenvalue of $\vG^\top \vG$,
i.e.\ $\widehat\lambda^{1/2}$ is a singular value of $\vG$ with left and
right unit singular vectors $\widehat \vu$ and $\widehat\vv$, then
\[0=\begin{pmatrix} -\widehat\lambda \vI_n & \vG^\top \\ \vG & -\vI_N \end{pmatrix}
\begin{pmatrix} \widehat\vv \\ \widehat\lambda^{1/2}\widehat\vu \end{pmatrix}\]
which implies, for any $\alpha \in \R$,
\[{-}\alpha\begin{pmatrix} \vV_r \\ \mathbf{0} \end{pmatrix}
\cdot \vV_r^\top \widehat\vv=\begin{pmatrix} -\widehat\lambda\vI_n-\alpha\vV_r\vV_r^\top & \vG^\top \\ \vG & -\vI_N \end{pmatrix}
\begin{pmatrix} \widehat\vv \\ \widehat\lambda^{1/2}\widehat\vu \end{pmatrix}.\]
Fixing $\alpha \in \R$ large enough, on the event $\cE$
of Lemma \ref{lemm:apply_fluctuation}, the generalized resolvent
\[\vcR(\widehat\lambda,\alpha)=\begin{pmatrix} -\widehat\lambda\vI_n-\alpha\vV_r\vV_r^\top & \vG^\top \\ \vG & -\vI_N
\end{pmatrix}^{-1}\]
exists, and multiplying both sides by $\vcR(\widehat\lambda,\alpha)$ gives
\begin{equation}\label{eq:eigequation}
\begin{pmatrix} \widehat\vv \\ \widehat\lambda^{1/2}\widehat\vu \end{pmatrix}
=-\alpha \vcR(\widehat\lambda,\alpha)\begin{pmatrix}\vV_r \\ \mathbf{0} \end{pmatrix}
\cdot \vV_r^\top \widehat\vv.
\end{equation}
Then, multiplying by $(\vV_r^\top\;\mathbf{0})$ on both sides and re-arranging, we get
$\vM_{\vK}(\widehat\lambda,\alpha) \cdot \vV_r^\top \widehat \vv=0$.

We remark
that if $\widehat\lambda$ is
an eigenvalue of multiplicity $k$, and $\vG$ has corresponding linearly
independent left singular vectors $\widehat\vu_1,\ldots,\widehat\vu_k$
and right singular vectors $\widehat\vv_1,\ldots,\widehat\vv_k$, then the
vectors $\{(\widehat \vv_j,\widehat\lambda^{1/2}\widehat \vu_j)\}_{j=1}^k$ on the left side
of \eqref{eq:eigequation} are linearly independent, implying that the vectors
$\{\vV_r^\top \widehat\vv_j\}_{j=1}^k$ on the right side must also be
(non-zero and) linearly independent vectors in $\ker \vM_{\vK}(\widehat\lambda,\alpha)$.
Conversely, if $\{\vy_j\}_{j=1}^k$ are linearly independent vectors in
$\ker \vM_{\vK}(\widehat\lambda,\alpha)$, then defining
\[\begin{pmatrix} \widehat\vv_j \\ \widehat\lambda^{1/2}\widehat \vu_j \end{pmatrix}
=-\alpha \vcR(\widehat\lambda,\alpha)\begin{pmatrix} \vV_r \\ \mathbf{0} \end{pmatrix}
\cdot \vy_j\]
and multiplying by $(\vV_r^\top\; \mathbf{0})$, we must have
$\vV_r^\top \widehat \vv_j=({-}\vM_{\vK}(\widehat\lambda,\alpha)+\vI)\vy_j=\vy_j$. Thus the
pairs $(\widehat\vv_j,\widehat\lambda^{1/2}\widehat\vu_j)$ are linearly independent vectors
satisfying \eqref{eq:eigequation}, and multiplying by
$\vcR(\widehat\lambda,\alpha)^{-1}$ and rearranging shows that
$\widehat\lambda^{1/2}$ must be a singular
value of $\vG$ with multiplicity at least $k$, with corresponding
singular vectors $\{(\widehat\vv_j,\widehat\vu_j)\}_{j=1}^k$.
This establishes part (a).

For part (b), the above argument has shown $\vV_r^\top \widehat \vv \in \ker
\vM_{\vK}(\widehat\lambda,\alpha)$. Multiplying
\eqref{eq:eigequation} on the left by
$$\begin{pmatrix} \vI_n & \mathbf{0} \\
	\mathbf{0} & \mathbf{0} \end{pmatrix}$$
and taking the squared norm (noting that $\widehat\lambda,\alpha$
and $\vcR(\widehat\lambda,\alpha)$ here are real) shows \eqref{eq:norm_vvr_v}.
Multiplying \eqref{eq:eigequation} on the left by $(\vv^\top\;\mathbf{0})$ shows
\eqref{eq:vvhat}. For part (c),
multiplying (\ref{eq:eigequation}) by $(\mathbf{0}\;\vu^\top)$, we have
\begin{align*}
\widehat\lambda^{1/2}\,\vu^\top \widehat\vu&=-\alpha\begin{pmatrix} \mathbf{0}\\ \vu \end{pmatrix}^\top
\vcR(\widehat\lambda,\alpha)\begin{pmatrix} \vV_r \\ \mathbf{0} \end{pmatrix} \cdot \vV_r^\top \widehat\vv
=-\alpha \vu^\top \vG \Big(\vG^\top \vG-\vGamma(\widehat\lambda,\alpha)\Big)^{-1}\vV_r
\cdot \vV_r^\top \widehat\vv
\end{align*}
where the second equality follows from the block matrix inversion of
$\vcR(\widehat\lambda,\alpha)$.
Then, recalling the forms of $\widetilde
\cR_{11}$ and $\widetilde \vcR_{1V}$ from \eqref{eq:R11} and \eqref{eq:R1V},
this gives
\[\vu^\top \widehat\vu=\frac{\alpha \widetilde
\vcR_{1V}(\widehat\lambda,\alpha)}{\widehat\lambda^{1/2}\,\widetilde
\cR_{11}(\widehat\lambda,\alpha)} \cdot \vV_r^\top \widehat\vv\]
which is \eqref{eq:uuhat}.
\end{proof}

For notational convenience, let us now introduce the shorthand
\[\psi_{N,0}(z)=z\tilde m_{N,0}(z), \qquad \psi(z)=z\tilde m(z).\]
By Lemma \ref{lemm:apply_fluctuation}(b) applied with $(\vv_1,\vv_2)$ being the
columns of $\vV_r$,
we see that $\vM_{\vK}(z,\alpha)$ is well-approximated by the (deterministic,
$N$-dependent) matrix
\begin{align}\label{eq:M(z)}
\vM_N(z,\alpha):=\vI_r-\alpha\Big((\alpha+z)\vI_r+\psi_{N,0}(z)
\diag(\lambda_1(\vSigma),\ldots,\lambda_r(\vSigma))\Big)^{-1}.
\end{align} 
To show Theorem \ref{thm:spiked}(a), we translate this approximation into a
comparison of the roots of $0=\det \vM_{\vK}(z,\alpha)$ and $0=\det
\vM_N(z,\alpha)$,
where the latter are explicitly given by $z_{N,0}(-1/\lambda_i(\vSigma))$ for
the function $z_{N,0}(\cdot)$ defined in~\eqref{eq:zN0}.\\

\begin{proof-of-theorem}[\ref{thm:spiked}(a)]
Let us fix any $\eps>0$ and $\alpha \in \R$ satisfying Lemmas
\ref{lemm:apply_fluctuation} and \ref{lemm:apply_fluctuation2}, and denote 
\[f_{N,i}(z,\alpha)=1-\frac{\alpha}{\alpha+z+\psi_{N,0}(z)\lambda_i(\vSigma)}\]
for each $i \in [r]$.
Then $\det\vM_N(z,\alpha)=\prod_{i=1}^r f_{N,i}(z,\alpha)$. Define also the
limiting functions
\[f_i(z,\alpha)=1-\frac{\alpha}{\alpha+z+\psi(z)\lambda_i},
\qquad \vM(z,\alpha)=\vI_r-\alpha\Big((\alpha+z)\vI_r+\psi(z)
\diag(\lambda_1,\ldots,\lambda_r)\Big)^{-1}\]
so $\det \vM(z,\alpha)=\prod_{i=1}^r f_i(z,\alpha)$.
We first analyze the roots of $0=\det \vM(z,\alpha)$: By the
definition $\psi(z)=z\tilde m(z)$, observe that $z \in \R
\setminus \supp{\tilde \mu}$ satisfies $0=\det \vM(z,\alpha)$ if and only if
either $z=0$ or 
\[\tilde m(z)=-1/\lambda_i \text{ for some } i \in [r].\]
(This condition is the same for any non-zero $\alpha \in \R$.)
Let $\cT=\{0\} \cup \{-1/\lambda:\lambda \in \supp{\nu}\}$ be as in
\eqref{eq:zdomain} where $\nu$ is the limit spectral law of $\vSigma_0$.
Then $-1/\lambda_i \in \R \setminus \cT$ for all $i \in [r]$ under Assumption \ref{assum:spike}, so
Proposition~\ref{prop:inv_z} implies that $\tilde m(z)=-1/\lambda_i$ holds
for some $z \in \R \setminus \supp{\tilde \mu}$ if and only if
$z'(-1/\lambda_i)>0$, i.e.\ $i \in \cI$. If $i \in \cI$, then
$\tilde m(z)=-1/\lambda_i$ holds for $z=z(-1/\lambda_i)$,
and we must have $z(-1/\lambda_i)>0$ strictly because for any $z \leq 0$,
we have $\tilde m(z)>0$
(and hence $\tilde m(z) \neq -1/\lambda_i$)
by the definition $\tilde m(z)=\int \frac{1}{x-z}d\tilde \mu(x)$.
Thus the roots of $0=\det \vM(z,\alpha)$ in $\R \setminus \cS=\R \setminus
(\supp{\tilde \mu} \cup \{0\})$ --- and hence
in $U(\eps) \cap \R$ for any sufficiently small $\eps>0$ --- are given
precisely by
\[z_i:=z(-1/\lambda_i) \text{ for } i \in \cI.\]
Since $\tilde m'(z)=\int \frac{1}{(x-z)^2}d\tilde \mu(x)>0$
for all $z \in \R \setminus \cS$, and $\{\lambda_i:i \in \cI\}$ are
distinct by assumption, these values $\{z_i:i \in \cI\}$ are simple roots of
$0=\det \vM(z_i,\alpha)$. Then
$(\det \vM)'(z_i,\alpha) \neq 0$ where $(\det \vM)'$ denotes the derivative in
$z$.

Lemma \ref{lemma:supportcompareb}(c) implies
$\tilde m_{N,0}(z) \to \tilde m(z)$ and $\tilde m_{N,0}'(z) \to \tilde m'(z)$
uniformly over $z \in U(\eps)$. Since also $\lambda_i(\vSigma) \to \lambda_i$,
we have $\det \vM_N(z,\alpha) \to \det \vM(z,\alpha)$ and
$(\det \vM_N)'(z,\alpha) \to (\det \vM)'(z,\alpha)$ uniformly over
$z \in U(\eps)$. This, together with the above condition
$(\det \vM)'(z_i,\alpha) \neq 0$, imply
that for all large $N$, the roots $z_{N,i} \in U(\eps) \cap \R$
of $0=\det \vM_N(z,\alpha)$
are in 1-to-1 correspondence with, and converge to, the above roots $z_i \in
U(\eps) \cap \R$ of
$0=\det \vM(z,\alpha)$. We note that $0=\det\vM_N(z,\alpha)$ if and only if
either $z=0$ or
\begin{equation}\label{eq:equation_z_i}
\tilde{m}_{N,0}(z)=-1/\lambda_i(\vSigma) \text{ for some } i \in [r].
\end{equation}
For each $i \in \cI$, we have $\lambda_i(\vSigma) \to \lambda_i$ where
$z'(-1/\lambda_i)>0$. Recall from Lemma \ref{lemma:supportcompareb}(a)
that $z_{N,0}(\tilde m) \to z(\tilde m)$ and $z_{N,0}'(\tilde m) \to z'(\tilde m)$
uniformly over compact subsets of $\R \setminus \cT$.
Then $z_{N,0}'(-1/\lambda_i(\vSigma)) \to z'(-1/\lambda_i)$, so
also $z_{N,0}'(-1/\lambda_i(\vSigma))>0$ for all large $N$. Then
Proposition~\ref{prop:inv_z} implies that \eqref{eq:equation_z_i} holds
for $z_{N,i}:=z_{N,0}(-1/\lambda_i(\vSigma))$. We have
$z_{N,i} \to z_i=z(-1/\lambda_i)$, so these must be the
roots of $\det \vM_N(z,\alpha)$ in $U(\eps) \cap \R$. Thus we have shown that
for any sufficiently small $\eps>0$ and all large $N$,
the roots $z \in U(\eps) \cap \R$ of $0=\det \vM_N(z,\alpha)$ are precisely
the values
\[z_{N,i}:=z_{N,0}(-1/\lambda_i(\vSigma)) \text{ for } i \in \cI,\]
and $z_{N,i} \to z_i>0$ for each $i \in \cI$.

Finally, we apply Lemma \ref{lemm:apply_fluctuation}(b) with $(\vv_1,\vv_2)$
being the columns of $\vV_r$. On the event $\cE$ of
Lemma \ref{lemm:apply_fluctuation}(a), we have
\begin{equation}\label{eq:MKlipschitz}
\norm{\vM_{\vK}(z,\alpha)} \leq C, \qquad
	\norm{\vM_{\vK}(z,\alpha)-\vM_{\vK}(z',\alpha)} \leq C|z-z'|
\end{equation}
for some $C>0$ and all $z,z' \in U(\eps/2)$.
Also $|\tilde{m}_{N,0}(z)|,|\tilde{m}_{N,0}'(z)|<C$
for a constant $C>0$, all $z \in U(\eps)$, and all large $N$, and thus
\begin{equation}\label{eq:MNlipschitz}
\norm{\vM_N(z,\alpha)} \leq C, \qquad
\norm{\vM_N(z,\alpha)-\vM_N(z',\alpha)} \leq C|z-z'|
\end{equation}
for some $C>0$ and all $z,z' \in U(\eps/2)$.
Then, applying Lemma \ref{lemm:apply_fluctuation}(b) and the Lipschitz
bounds of \eqref{eq:MKlipschitz} and \eqref{eq:MNlipschitz} to take
a union bound over a sufficiently fine covering net of $U(\eps/2)$, we get
\begin{equation}\label{eq:MKuniformapprox}
\sup_{z \in U(\eps/2)} \norm{\vM_N(z,\alpha)-\vM_{\vK}(z,\alpha)} \prec 1/\sqrt{N}.
\end{equation}
Applying also the first bounds of \eqref{eq:MKlipschitz}
and \eqref{eq:MNlipschitz}, this gives
\begin{equation}\label{eq:detMcompare}
\sup_{z \in U(\eps/2)} |\det \vM_N(z,\alpha)-\det \vM_{\vK}(z,\alpha)|
\prec 1/\sqrt{N}.
\end{equation}
Since $\det \vM_N(z,\alpha)$ and $\det \vM_{\vK}(z,\alpha)$ are both holomorphic
over $z \in U(\eps/2)$ on this event $\cE$,
the Cauchy integral formula then implies
\[\sup_{z \in U(\eps)} |(\det \vM_N)'(z,\alpha)-(\det \vM_{\vK})'(z,\alpha)|
\prec 1/\sqrt{N}.\]
In particular, combining with the uniform convergence statements
$\det \vM_N(z,\alpha) \to \det \vM(z,\alpha)$ and $(\det \vM_N)'(z,\alpha) \to
(\det \vM)'(z,\alpha)$ over $z \in U(\eps)$ as argued above,
this shows that on an event $\cE$ satisfying $\bI\{\cE^c\} \prec
0$ and for some $\delta_N \to 0$, we have
\[\sup_{z \in U(\eps) \cap \R} |\det \vM(z,\alpha)-\det \vM_{\vK}(z,\alpha)|,
|(\det \vM)'(z,\alpha)-(\det \vM_{\vK})'(z,\alpha)|<\delta_N.\]
Thus, on this event $\cE$ and as $N \to \infty$, the roots $\widehat \lambda_i
\in U(\eps) \cap \R$ of $0=\det
\vM_{\vK}(z,\alpha)$ are also in 1-to-1 correspondence with, and converge to,
the roots $z_i \in U(\eps) \cap \R$ of $0=\det \vM(z,\alpha)$.
Furthermore, the condition $(\det \vM)'(z_i,\alpha) \neq 0$ implies that
$|(\det \vM_N)'(z,\alpha)|$ and $|(\det \vM_{\vK})'(z,\alpha)|$ are bounded away
from 0 in a neighborhood of each such root $z_i$, so \eqref{eq:detMcompare}
then implies that the corresponding roots $\widehat \lambda_i$ and $z_{N,i}$ of
$0=\det \vM_{\vK}(z,\alpha)$ and $0=\det \vM_N(z,\alpha)$ satisfy
\[|\widehat \lambda_i-z_{N,i}| \prec 1/\sqrt{N}.\]

Proposition \ref{prop:master} shows that on this event $\cE$,
these roots $\{\widehat \lambda_i:i \in \cI\}$ are
precisely the eigenvalues of $\vG^\top \vG$ in $U(\eps) \cap \R$.
By the definition of $\vM(z_i,\alpha)$,
each root $z_i$ of $\det \vM(z_i,\alpha)$ is such that $\ker \vM(z_i,\alpha)$
has dimension 1. Since $\bI\{\cE\}(\widehat\lambda_i-z_i) \to 0$, we have
$\bI\{\cE\}\|\vM_{\vK}(\widehat \lambda_i,\alpha)-\vM(z_i,\alpha)\| \to 0$,
so $\ker \vM_{\vK}(\widehat \lambda_i,\alpha)$ also has dimension 1 on this
event $\cE$ for all large $N$. Then Proposition \ref{prop:master} implies that
the eigenvalues $\{\widehat \lambda_i:i \in \cI\}$ of $\vG^\top \vG$ are
simple, and thus in 1-to-1 correspondence with $\{\lambda_i:i \in \cI\}$.
This proves part (a) of the theorem.
 
\end{proof-of-theorem}

\begin{lemma}\label{lemm:R(z,alph)^2}
Under the assumptions of Theorem
\ref{thm:spiked}, for any fixed $\eps>0$, there exists $\alpha_0>0$ such
that fixing any $\alpha \in \C$ with $|\alpha|>\alpha_0$,
uniformly over $z \in U(\eps)$,
\begin{multline*}
		\Bigg\|  \begin{pmatrix} \vV_r \\ \mathbf{0} \end{pmatrix}^\top
		\vcR(z,\alpha) 
		\begin{pmatrix} \vI_n & \mathbf{0}  \\
		\mathbf{0}   & \mathbf{0} 
		\end{pmatrix} 
		\vcR(z,\alpha) \begin{pmatrix} \vV_r \\ \mathbf{0} \end{pmatrix} \\
		 -\left((\alpha+z)\vI_r+\psi_{N,0}(z)\vSigma_r
		\right)^{-2}\left(\vI_r+\psi_{N,0}'(z)\vSigma_r\right) \Bigg\|\prec\frac{1}{\sqrt{N}}.
\end{multline*}

\end{lemma}
\begin{proof}
Fix any $\alpha \in \C$ satisfying Lemma \ref{lemm:apply_fluctuation}, and
denote
\[f_N(z,\alpha):=\begin{pmatrix}
	\vV_r^\top &\mathbf{0}
\end{pmatrix}\vcR(z,\alpha)\begin{pmatrix}
	\vV_r
	\\\mathbf{0}
\end{pmatrix}, \qquad
g_N(z,\alpha):=-\left((\alpha+z)\vI_r+\psi_{N,0}(z)\vSigma_r  \right)^{-1}.\]
Applying Lemma \ref{lemm:apply_fluctuation}(b) and
the Lipschitz continuity statements of
\eqref{eq:MKlipschitz} and \eqref{eq:MNlipschitz} to take a union bound over a
sufficiently fine covering net of $U(\eps/2)$, we have
\[\sup_{z \in U(\eps/2)} \|f_N(z,\alpha)-g_N(z,\alpha)\| \prec 1/\sqrt{N}.\]
Then by the Cauchy integral formula,
$\sup_{z \in U(\eps)} \|f_N'(z,\alpha)-g_N'(z,\alpha)\| \prec 1/\sqrt{N}$
where $f_N'$ and $g_N'$ denote the entrywise derivatives in $z$. 
The lemma follows, since differentiating $\vcR(z,\alpha)$ in
\eqref{eq:matrix_block} shows
\[f_N'(z,\alpha)=\begin{pmatrix}
	\vV_r^\top &\mathbf{0}	
\end{pmatrix}\vcR(z,\alpha)
\begin{pmatrix} \vI & \mathbf{0} \\ \mathbf{0} & \mathbf{0} \end{pmatrix}
\vcR(z,\alpha)
\begin{pmatrix}
	\vV_r \\\mathbf{0},
\end{pmatrix}\]
while $g_N'(z,\alpha) = ((\alpha+z)\vI_r+\psi_{N,0}(z)\vSigma_r)^{-2}
(\vI_r+\psi_{N,0}'(z)\vSigma_r)$.
\end{proof}

\begin{proof-of-theorem}[\ref{thm:spiked}(b)]
Let $\widehat \vv_i$ be the given unit-norm eigenvector of $\vK$ with eigenvalue
$\widehat \lambda_i$. Let $z_{N,i}=z_{N,0}(-1/\lambda_i(\vSigma))$ and
$z_i=z(-1/\lambda_i)$.
Then, fixing any $\alpha \in \R$ large enough to satisfy Lemmas
\ref{lemm:apply_fluctuation} and \ref{lemm:apply_fluctuation2},
Proposition~\ref{prop:master}(b) shows that
	$\vV_r^\top\widehat{\vv}_i \in\ker
\vM_{\vK}(\widehat\lambda_i,\alpha)$.
By \eqref{eq:MKuniformapprox}, \eqref{eq:MNlipschitz}, and the bound $|\widehat
\lambda_i-z_{N,i}| \prec N^{-1/2}$ of part (a) of the theorem already proven,
we have
	\begin{align}
	\norm{\vM_{\vK}(\widehat\lambda_i,\alpha)-\vM_N(z_{N,i},\alpha)} &\le
\norm{\vM_{\vK}(\widehat\lambda_i,\alpha)-\vM_N(\widehat
\lambda_i,\alpha)}+\norm{\vM_N(\widehat\lambda_i,\alpha)-\vM_N(z_{N,i},\alpha)}
\nonumber\\
&\prec N^{-1/2}.\label{eq:MKNlambdacompare}
	\end{align}
Let $\vv_1,\ldots,\vv_r$ denote the columns of $\vV_r$, which are the unit
eigenvectors of $\vSigma$.
	Then, applying $\vV_r^\top\widehat{\vv}_i \in\ker
\vM_{\vK}(\widehat\lambda_i,\alpha)$, \eqref{eq:MKNlambdacompare}, and
the definition of $\vM_N(z,\alpha)$, and noting that
$\psi_{N,0}(z_{N,i})=z_{N,i}\tilde
m_{N,0}(z_{N,i})=-z_{N,i}/\lambda_i(\vSigma)$, we have
\[\norm{\vM_N(z_{N,i},\alpha) \cdot \vV_r^\top\widehat{\vv}_i}^2
=\sum_{j=1}^r
\left(1-\frac{\alpha}{\alpha+z_{N,i}(1-\lambda_j(\vSigma)/\lambda_i(\vSigma))}
\right)^2 (\vv_j^\top \widehat\vv_i)^2 \prec 1/N.\]
For each $j \in [r] \setminus \{i\}$, we have that
$z_{N,i}(1-\lambda_j(\vSigma)/\lambda_i(\vSigma))$ is bounded away from 0
as $N \to \infty$ because $z_{N,i} \to z_i>0$ and
$\lambda_j(\vSigma)/\lambda_i(\vSigma) \to \lambda_j/\lambda_i \neq 1$.
So this implies
	\begin{equation}\label{eq:y_q_j}
	|\vv^\top_j\widehat{\vv}_i|^2\prec 1/N\quad\text{ for all }\quad j\in [r]\setminus \{i\}.
	\end{equation}
	At the same time, applying Lemma~\ref{lemm:R(z,alph)^2} 
and $|\widehat \lambda_i-z_{N,i}| \prec N^{-1/2}$ to
bound~\eqref{eq:norm_vvr_v} in Proposition~\ref{prop:master}(b), we have 
\begin{align}
\frac{1}{\alpha^2}
&=\widehat\vv_i^\top \vV_r
\begin{pmatrix} \vV_r \\ \mathbf{0} \end{pmatrix}^\top
\vcR(z_{N,i},\alpha) \begin{pmatrix} \vI_n & \mathbf{0} \\
\mathbf{0} & \mathbf{0} \end{pmatrix} 
\vcR(z_{N,i},\alpha) \begin{pmatrix} \vV_r \\ \mathbf{0} \end{pmatrix}
\vV_r^\top \widehat\vv_i+\SD{N^{-1/2}}\notag\\
&=\widehat\vv_i^\top \vV_r
\Big((\alpha+z_{N,i})\vI_r+\psi_{N,0}(z_{N,i})\vSigma_r\Big)^{-2}
\Big(\vI_r+\psi_{N,0}'(z_{N,i})\vSigma_r\Big) \vV_r^\top \widehat\vv_i
+\SD{N^{-1/2}}\notag\\
&=|\vv^\top_i\widehat{\vv}_i|^2\cdot\frac{1+\psi_{N,0}'(z_{N,i})\lambda_i(\vSigma)}{\alpha^2}\notag\\
&\hspace{1in}+\sum_{j\neq
i} |\vv^\top_j\widehat{\vv}_i|^2
\cdot\frac{1+\psi_{N,0}'(z_{N,i})\lambda_j(\vSigma)}{(\alpha+z_{N,i}(1-\lambda_j(\vSigma)/\lambda_i(\vSigma)))^2}+\SD{N^{-1/2}}\notag\\
&=
|\vv^\top_i\widehat{\vv}_i|^2\cdot\frac{1+\psi_{N,0}'(z_{N,i})\lambda_i(\vSigma)}{\alpha^2}+\SD{N^{-1/2}}, \label{eq:x-y}
	\end{align}
	the last equality applying \eqref{eq:y_q_j}.
Observe that
\begin{align*}
1+\psi_{N,0}'(z_{N,i})\lambda_i(\vSigma)
&=1+z_{N,i}\tilde m_{N,0}'(z_{N,i})\lambda_i(\vSigma)
+\tilde m_{N,0}(z_{N,i})\lambda_i(\vSigma)\\
&=z_{N,i}\tilde m_{N,0}'(z_{N,i})\lambda_i(\vSigma)
=z_{N,i}\lambda_i(\vSigma)/z_{N,0}'(-1/\lambda_i(\vSigma)),
\end{align*}
where the last two equalities use $z_{N,i}=z_{N,0}(-1/\lambda_i(\vSigma))$
and $\tilde m_{N,0}(\cdot)$ is the inverse function of $z_{N,0}(\cdot)$.
Then, multiplying by $\alpha^2/(1+\psi_{N,0}'(z_{N,i})\lambda_i(\vSigma))$
we obtain
\[|\vv^\top_i\widehat{\vv}_i|^2=\frac{z_{N,0}'(-1/\lambda_i(\vSigma))}
{z_{N,i}\lambda_i(\vSigma)}+\SD{N^{-1/2}}
=\varphi_{N,0}({-}1/\lambda_i(\vSigma))+\SD{N^{-1/2}},\]
where we recall $\varphi_{N,0}$ from \eqref{eq:zN0}.
We have $\varphi_{N,0}(-1/\lambda_i(\vSigma)) \to \varphi(-1/\lambda_i)
=z'(-1/\lambda_i)/(\lambda_iz_i)>0$, so taking a square root gives
\begin{equation}\label{eq:vivhati}
|\vv^\top_i\widehat{\vv}_i|=\sqrt{\varphi_{N,0}({-}1/\lambda_i(\vSigma))}
+\SD{N^{-1/2}}.
\end{equation}

Finally, for any unit vector $\vv \in \R^n$, by \eqref{eq:vvhat} in
Proposition~\ref{prop:master}(b),
Lemma \ref{lemm:apply_fluctuation}(b), and the bound
$|\widehat\lambda_i-z_{N,i}| \prec N^{-1/2}$ in part (a) of the theorem
already shown, we know that 
\begin{align*}
\vv^\top\widehat{\vv}_i =~& {-}\alpha \cdot \begin{pmatrix} \vv^\top & \mathbf{0}
\end{pmatrix}
\vcR(z_{N,i},\alpha)\begin{pmatrix} \vV_r \\ \mathbf{0} \end{pmatrix}
\cdot \vV_r^\top\widehat{\vv}_i+\SD{N^{-1/2}}\\
=~& {-}\alpha
\sum_{j=1}^r\frac{\vv^\top\vv_j\cdot\vv_j^\top\widehat{\vv}_i}{\alpha+z_{N,i}+\psi_{N,0}(z_{N,i})\lambda_j(\vSigma)}+\SD{N^{-1/2}} \\
=~& {-}\alpha \sum_{j=1}^r
\frac{\vv^\top\vv_j\cdot\vv_j^\top\widehat{\vv}_i}{\alpha+z_{N,i}\cdot(1-\lambda_j(\vSigma)/\lambda_i(\vSigma))}+\SD{N^{-1/2}}.
\end{align*}
Applying \eqref{eq:y_q_j} and \eqref{eq:vivhati}, only the summand with $j=i$
contributes, and we obtain as desired
\[|\vv^\top \widehat\vv_i|=\sqrt{\varphi_{N,0}({-}1/\lambda_i(\vSigma))}
\cdot |\vv^\top \vv_i|+\SD{N^{-1/2}}.\]
\end{proof-of-theorem}

\begin{proof-of-theorem}[\ref{thm:spiked}(c)]
Applying Lemma~\ref{lemm:apply_fluctuation2}(b) and block matrix inversion of
$\widetilde \vGamma+\psi_{N,0}(z)\widetilde \vSigma$ to the definitions of
$\widetilde \cR_{11}$ and $\widetilde \vcR_{1V}$ in \eqref{eq:R11} and
\eqref{eq:R1V}, we have
\begin{align*}
\left|\widetilde
\cR_{11}(z,\alpha)+\left(z+\alpha+\psi_{N,0}(z)\cdot\E[u^2]-\psi_{N,0}(z)^2
\cdot \E[u\vg]^\top\left(\vGamma+\psi_{N,0}(z)\vSigma\right)^{-1}\E[u\vg]\right)^{-1}\right|
&\prec \frac{1}{\sqrt{N}},\\
\left\|\widetilde \vcR_{1V}(z,\alpha)-\frac{\psi_{N,0}(z) \cdot
\E[u\vg]^\top\left(\vGamma+\psi_{N,0}(z)\vSigma
\right)^{-1}\vV_r}{z+\alpha+\psi_{N,0}(z)\cdot\E[u^2]-\psi_{N,0}(z)^2
\cdot \E[u\vg]^\top\left(\vGamma+\psi_{N,0}(z)\vSigma
\right)^{-1}\E[u\vg]}\right\| &\prec \frac{1}{\sqrt{N}}.
\end{align*}
Hence, 
\begin{align*}
	\left\|\frac{ \widetilde \vcR_{1V}(z,\alpha)}{ \widetilde
\cR_{11}(z,\alpha)}+\psi_{N,0}(z) \cdot
\E[u\vg]^\top\vV_r\cdot\left((\alpha+z)\vI_r+\psi_{N,0}(z)\vSigma_r
\right)^{-1}\right\| &\prec \frac{1}{\sqrt{N}}.
\end{align*} 
Applying this and the bound
$|\widehat\lambda_i-z_{N,i}| \prec N^{-1/2}$
to Proposition~\ref{prop:master}(c),
	\begin{align*}
		\vu^\top\widehat{\vu}_i=
\frac{\alpha}{\widehat\lambda_i^{1/2}}
\frac{\widetilde \vcR_{1V}(\widehat\lambda_i,\alpha)}
{\widetilde \cR_{11}(\widehat\lambda_i,\alpha)} \cdot \vV_r^\top \widehat\vv_i
&=-\frac{\alpha}{\sqrt{z_{N,i}}}
\sum_{j=1}^r\frac{\psi_{N,0}(z_{N,i}) \cdot
\E[u\vg]^\top\vv_j\cdot\vv_j^\top\widehat{\vv}_i}{\alpha+z_{N,i}+\psi_{N,0}(z_{N,i})\lambda_j(\vSigma)}+\SD{N^{-1/2}}\\
&=-\frac{\alpha}{\sqrt{z_{N,i}}}
\sum_{j=1}^r\frac{\psi_{N,0}(z_{N,i}) \cdot
\E[u\vg]^\top\vv_j\cdot\vv_j^\top\widehat{\vv}_i}{\alpha+z_{N,i}(1-\lambda_j(\vSigma)/\lambda_i(\vSigma))}+\SD{N^{-1/2}}.
	\end{align*}
Then, applying again \eqref{eq:y_q_j} and \eqref{eq:vivhati}, only the summand
with $j=i$ contributes, and this gives
	\begin{align*}
		|\vu^\top\widehat{\vu}_i|=~&
\frac{|\E[u\vg]^\top\vv_i|\cdot |\psi_{N,0}(z_{N,i})|\sqrt{\varphi_{N,0}
(-1/\lambda_i(\vSigma))}}{\sqrt{z_{N,i}}}+\SD{N^{-1/2}}.
	\end{align*}
Recalling $\psi_{N,0}(z_{N,i})=-z_{N,i}/\lambda_i(\vSigma)$, this yields part
(c) of the theorem.
\end{proof-of-theorem}

\section{Proofs for propagation of spiked eigenstructure in deep NNs}\label{sec:CK_spike_proof}

We next prove Theorems \ref{thm:no_outlier_ck} and \ref{thm:ck_spike}.
Appendix \ref{subsec:onelayer} first
establishes these results for a one-hidden-layer NN, $L=1$. We then apply this
result for $L=1$ inductively in Appendix \ref{subsec:propagation} to obtain
these results for general $L$. Appendix \ref{subsec:GMM_proof} proves
Corollary \ref{cor:gmm}.

\subsection{Spike analysis for one-hidden-layer CK}\label{subsec:onelayer}

Consider the setup in Section~\ref{sec:spike_CK} with a single hidden
layer $L=1$. In this setting, let us simplify notation and denote
\[\vX=\vX_0, \qquad \vW=\vW_0, \qquad d=d_0, \qquad N=d_1,\]
\[\vY=\vX_1=\frac{1}{\sqrt{N}}\sigma (\vW\vX),
\qquad \vK=\vK_1=\vY^\top \vY.\]
We denote the rows of $\vW$ and columns of $\vX$ respectively by
\[\vw_i^\top \in \R^d \text{ for } i \in [N],
\qquad \vx_\alpha \in \R^d \text{ for } \alpha \in [n].\]
We write $\E_{\vw}$ for
the expectation over a standard Gaussian vector $\vw \sim \cN(0,\vI)$ in $\R^d$.

Note that for a sufficiently large constant $B>0$
(depending on $\supp{\nu}$ and $\lambda_1,\ldots,\lambda_r$),
Assumption \ref{assump:data} implies that the event
\begin{equation}\label{eq:goodevent}
\cE(\vX)=\Big\{\|\vX\|<B,\; |\vx_\alpha^\top \vx_\beta|<\tau_n
\text{ and } |\|\vx_\alpha\|_2-1|<\tau_n
\text{ for all } \alpha \neq \beta \in [n]\Big\}
\end{equation}
holds almost surely for all large $n$. We will use throughout this
section the following argument: Since $\vW \equiv \vW^{(n)}$ is independent
of $\vX \equiv \vX^{(n)}$, and $\cE(\vX^{(n)})$ holds for all large $n$ with
probability 1 over $\{\vX^{(n)}\}_{n=1}^\infty$, to prove any almost-sure
statement, it suffices to show that the statement holds with
probability 1 over $\{\vW^{(n)}\}_{n=1}^\infty$, for any deterministic
matrices $\{\vX^{(n)}\}_{n=1}^\infty$ satisfying $\cE(\vX^{(n)})$.
Thus, we assume in the remainder of this section that
$\vX$ is deterministic and satisfies $\cE(\vX)$
for all large $n$, and write $\E,\P$ for the
expectation and probability over only the random weight matrix $\vW$.

We will apply Theorem \ref{thm:spiked} to a centered version of $\vY$,
\[\vG:=\vY-\E \vY=\frac{1}{\sqrt{N}}[\vg_1,\ldots,\vg_N]^\top,
\qquad \vg_i^\top:=\sigma(\vw_i^\top {\vX})-\E_{\vw}[\sigma(\vw^\top {\vX})].\]
Note that these rows $\vg_i^\top$ are i.i.d.\ with mean $\mathbf{0}$ and covariance
\begin{equation}\label{eq:Sigma_CK}
	\vSigma:=\E_{\vw}[\sigma(\vw^\top \vX)^\top\sigma(\vw^\top
\vX)]-\E_{\vw}[\sigma(\vw^\top \vX)]^\top\E_{\vw}[\sigma(\vw^\top
\vX)]\in\R^{n\times n}.
\end{equation}

\begin{lemma}\label{lemm:expect_rows}
Suppose Assumptions \ref{assump:NNasymptotics}, \ref{assump:data}, and
\ref{assump:sigma} hold, with $L=1$ and deterministic $\vX$. Then
\[\|\E_{\vw}[\sigma(\vw^\top \vX)]\|_2 \to 0, \qquad \|\E\vY\| \to 0.\]
\end{lemma}
\begin{proof}
Denote $\xi \sim \cN(0,1)$. Applying $\E[\sigma(\xi)]=0$,
$\E[\sigma'(\xi)\xi]=\E[\sigma''(\xi)]=0$, and a
Taylor approximation of $\sigma$, for any $\alpha \in [n]$,
	\begin{align*}
	\E_{\vw}[\sigma(\vw^\top\vx_\alpha)]=~&\E[\sigma(\norm{\vx_\alpha}\xi)]-\E[\sigma(\xi)]\\
	= ~&
\E[\sigma'(\xi)\xi(\norm{\vx_\alpha}-1)]+\E[\sigma''(\eta)\xi^2(\norm{\vx_\alpha}-1)^2]=\E[\sigma''(\eta)\xi^2(\norm{\vx_\alpha}-1)^2]
	\end{align*}
for some $\eta$ between $\xi$ and $\norm{\vx_i}\xi$.
Then, applying $|\sigma''(x)| \leq \lambda_\sigma$ and the
$\tau_n$-orthonormality of $\vX$ under $\cE(\vX)$,
\[|\E_{\vw}[\sigma(\vw^\top\vx_\alpha)]| \leq \lambda_\sigma \tau_n^2.\]
This gives $\|\E_{\vw}[\sigma(\vw^\top\vX)]\|_2 \leq \lambda_\sigma\tau_n^2\sqrt{n} \to 0$,
so also
$\norm{\E \vY}=\norm{\frac{1}{\sqrt{N}}\bI_N \cdot
\E_{\vw}[\sigma(\vw^\top \vX)]} \to 0$.
\end{proof}

Next, we recall from \cite{wang2021deformed} 
an approximation of $\vSigma$ by the linearized matrix
\begin{equation}\label{def:Phi0}
	\vSigma_{\mathrm{lin}}:=b_\sigma^2\vX^\top \vX+(1-b_\sigma^2 )\vI_n
\end{equation}
in the operator norm.

\begin{lemma}\label{lemm:phi_0}
Suppose Assumptions \ref{assump:NNasymptotics}, \ref{assump:data}, and
\ref{assump:sigma} hold, with $L=1$ and deterministic $\vX$.
\[\|\vSigma-\vSigma_{\mathrm{lin}}\| \to 0.\] 
Consequently, ordering $\lambda_1(\vSigma),\ldots,\lambda_n(\vSigma)$
in the same order as
$\lambda_1(\vX^\top\vX),\ldots,\lambda_n(\vX^\top\vX)$,
\begin{equation}\label{eq:lambda_i_diff}
	\sup_{i \in [n]} \Big|b_\sigma^2\lambda_i(\vX^\top \vX)
+(1-b_\sigma^2)-\lambda_i(\vSigma)\Big| \to 0.
\end{equation}
\end{lemma}
\begin{proof}
Denote $\xi \sim \cN(0,1)$. Let $\zeta_k(\sigma)=\E[\sigma(\xi)h_k(\xi)]$ be
the $k$-th Hermite coefficient of $\sigma$, where $h_k(x)$ is the $k$-th
Hermite polynomial normalized so that $\E[h_k(\xi)^2]=1$. Note that by Gaussian
integration by parts and the assumption $\E[\sigma''(\xi)]=0$,
\[\zeta_1(\sigma)=\E[\xi \sigma(\xi)]
=\E[\sigma'(\xi)]=b_\sigma,
\qquad \sqrt{2}\,\zeta_2(\sigma)=\E[(\xi^2-1)\sigma(\xi)]
=\E[\xi\sigma'(\xi)]=\E[\sigma''(\xi)]=0.\]
Then by \cite[Lemma 5.2]{wang2021deformed} and the first statement
of Lemma~\ref{lemm:expect_rows}, we have
\[\|\vSigma_0-\vSigma\|
\leq \|\vSigma_0-\E_{\vw}[\sigma(\vw^\top \vX)^\top \sigma(\vw^\top \vX)]\|
+\|\E_{\vw}[\sigma(\vw^\top \vX)^\top \sigma(\vw^\top \vX)]-\vSigma\| \to
0\]
where
	\[{\vSigma}_0=\zeta_1(\sigma)^2 \vX^\top \vX+  \zeta_3(\sigma)^2(
\vX^\top  \vX)^{\odot 3}+(1-\zeta_1(\sigma)^2-\zeta_3(\sigma)^2)\vI_n.\]
(Here, examination of the proof of \cite[Lemma 5.2]{wang2021deformed} shows
that the condition $\sum_\alpha (\|\vx_\alpha\|_2-1)^2 \leq B^2$ for
$(\eps,B)$-orthonormality is not
used when $\zeta_2(\sigma)=0$, and the remaining conditions of
$(\eps,B)$-orthonormality hold under $\cE(\vX)$.)
The lemma then follows upon observing that under $\cE(\vX)$,
\begin{align*}
	\norm{( \vX^\top  \vX)^{\odot 3}-\vI_n}\le~& \norm{\diag(( \vX^\top  \vX)^{\odot 3}-\vI_n)} +\norm{\offdiag(\vX^\top  \vX)^{\odot 3}}_F\\
	\le ~& \max_{\alpha\in [n]} \left|\norm{
\vx_\alpha}^6-1\right|+n\cdot\max_{\alpha\neq\beta\in[n]}|\vx_\alpha^\top\vx_\beta|^3
\leq C(\tau_n+n\tau_n^3),
\end{align*}
so that $\norm{\vSigma_{\mathrm{lin}}-\vSigma_0}=\zeta_3(\sigma)^2
\norm{( \vX^\top  \vX)^{\odot 3}-\vI_n} \to 0$ when $\lim_{n \to \infty} \tau_n
\cdot n^{1/3}=0$.
\end{proof}

Theorem \ref{thm:spiked} will provide a characterization
of outlier eigenvalues of $\vK$ that are separated from
$\cS_1=\supp{\mu_1} \cup \{0\}$, which is different from $\supp{\mu_1}$ when
$\gamma_1<1$. For $\gamma_1<1$, we augment this
statement with a small-ball argument to bound the smallest eigenvalue of $\vK$,
using the following result of \cite[Theorem 2.1]{yaskov2016controlling}.

\begin{lemma}[\cite{yaskov2016controlling}]\label{lemma:smallball}
Let $\vG=\frac{1}{\sqrt{N}}[\vg_1,\ldots,\vg_N]^\top \in \R^{N \times n}$ where
the rows $\vg_i \in \R^n$ are i.i.d.\ and equal in law to $\vg \in \R^n$. Define
\[\vSigma=\E\vg\vg^\top,
\qquad c_{\vg}=\inf_{\vv \in \R^n:\|\vv\|_2=1} \E|\vg^\top \vv|,
\qquad L_{\vg}(\delta,\iota)=\sup_{\vPi:\rank{\vPi} \geq \iota n}
\P\Big[|\vPi\vg|^2 \leq \delta\,\rank{\vPi}\Big]\]
where the latter supremum is taken over all orthogonal projections
$\vPi \in \R^{n \times n}$ with rank at least $\iota \cdot n$.

Suppose $\lambda_{\max}(\vSigma) \leq 1$, $c(\vg) \geq c$, and $n/N \leq
y$ for some constants $c>0$ and $y \in (0,1)$. Then there exist constants $s_0,\iota>0$
depending only on $(c,y)$ such that for any $\delta \in (0,1)$ and $s>0$,
\[\P[\lambda_{\min}(\vG^\top \vG) \geq (s_0-L(\delta,\iota;\vg)-s)\delta]
\geq 1-2e^{-yns^2/2}.\]
\end{lemma}

\begin{lemma}\label{lemma:outliersat0}
Suppose Assumptions \ref{assump:NNasymptotics}, \ref{assump:data}, and
\ref{assump:sigma} hold, with $L=1$ and deterministic $\vX$. Let
$\vG=\vY-\E\vY$.
\begin{enumerate}[label=(\alph*)]
\item If $\gamma_1 \geq 1$, then $0 \in \supp{\mu_1}$.
\item If $\gamma_1<1$, then there is a constant $c>0$ such that
$\lambda_{\min}(\vG^\top \vG)>c$ almost surely for all large $n$.
\end{enumerate}
\end{lemma}
\begin{proof}
If $\gamma_1>1$ strictly, then by definition
\[\mu_1=\frac{1}{\gamma_1}\,\tilde \mu_1+\frac{\gamma_1-1}{\gamma_1}\,\delta_0\]
is a mixture of $\tilde \mu_1$ and a point mass at 0, so $0 \in \supp{\mu_1}$.
If $\gamma=1$, then $\mu_1=\tilde \mu_1$. In this case,
recall from Proposition \ref{prop:inv_z}
that $\supp{\tilde \mu_1}$ is characterized by the function
\[z(\tilde m)={-}\frac{1}{\tilde m}+\gamma_1
\int \frac{\lambda}{1+\lambda \tilde m}d\nu_0(\lambda).\]
When $\gamma_1=1$, we have for all $\tilde m \in (0,\infty)$ that
\[z(\tilde m)<0, \qquad
z'(\tilde m)=\frac{1}{\tilde m^2}-\int \frac{\lambda^2}{(1+\lambda \tilde m)^2}
d\nu(\lambda)>0,\]
so $z(\tilde m)$ increases from $-\infty$ to 0 over the positive line
$\tilde m \in (0,\infty)$. Suppose by contradiction that
$0 \notin \supp{\tilde \mu}$. Then by
Proposition \ref{prop:inv_z}, there must be a point $\tilde m \in \R \setminus
\cT$ where $z(\tilde m)=0$ and $z'(\tilde m)>0$ strictly, implying that
there is an open interval $(\tilde m_-,\tilde m_+) \ni \tilde m$
on which $z(\cdot)$ increases from $z(\tilde m_-)<0$ to $z(\tilde m_+)>0$.
We must have $\tilde m<0$ by the above behavior of $z(\cdot)$ on
$(0,\infty)$, and the range $[z(\tilde m_-),z(\tilde m_+)]$ must overlap with
$[z(a),z(b)]$ for some sufficiently large $a,b \in (0,\infty)$.
But this contradicts the non-intersecting property
shown in \cite[Theorem 4.4]{silverstein1995analysis}. So
also in this case $0 \in \supp{\tilde \mu_1}=\supp{\mu_1}$, showing part (a).

For part (b), we apply Lemma \ref{lemma:smallball}. Under $\cE(\vX)$,
the condition $b_\sigma \neq 0$ implies
$c_0<\lambda_{\min}(\vSigma_{\mathrm{lin}}) \leq
\lambda_{\max}(\vSigma_{\mathrm{lin}})<C_0$ for some constants $C_0,c_0>0$.
Hence also
\begin{equation}\label{eq:Sigmabounds}
c_0<\lambda_{\min}(\vSigma) \leq \lambda_{\max}(\vSigma)<C_0
\end{equation}
for all large $n$, by Lemma \ref{lemm:phi_0}.
We may assume without loss of generality that
$\lambda_{\max}(\vSigma) \leq 1$ as needed in Lemma \ref{lemma:smallball};
otherwise, the following argument may be applied to a rescaling of $\vSigma$
and $\vG$.

To lower bound $c_{\vg}$ in Lemma \ref{lemma:smallball}, observe that
for any unit vector $\vv \in \R^n$ we have
\[\E[(\vg^\top \vv)^2]=\vv^\top \vSigma\vv>c_0.\]
Viewing $F(\vw)=\vg^\top \vv=\sigma(\vw^\top \vX)\vv=\sum_{\alpha=1}^n v_\alpha
\sigma(\vw^\top \vx_\alpha)$ as a function of $\vw  \sim \cN(0,\vI)$, we have
$\nabla F(\vw)=\sum_{\alpha=1}^n v_\alpha \sigma'(\vw^\top \vx_\alpha)
\vx_\alpha=\vX(\vv \odot \sigma'(\vw^\top \vX))$ where
$\sigma'(\cdot)$ is applied coordinatewise and $\odot$ is the coordinatewise
product. Then, applying $|\sigma'(x)| \leq \lambda_\sigma$, observe that
$\|\nabla F(\vw)\|_2 \leq \|\vX\| \cdot \|\vv \odot \sigma'(\vw^\top \vX)\|_2
\leq \lambda_\sigma \|\vX\|$, so $F(\vw)$ is $C$-Lipschitz in $\vw$
for a constant $C>0$ (not depending on $\vv$) on the event $\cE(\vX)$.
This implies by Gaussian
concentration-of-measure that $\vg^\top \vv$ is sub-gaussian,
i.e.\ for some constants $C,c>0$ and any $t>0$,
$\P[|\vg^\top \vv| \geq t] \leq Ce^{-ct^2}$.
Integrating this tail bound, for some constant $t>0$ sufficiently large,
we have $\E[(\vg^\top \vv)^2\bI_{\{|\vg^\top \vv|>t\}}]
\leq c_0/2$, and hence
\[c_0<\E[(\vg^\top \vv)^2]
\leq \E[(\vg^\top \vv)^2\bI_{\{|\vg^\top \vv| \leq t\}}]+\frac{c_0}{2}
\leq t \cdot \E|\vg^\top \vv|+\frac{c_0}{2}.\]
So $\E|\vg^\top \vv| \geq c_0/(2t)$, and hence $c_{\vg} \geq c_0/(2t)>c$ for a
constant $c>0$.

Now let $s_0,\iota>0$ be the constants depending on $(c,\gamma_1)$ in the
statement of Lemma \ref{lemma:smallball}.
By the nonlinear Hanson-Wright inequality of
\cite[Eq.\ (3.3)]{wang2021deformed}, for any orthogonal projection $\vPi \in
\R^{n \times n}$, any $t>0$, and some constant $c>0$, we have
\[\P\big[|\vg^\top \vPi \vg-\E\vg^\top \vPi \vg|>t\big] \leq
2e^{-c\min(t^2/\|\vPi\|_F^2,t/\|\vPi\|)}.\]
Here $\E \vg^\top \vPi \vg=\Tr \vPi\vSigma>c_0\rank{\vPi}$,
$\|\vPi\|_F^2=\rank{\vPi}$, and $\|\vPi\|=1$,
so applying this with $t=(c_0/2)\rank{\vPi}$ yields
\[\P[|\vPi\vg|^2 \leq (c_0/2)\rank{\vPi}] \leq 2e^{-c'\rank{\vPi}}.\]
Then, choosing $\delta=c_0/2$, we get $L_{\vg}(\delta,\iota) \to 0$ as $n \to
\infty$. Then Lemma \ref{lemma:smallball} implies $\lambda_{\min}(\vG^\top
\vG)>s_0\delta/2$ almost surely for all large $n$, as desired.
\end{proof}

The following is the main result of this section, showing that Theorems
\ref{thm:no_outlier_ck} and \ref{thm:ck_spike} hold in this setting of $L=1$.

\begin{lemma}\label{lemma:onelayer}
Theorems \ref{thm:no_outlier_ck} and
\ref{thm:ck_spike} hold for a single layer $L=1$.
Furthermore, $\vY$ is $C\tau_n$-orthonormal for some constant $C>0$, 
almost surely for all large $n$.
\end{lemma}
\begin{proof}
We condition on $\vX$ as discussed at the start of this section, and
apply Theorem \ref{thm:spiked} to the centered matrix
$\vG=\vY-\E\vY=\frac{1}{\sqrt{N}}[\vg_1,\ldots,\vg_N]^\top$.
Let us verify Assumption~\ref{assump:G} for $\vG$:
We have shown Assumption~\ref{assump:G}(a) in \eqref{eq:Sigmabounds}.
The rows $\vg_i$ are sub-gaussian as shown in the above proof of
Lemma \ref{lemma:outliersat0}(b), so Assumption~\ref{assump:G}(b) holds by
\cite[Eq.\ (5.26)]{vershynin2010introduction}, and
Assumption \ref{assump:G}(d) holds by \cite[Lemma 2]{jin2019short}.
The nonlinear Hanson-Wright inequality of
\cite[Eq.\ (3.3)]{wang2021deformed} implies
\[|\vg_i^\top\vA\vg_i-\Tr\vA\vSigma| \prec \|\vA\|_F\]
uniformly over $i \in [N]$ and deterministic matrices
$\vA \in \C^{n \times n}$. Furthermore, it is clear from the argument preceding
\cite[Eq.\ (3.3)]{wang2021deformed} that for any $i \neq j \in [N]$, the joint
vector $(\vg_i,\vg_j) \in \R^{2n}$ also satisfies Lipschitz concentration, hence
		\[\left|\begin{pmatrix}
			\vg_i^\top & \vg_j^\top
		\end{pmatrix} \vB
		\begin{pmatrix}
			\vg_i\\
			\vg_j
		\end{pmatrix}
-\Tr \vB\begin{pmatrix} \vSigma & \mathbf{0} \\
\mathbf{0} & \vSigma \end{pmatrix} \right| \prec \|\vB\|_F\]
uniformly over $i \neq j \in [N]$ and deterministic matrices $\vB \in \C^{2n
\times 2n}$. Applying this with
\[\vB=\begin{pmatrix} \mathbf{0} &\vA\\ \vA &\mathbf{0} \end{pmatrix}\]
verifies both statements of Assumption \ref{assump:G}(c).

Next, we check Assumption~\ref{assum:spike} for the population
covariance matrix $\vSigma$.
Combining Assumption~\ref{assump:data} and \eqref{eq:lambda_i_diff} from
Lemma~\ref{lemm:phi_0}, we have
\begin{align}
\frac{1}{n-r}\sum_{i=r+1}^n \delta_{\lambda_i(\vSigma)}
&\to \nu_0:=b_\sigma^2 \otimes \mu_0 \oplus (1-b_\sigma^2) \text{ weakly},
\label{eq:onelayernu}\\
\lambda_i(\vSigma) &\to -\frac{1}{s_{i,0}}:=b_\sigma^2\lambda_i+(1-b_\sigma^2)
\not\in\supp{\nu_0} \text{ for } i=1,\ldots,r.
\label{eq:onelayersi}
\end{align}
Here, the statement $-1/s_{i,0} \notin \supp{\nu_0}$ in \eqref{eq:onelayersi}
follows from the assumptions
$\lambda_i \notin \supp{\mu_0}$ and $b_\sigma \neq 0$.
This then implies by the definition of $\cT_1$
that $s_{i,0} \in \R \setminus \cT_1$, as claimed in Theorem \ref{thm:ck_spike}(a).
Furthermore, for any fixed $\eps>0$ and all large $n$,
Assumption~\ref{assump:data} and \eqref{eq:lambda_i_diff} imply also that
	\[\lambda_i(\vSigma) \in \supp{\nu_0}+(-\eps,\eps) \text{ for all }
	i \geq r+1.\]
Thus Assumption~\ref{assum:spike} holds for $\vSigma$ as $n \to \infty$.

Then we can apply Theorems~\ref{thm:deterministic_equ} and \ref{thm:spiked}
for $\bar\vK:=\vG^\top\vG$. The Stieltjes transform approximation in Theorem
\ref{thm:deterministic_equ} and Lemma \ref{lemma:supportcompareb}(c) together
imply $m_{\bar\vK}(z) \to m_1(z)$ almost surely for each fixed $z \in \C^+$,
where $m_1(z)$ is the Stieltjes transform of the measure
$\mu_1=\rho_{\gamma_1}^\MP \boxtimes \nu_0$. This implies the weak convergence
\begin{equation}\label{eq:Kbarweakconvergence}
\frac{1}{n}\sum_{i=1}^n \delta_{\lambda_i(\bar\vK)} \to \mu_1 \text{ a.s.}
\end{equation}
Theorem~\ref{thm:spiked}(a,b) further justifies:
	\begin{itemize}
\item Let $z_1(\cdot)$ and $\cI_1$ be defined by \eqref{eq:zell} and
\eqref{eq:index_set} with $\ell=1$. Then for any sufficiently small
constant $\eps>0$, almost surely for all large $n$,
there is a 1-to-1 correspondence between the eigenvalues of $\bar\vK$ outside
$\cS_1+(-\eps,\eps)$ and $\{i:i \in \cI_1\}$. Furthermore,
for each $i \in \cI_1$,
\begin{equation}\label{eq:Kbareigval}
\lambda_i(\bar\vK) \to z_1(s_{i,0})>0.
\end{equation}
almost surely as $n \to \infty$.

\item Let $\varphi_1(\cdot)$ be defined by \eqref{eq:zell} with $\ell=1$.
For each $i \in \cI_1$, let $\vv_i(\bar\vK) \in \R^n$
be a unit-norm eigenvector of $\bar\vK$
corresponding to $\lambda_i(\bar\vK)$, and for each $j \in [r]$,
let $\vv_j(\vSigma)$ be a unit-norm eigenvector of $\vSigma$
corresponding to $\lambda_j(\vSigma)$. Then
almost surely as $n \to \infty$, for each $i \in \cI_1$ and $j \in [r]$,
\begin{equation}\label{eq:Kbareigvec1}
|\vv_j(\vSigma)^\top \vv_i(\bar\vK)| \to \sqrt{\varphi_1(s_{i,0})} \cdot \bI\{i=j\}
\end{equation}
where $\varphi_1(s_{i,0})>0$. Moreover,
letting $\vv\in\R^n$ be any unit vector independent of $\vW$, almost surely
\begin{equation}\label{eq:Kbareigvec2}
|\vv^\top \vv_i(\bar\vK)|-\sqrt{\varphi_1(s_{i,0})}
\cdot|\vv^\top\vv_i(\vSigma)| \to 0.
\end{equation}
\end{itemize}

If $\gamma_1 \geq 1$, then Lemma \ref{lemma:outliersat0}(a) shows that
$\supp{\mu_1}=\supp{\mu_1} \cup \{0\}=\cS_1$. If $\gamma_1<1$, then
Lemma \ref{lemma:outliersat0}(b) shows that for any sufficiently small constant
$\eps>0$, $\bar\vK$ has no eigenvalues in $[0,\eps)$ almost surely
for all large $n$. Thus, in both cases, the first statement above in fact
establishes a 1-to-1 correspondence between $\{i:i \in \cI_1\}$ and
all eigenvalues of $\bar\vK$ outside $\supp{\mu_1}+(-\eps,\eps)$, almost surely
for all large $n$.

To translate these statements to the non-centered matrix
$\vK=\vY^\top\vY$, recall from Lemma~\ref{lemm:expect_rows} that $\|\vY-\vG\|
\to 0$ almost surely, and from Assumption \ref{assump:G}(b) verified above that 
$\bI\{\|\vG^\top \vG\|>B'\} \prec 0$ for a constant $B'>0$.
Then, almost surely as $n \to \infty$,
	\[\norm{\vK-\bar\vK} \to 0.\]
Therefore, by Weyl's inequality and \eqref{eq:Kbarweakconvergence}, the
empirical eigenvalue distribution $\widehat \mu_1$ of $\vK$ converges also to
$\mu_1$ weakly a.s., as claimed in Theorem \ref{thm:no_outlier_ck}.
Furthermore, by \eqref{eq:Kbareigval}, almost surely for all
large $n$, the eigenvalues $\widehat\lambda_{i,1}$ of $\vK$ outside
$\supp{\mu_1}+(-\eps,\eps)$ are also in 1-to-1
correspondence with $\{i:i \in \cI_1\}$, where
$\widehat\lambda_{i,1} \to z_1(s_{i,0})$ for each $i \in \cI_1$. In particular,
if $r=0$, then also $|\cI_1|=0$, so $\vK$ has no eigenvalues outside
$\supp{\mu_1}+(-\eps,\eps)$. This proves Theorem \ref{thm:no_outlier_ck}.

For each $i \in \cI_1$, let $\widehat \vv_{i,1} \in \R^n$ be a
	unit-norm eigenvector of $\vK$ 
	corresponding to $\widehat \lambda_{i,1}$.
Then by the Davis-Kahan Theorem \cite{davis1970rotation},
we may choose a sign for $\widehat \vv_{i,1}$ such that
	\[\norm{\widehat\vv_{i,1}-\vv_i(\bar\vK)} \le
\frac{\sqrt{2}\|\vK-\bar{\vK}\|}{\dist{
\widehat{\lambda}_{i,1},\spec{(\bar\vK)}\setminus \{\lambda_i(\bar\vK)\}}}.\]
We note that $\widehat{\lambda}_{i,1} \to z_1(s_{i,0})$ a.s., which is distinct
from the limit values $\{z_1(s_{j,0}):j \in \cI_1 \setminus \{i\}\}$ of
$\{\lambda_j(\bar \vK):\cI_1 \setminus \{i\}\}$
by bijectivity of the map $z_1(\cdot)$ in Proposition
\ref{prop:inv_z}. Furthermore $z_1(s_{i,0})$ falls outside
$\supp{\mu_1}+(-\eps,\eps)$ for sufficiently small $\eps>0$,
which contains all other eigenvalues of $\bar \vK$.
Thus $\dist{\widehat{\lambda}_{i,1},\spec{(\bar\vK)}\setminus
\{\lambda_i(\bar\vK)\}} \geq c$ for a constant $c>0$ almost surely for
all large $n$, so
\[\norm{\widehat\vv_{i,1}-\vv_i(\bar\vK)} \to 0 \text{ a.s.}\]
Similarly, by the convergence
$\|\vSigma-\vSigma_{\mathrm{lin}}\| \to 0$ and the assumption $b_\sigma \neq 0$,
we have $\norm{\vv_j(\vSigma)-\vv_j} \to 0$ for each $j \in [r]$, where $\vv_j$
is the unit-norm eigenvector of $\vSigma_{\mathrm{lin}}$ corresponding to
its eigenvalue
$b_\sigma^2\lambda_j(\vX^\top \vX)+(1-b_\sigma^2)$, i.e.\ the
eigenvector of $\vX^\top \vX$ corresponding to $\lambda_j(\vX^\top \vX)$.
Then \eqref{eq:Kbareigvec1} and \eqref{eq:Kbareigvec2} imply also
\[|\vv_j^\top \widehat \vv_i|^2 \to \varphi_1(s_{i,0}) \cdot \bI\{i=j\},
\qquad |\vv^\top \widehat \vv_i|^2-\varphi_1(s_{i,0}) \cdot|\vv^\top\vv_i|^2 \to 0.\]
This shows all claims of Theorem \ref{thm:ck_spike} for $L=1$.

Finally, on $\cE(\vX)$, the matrix $\vY$ is $C\tau_n$-orthonormal for a
constant $C>0$ by \cite[Lemma D.3(b)]{fan2020spectra}.
(The proof of \cite[Lemma D.3(b)]{fan2020spectra} again does not use
the condition $\sum_\alpha (\|\vx_\alpha\|_2^2-1)^2 \leq B^2$ of
$(\eps,B)$-orthonormality therein, and the remaining conditions of
$(\eps,B)$-orthonormality hold under $\cE(\vX)$.)
This shows the last claim of the lemma.
\end{proof}

\subsection{Spike analysis for multiple layers}\label{subsec:propagation}

We now prove Theorem~\ref{thm:ck_spike} by inductively applying the result for
$L=1$ through multiple layers. We follow the notations of
Section~\ref{sec:spike_CK}.\\

\begin{proof-of-theorem}[\ref{thm:ck_spike}]
Suppose inductively that Assumption \ref{assump:data}
holds with $\vX_{\ell-1}$ in place of $\vX_0$, and all
conclusions of Theorem \ref{thm:ck_spike} hold for $\vK_\ell$. The base case of
$\ell=1$ follows from Lemma \ref{lemma:onelayer}.

Then the last statement of Lemma \ref{lemma:onelayer} implies that
$\vX_\ell$ is $\tau_n'$-orthonormal almost surely for all large $n$,
for some $\tau_n'$ satisfying $\tau_n' \cdot n^{1/3} \to 0$.
Furthermore, the conclusions of Theorem \ref{thm:ck_spike}(b,c) for $\vK_\ell$
imply that statements (a) and (b) of Assumption \ref{assump:data} also
hold for $\vX_\ell$, in the following sense: Let $r_\ell=|\cI_\ell|$. Then
\[\frac{1}{n-|r_\ell|}\sum_{i \notin \cI_\ell} \delta_{\lambda_i(\vX_\ell^\top
\vX_\ell)} \to \mu_\ell \text{ weakly a.s.}\]
For any fixed $\eps>0$, almost surely for all large $n$,
$\widehat \lambda_{i,\ell}:=
\lambda_i(\vX_\ell^\top \vX_\ell) \in \supp{\mu_\ell}+(-\eps,\eps)$
for all $i \notin \cI_\ell$. Furthermore, for each $i \in \cI_\ell$,
$\widehat \lambda_{i,\ell} \to z_\ell(s_{i,\ell-1}) \notin
\supp{\mu_\ell}$.

Then we may apply Lemma \ref{lemma:onelayer} with input data
$\vX=\vX_\ell$ in place of $\vX_0$. This shows that for any fixed $\eps>0$ and all large $n$, there is a 1-to-1 correspondence between the eigenvalues $\widehat
\lambda_{i,\ell+1}$ of $\vK_{\ell+1}$ outside
$\supp{\mu_{\ell+1}}+(-\eps,\eps)$ and $\{i:i\in\cI_{\ell+1}\}$,
where
$\widehat \lambda_{i,\ell+1}\to z_{\ell+1}\left(s_{i,\ell}\right)>0$ a.s.,
and $s_{i,\ell} \in \R \setminus \cT_{\ell+1}$.
Moreover, for any unit vector $\vv \in \R^n$ independent of
$\vW_1,\ldots,\vW_{\ell+1}$,
	\[|\widehat \vv_{i,\ell+1}^\top
\vv|^2-\varphi_{\ell+1}(s_{i,\ell})\cdot|\widehat \vv_{i,\ell}^\top \vv|^2\to 0,\]
where also $\varphi_{\ell+1}(s_{i,\ell})>0$.
Then by the induction hypothesis for $|\widehat \vv_{i,\ell}^\top \vv|^2$,
	\begin{align*}
	    |\widehat \vv_{i,\ell+1}^\top \vv|^2\to ~& \prod_{k=1}^{\ell+1}
\varphi_{k}(s_{i,k-1}) \cdot |\vv_{i}^\top \vv |^2,
	\end{align*}
and specializing to $\vv=\vv_j$ for $j \in [r]$ gives
	\[|\widehat \vv_{i,\ell+1}^\top \vv_j |^2\to \prod_{k=1}^{\ell+1}
\varphi_{k}(s_{i,k-1}) \cdot \bI\{i=j\}.\]
This verifies all conclusions of Theorem \ref{thm:ck_spike}
for $\vK_{\ell+1}$, completing the induction.
\end{proof-of-theorem}

\subsection{Corollary for signal-plus-noise input data}\label{subsec:GMM_proof}

\begin{proof-of-coro}[\ref{cor:gmm}]
It is shown in \cite[Section 3.1]{benaych2012singular} that asymptotically as
$d,n \to \infty$ with $n/d \to \gamma_0$, the data matrix $\vX$ has a spike
singular value corresponding to $\theta_i$ if and only if
$\theta_i>\gamma_0^{1/4}$, in which case
\[\lambda_i(\vK_0) \to
\lambda_i:=\frac{(1+\theta_i^2)(\gamma_0+\theta_i^2)}{\theta_i^2},
\qquad |\vb_i^\top \vv_i|^2 \to
1-\frac{\gamma_0(1+\theta_i^2)}{\theta_i^2(\theta_i^2+\gamma_0)}\]
where $\vv_i$ is the unit eigenvector of the input Gram matrix
$\vK_0=\vX^\top\vX$. Thus claims (a) and (b) of Assumption \ref{assump:data}
hold with $r=|\{i:\theta_i>\gamma_0^{1/4}\}|$,
$\mu_0=\rho_{\gamma_0}^\MP$ being the standard Marcenko-Pastur law, and
$\lambda_i=(1+\theta_i^2)(\gamma_0+\theta_i^2)/\theta_i^2$
being the above values.

We note that $\vX$ is $n^{-1/2+\eps}$-orthonormal for any $\eps>0$
almost surely for all large $n$, by the given condition
$\max_{i=1}^r \|\vb_i\|_\infty<n^{-1/2+\eps}$ and the bounds, for any
$\alpha,\beta \in [n]$,
\begin{align*}
\|\vx_\alpha\|_2&=\left\|\vz_\alpha\right\|_2
+\sum_{i=1}^r \SD{\|\va_i\|_2 |\theta_i| |b_{i,\alpha}|}
=\left\|\vz_\alpha\right\|_2+\SD{n^{-1/2+\eps}}=1+\SD{n^{-1/2+\eps}},\\
\vx_\alpha^\top \vx_\beta&=\vz_\alpha^\top \vz_\beta
+\sum_{i=1}^r \SD{|\theta_i| \Big(
|\va_i^\top \vz_\alpha| |b_{i,\alpha}|
+|\va_i^\top \vz_\beta| |b_{i,\beta}|\Big)
+\theta_i^2 \|\va_i\|_2^2 |b_{i,\alpha}b_{i,\beta}|}
=\SD{n^{-1/2+\eps}}.
\end{align*}
Hence Theorem \ref{thm:ck_spike} applies, showing
that $\vK_\ell$ has an outlier eigenvalue corresponding to each input
signal $\theta_i$ if and only if $\theta_i>\gamma_0^{1/4}$ and
$i \in \cI_\ell$. The statement \eqref{eq:gmmalignment} follows from 
Theorem \ref{thm:ck_spike}(c) applied with $\vv=\vb_i$.
\end{proof-of-coro}

\section{Proofs for spiked eigenstructure of the trained CK}\label{sec:trained_CK_proof}

In this section, we prove Theorem~\ref{thm:gd_spike}. The proof is an
application of Theorem \ref{thm:spiked} as in the one-layer setting of the
preceding section, but now reversing the roles of $\vX$ and $\vW$. We abbreviate
\[\vW=\vW_{\text{trained}}, \qquad
\vY=\frac{1}{\sqrt{N}}\sigma(\tilde{\vX}\vW), \qquad
\vK=\vY\vY^\top\]
where $\vK$ is the CK matrix of interest. In contrast to the preceding section,
the theorem requires characterizing the \emph{left} spike singular vector of
$\vY$, and we will do so using Theorem \ref{thm:spiked}(c).

We first recall the following approximation of $\vW$ from \cite{ba2022high}.
\begin{prop} \label{prop:W_spike}
Under Assumption \ref{assum:GD}, set $\eta=\sum_{t=1}^T \eta_t$, and let
$\theta_1,\theta_2$ be as defined in \eqref{def:theta_12}. Then
	\begin{align}
\|\vW-\widetilde \vW\| \prec N^{-1/2} \quad \text{ where } \quad
		\widetilde\vW=\vW_0+\frac{b_\sigma\eta}{n}\vX^\top\vy\va^\top.\label{eq:tilde_W}
	\end{align}
The largest singular value $s_{\max}(\vW)$ falls outside the limit of its
empirical singular value distribution if and only if
$\theta_1>\gamma_0^{1/4}$, in which case $s_{\max}(\vW)$ 
and its unit-norm left singular vector $\vu(\vW)$ satisfy
	\begin{equation}\label{eq:trainedWspikes}
			s_{\max}(\vW)\to s_1:=
\sqrt{\frac{(1+\theta_1^2)(\gamma_0+\theta_1^2)}{\theta_1^2}}, \quad
		|\vu(\vW)^\top
\vbeta_*|^2\to\frac{\theta_2^2}{\theta_1^2}\left(1-\frac{\gamma_0+\theta_1^2}{\theta_1^2(\theta_1^2+1)}\right)
\quad \text{a.s.}
\end{equation}
\end{prop}
\begin{proof}
Notice that each gradient update matrix
$\vG_t$ of \eqref{eq:trained_W} takes the form
\begin{align*} 
	\vG_t=
   \frac{1}{n} \vX^\top \left[\left(\frac{1}{\sqrt{N}}\left(\vy -
\frac{1}{\sqrt{N}}\sigma(\vX\vW_{\!t})\va\right)\va^\top\right) \odot
\sigma'(\vX\vW_{\!t})\right], 
\end{align*}  
From the proof of \cite[Lemma 16]{ba2022high}, for each $t=1,\ldots,T$, this
matrix $\vG_t$ satisfies the same rank-one approximation
\[\Big\|\sqrt{N}\,\vG_t-\frac{b_\sigma}{n}\vX^\top\vy\va^\top\Big\|
\prec N^{-1/2}.\]
This implies \eqref{eq:tilde_W} in light of \eqref{eq:trained_W},
and the statements of \eqref{eq:trainedWspikes} then
follow from  \cite[Theorem 3]{ba2022high}.
\end{proof}

We denote the columns of $\vW \equiv
\vW_{\text{trained}} \in \R^{d \times N}$
and of the initialization $\vW_0 \in \R^{d \times N}$ by
\[\vw_i \in \R^d,\; \vw_{i,0} \in \R^d \text{ for } i \in [N]\]
respectively. Fixing a large constant $B>0$ and small constant $\eps>0$,
define the event
\[\cE(\vW)=\Big\{\|\vW\|<B,\,|\vw_i^\top \vw_j|<n^{-1/2+\eps}
\text{ and } |\|\vw_i\|_2-1|<n^{-1/2+\eps} \text{ for all } i \neq j \in
[N]\Big\}.\]

\begin{lemma}\label{lemma:event_Wt}
Under Assumption \ref{assum:GD}, for some sufficiently large constant
$B>0$ and any fixed $\eps>0$, $\cE(\vW)$ holds almost surely for all large $n$ and any fixed $T\in\N$.
\end{lemma}
\begin{proof}
By the assumption $[\vW_0]_{ij} \overset{iid}{\sim} \cN(0,1/d)$, it is
immediate to check that $\cE(\vW_0)$ holds almost surely for all large $n$.
To show that $\cE(\vW)$ holds, we apply the approximation \eqref{eq:tilde_W}.
Here, under Assumption \ref{assum:GD},
we have by standard tail bounds for Gaussian vectors and matrices that
\[\bI\{\|\vX^\top \vy\|>Cn\}
\leq \bI\{\|\vX\| \cdot (\lambda_{\sigma_*}\|\vX\vbeta_*\|+\|\vvarepsilon\|_2)>Cn\} \prec 0,
\qquad \bI\{\|\va\|_2>C\} \prec 0\]
for a sufficiently large constant $C>0$, and also
$\|\va\|_\infty \prec N^{-1/2}$. Then this implies
\begin{align}
\max_{i=1}^N \|\vw_i-\vw_{i,0}\|_2
&\leq \max_{i=1}^N \|\widetilde \vw_i-\vw_{i,0}\|_2
+\|\vW-\widetilde \vW\| \prec N^{-1/2}\label{eq:W_12_norm}
\end{align}
and $\bI\{\|\vW-\vW_0\|>C'\} \prec 0$ for a constant $C'>0$. Then $\cE(\vW)$
also holds almost surely for all large $n$, as claimed.
\end{proof}

Analogous to the argument of Appendix \ref{subsec:onelayer},
we may now condition on $\vW$, i.e.\ we assume that $\vW$ is deterministic and
satisfies $\cE(\vW)$ for all large $n$, and we
write $\E$ for the expectation
over only the randomness of the new data $(\tilde \vX,\tilde \vy)$.
Defining
\begin{equation}\label{eq:udef}
\vG=\sqrt{\frac{N}{n}}(\vY-\E\vY) \in \R^{n \times N},
\quad \vu=\frac{1}{\sqrt{n}}\tilde\vy \in \R^n
\quad \text{ where } \quad
\vY=\frac{1}{\sqrt{N}}\sigma(\tilde \vX\vW),
\end{equation}
observe that $[\vu,\vG] \in \R^{n \times (N+1)}$ has centered i.i.d.\ rows with
respect to the randomness of $(\tilde \vX,\tilde\vy)$.
We will write $\E_{\vx}$ for the expectation with respect to a standard Gaussian
vector $\vx \sim \cN(\mathbf{0},\vI_d)$.

\begin{lemma}\label{lemma:bar_sigma_app}
Suppose $\vW$ satisfies $\cE(\vW)$ for all large $n$. Then
\begin{align}
\|\E_{\vx}[\sigma(\vx^\top\vW)]\|_2 \to 0,
\qquad \|\E\vY\| \to 0,
\qquad \norm{\vSigma-\vSigma_{\mathrm{lin}}} \to 0
\end{align}
        where
\begin{align}
	\vSigma:=~&\E_{\vx}[\sigma(\vx^\top\vW)^\top\sigma(\vx^\top\vW)]-\E_{\vx}[\sigma(\vx^\top\vW)]^\top\E_{\vx}[\sigma(\vx^\top\vW)]\\
	\vSigma_{\mathrm{lin}}:=~& b_\sigma^2(\vW^\top \vW)+(1-b_\sigma^2)\vI_N.
\end{align}
\end{lemma}
\begin{proof}
The proof is the same as Lemmas~\ref{lemm:expect_rows} and~\ref{lemm:phi_0}.
\end{proof}

\begin{proof-of-theorem}[\ref{thm:gd_spike}]
We condition on $\vW$ satisfying $\cE(\vW)$ for all large $n$, and we
apply Theorem~\ref{thm:spiked}(c) for $[\vu,\vG] \in \R^{(n+1)\times N}$
(exchanging $n$ and $N$). It may be checked that
Assumption~\ref{assump:G} holds for $[\vu,\vG]$ by the same argument as
in Lemma~\ref{lemma:onelayer}.

By the convergence $\|\vSigma-\vSigma_{\mathrm{lin}}\| \to 0$ in
Lemma \ref{lemma:bar_sigma_app} and Proposition \ref{prop:W_spike},
if $\theta_1>\gamma_0^{1/4}$, then
Assumption \ref{assum:spike} holds for $\vSigma$ with $r=1$ and
\[\nu=b_\sigma^2\otimes\rho^\MP_{\gamma_0}\oplus(1-b_\sigma^2), \qquad
\lambda_1=b_\sigma^2\frac{(1+\theta_1^2)(\gamma_0+\theta_1^2)}{\theta_1^2}
+(1-b_\sigma^2) \notin \supp{\nu},\]
where $\rho_{\gamma_0}^\MP$ is the standard Marcenko-Pastur limit
for the empirical eigenvalue distribution of $\vW^\top \vW$, hence $\nu$ is
the limit empirical eigenvalue distribution of $\vSigma$, and $\lambda_1$ is the
limit of $\lambda_{\max}(\vSigma)$. If instead $\theta_1 \leq \gamma_0^{1/4}$,
then Assumption \ref{assum:spike} holds with $r=0$.

Then Theorem \ref{thm:spiked}(a,c)
characterizes the outlier eigenvalue and eigenvector of $\vG\vG^\top$, showing:
\begin{itemize}
	\item $\vG\vG^\top$ has a spike eigenvalue if and only if
$\theta_1>\gamma_0^{1/4}$ and $z'(-1/\lambda_1)>0$, where $z(\cdot)$ is defined 
by \eqref{eq:inv_z} with $\gamma=\gamma_1$ and the measure $\nu$ given above.
In this case, $\lambda_{\max}(\vG\vG^\top)\to z(-1/\lambda_1)$ almost surely.
	\item When $\theta_1>\gamma_0^{1/4}$ and $z'(-1/\lambda_1)>0$,
letting $\vu(\vG),\vv(\vSigma)$ be the leading unit-norm left singular vector
	of $\vG$ and leading unit-norm eigenvector of $\vSigma$, almost surely
			\[|\vu^\top \vu(\vG)|-
	\frac{\sqrt{z(-1/\lambda_1)\varphi({-}1/\lambda_1)}}{\lambda_1}
	\cdot\left|\E_{\vx}[\sigma_*(\vbeta_*^\top\vx)\sigma(\vx^\top\vW)]\vv(\vSigma)
	\right| \to 0.\]
where $\varphi(\cdot)$ is defined by \eqref{eq:phi_limit} also
with $\gamma=\gamma_1$ and the above measure $\nu$.
\end{itemize}

By an application of Weyl's inequality and  the Davis-Kahan Theorem as in the proof of
Theorem~\ref{thm:ck_spike}, this implies for $\vK=\vY\vY^\top$ that if
$\theta_1>\gamma_0^{1/4}$ and $z'(-1/\lambda_1)>0$, then its leading
eigenvalue $\lambda_{\max}(\vK)$ and unit eigenvector $\widehat \vu$ satisfy
\begin{equation}\label{eq:learnedKspike}
	\begin{aligned}
		&\lambda_{\max}(\vK)\to \gamma_1^{-1}z(-1/\lambda_1),\\
		&|\vu^\top \widehat\vu|-
			\frac{\sqrt{z(-1/\lambda_1)\varphi({-}1/\lambda_1)}}{\lambda_1}
			\cdot\left|\E_{\vx}[\sigma_*(\vbeta_*^\top\vx)\sigma(\vx^\top\vW)]\vv(\vW)\right|
			\to 0,
	\end{aligned}
\end{equation}
where $\vv(\vW)$ is the leading unit eigenvector of $\vSigma_{\text{lin}}$,
i.e.\ the leading right singular vector of $\vW$.
If $\theta_1 \leq \gamma_0^{1/4}$ or $z'(-1/\lambda_1) \leq 0$, then all
eigenvalues of $\vK$ converge to the support of its limiting empirical
eigenvalue law.

Finally, in the case of $\theta_1>\gamma_0^{1/4}$ and $z'(-1/\lambda_1)>0$,
we may conclude the proof by showing
\begin{equation}\label{eq:align_y_limit}
	\left\|\E_{\vx}[\sigma_*(\vbeta_*^\top\vx)\sigma(\vx^\top\vW)]
	-b_\sigma b_{\sigma_*} \vbeta_*^\top\vW\right\|_2 \to 0 \text{ a.s.}
\end{equation}
For each column $i\in [N]$, we have from \eqref{eq:W_12_norm} that
$\|\vw_i-\vw_{i,0}\|_2 \prec N^{-1/2}$, where $\vw_{i,0} \sim \cN(0,d^{-1}\vI)$
and $\vbeta_*$ is deterministic. Hence $(\vw_i,\vbeta_*)$ satisfy the approximate
orthonormality conditions $|\|\vw_i\|_2-1| \prec N^{-1/2}$, $\|\vbeta_*\|_2-1=0$,
and $|\vw_i^\top \vbeta_*| \prec N^{-1/2}$. Then
\cite[Lemma D.3(a)]{fan2020spectra} implies
\begin{align*}
	\left|\E_{\vx}[\sigma_*(\vbeta_*^\top\vx)\sigma(\vx^\top\vw_i)]-b_\sigma
b_{\sigma_*} \vbeta_*^\top\vw_i\right|\prec N^{-1}.
\end{align*} 
(We note that \cite[Lemma D.3(a)]{fan2020spectra} assumes $\sigma=\sigma_*$, but the proof is
identical for $\sigma \neq \sigma_*$ both satisfying
Assumption \ref{assump:sigma}.)
Applying this to each coordinate $i \in [N]$ yields \eqref{eq:align_y_limit}.
Observe that $\vbeta_*^\top\vW\vv(\vW)=s_{\max}(\vW) \cdot \vbeta_*^\top\vu(\vW)$
where $s_{\max}(\vW)$ and $\vu(\vW)$ are the leading singular value and left
singular vector of $\vW$, and recall from the definitions \eqref{eq:udef}
that $\vu=\frac{1}{\sqrt{n}}\tilde \vy$.
Then we can apply \eqref{eq:align_y_limit} and
Proposition~\ref{prop:W_spike} to \eqref{eq:learnedKspike} to conclude that 
\[\frac{1}{\sqrt{n}}|\tilde\vy^\top\widehat{\vu}|\to b_\sigma b_{\sigma_*}
\frac{\sqrt{z(-1/\lambda_1)\varphi(-1/\lambda_1)}}{\lambda_1}\cdot
\frac{\theta_2\sqrt{(\theta_1^4-\gamma_0)(\gamma_0+\theta_1^2)}}{\theta_1^3}
>0 \text{ a.s.}\]
\end{proof-of-theorem}

\end{document}